\documentclass[sigconf]{aamas}
\usepackage{balance} % for balancing columns on the final page

\usepackage[utf8]{inputenc} % allow utf-8 input
\usepackage[T1]{fontenc}    % use 8-bit T1 fonts
\usepackage{booktabs}       % professional-quality tables
\usepackage{amsfonts}       % blackboard math symbols
\usepackage{nicefrac}       % compact symbols for 1/2, etc.https://www.overleaf.com/project/620e9cef2028f61302d7a900
\usepackage{microtype}      % microtypography

\usepackage{mysymbols}
\usepackage{attrib}
\usepackage{amsmath}
\usepackage{amsthm}
\usepackage{stmaryrd}
\usepackage{latexsym}
\usepackage{mathtools}
\usepackage{bm}
\usepackage{xcolor}
\usepackage{graphicx}
\usepackage{caption}
\usepackage{subcaption}
\usepackage{thmtools,thm-restate}
\theoremstyle{plain}

\newtheorem{remark}{Remark}

\theoremstyle{definition}

\usepackage{algorithm}
\usepackage{algorithmic}
\usepackage{cleveref}
\usepackage{enumitem}
\usepackage{multirow}
\usepackage{multicol}
\usepackage{wrapfig}

\DeclareMathOperator*{\argmax}{argmax}

\setcopyright{ifaamas}
\acmConference[AAMAS '23]{Proc.\@ of the 22nd International Conference
on Autonomous Agents and Multiagent Systems (AAMAS 2023)}{May 29 -- June 2, 2023}
{London, United Kingdom}{A.~Ricci, W.~Yeoh, N.~Agmon, B.~An (eds.)}
\copyrightyear{2023}
\acmYear{2023}
\acmDOI{}
\acmPrice{}
\acmISBN{}

\acmSubmissionID{???}

\title[MARL for AMR]{Multi-Agent Reinforcement Learning for Adaptive Mesh Refinement}
\author{Jiachen Yang}
\affiliation{
  \institution{LLNL}
  \city{Livermore, CA}
  \country{USA}
}
\email{yang40@llnl.gov}

\author{Ketan Mittal}
\affiliation{
  \institution{LLNL}
  \city{Livermore, CA}
  \country{USA}
}
\email{mittal3@llnl.gov}

\author{Tarik Dzanic}
\affiliation{
  \institution{Texas A\&M University}
  \city{College Station, TX}
  \country{USA}
}
\email{tdzanic@tamu.edu}

\author{Socratis Petrides}
\affiliation{
  \institution{LLNL}
  \city{Livermore, CA}
  \country{USA}
  }
\email{petrides1@llnl.gov}

\author{Brendan Keith}
\affiliation{
  \institution{Brown University}
  \city{Providence, RI}
  \country{USA}
  }
\email{brendan_keith@brown.edu}

\author{Brenden Petersen}
\affiliation{
  \institution{LLNL}
  \city{Livermore, CA}
  \country{USA}
  }
\email{petersen33@llnl.gov}

\author{Daniel Faissol}
\affiliation{
  \institution{LLNL}
  \city{Livermore, CA}
  \country{USA}
  }
\email{faissol1@llnl.gov}

\author{Robert Anderson}
\affiliation{
  \institution{LLNL}
  \city{Livermore, CA}
  \country{USA}
  }
\email{anderson110@llnl.gov}

\begin{abstract}
Adaptive mesh refinement (AMR) is necessary for efficient finite element simulations of complex physical phenomenon, as it allocates limited computational budget based on the need for higher or lower resolution, which varies over space and time. We present a novel formulation of AMR as a fully-cooperative Markov game, in which each element is an independent agent who makes refinement and de-refinement choices based on local information. We design a novel deep multi-agent reinforcement learning (MARL) algorithm called Value Decomposition Graph Network (VDGN), which solves the two core challenges that AMR poses for MARL: posthumous credit assignment due to agent creation and deletion, and unstructured observations due to the diversity of mesh geometries. For the first time, we show that MARL enables anticipatory refinement of regions that will encounter complex features at future times, thereby unlocking entirely new regions of the error-cost objective landscape that are inaccessible by traditional methods based on local error estimators. Comprehensive experiments show that VDGN policies significantly outperform error threshold-based policies in global error and cost metrics. We show that learned policies generalize to test problems with physical features, mesh geometries, and longer simulation times that were not seen in training. We also extend VDGN with multi-objective optimization capabilities to find the Pareto front of the tradeoff between cost and error.
\end{abstract}


\newcommand{\BibTeX}{\rm B\kern-.05em{\sc i\kern-.025em b}\kern-.08em\TeX}

\begin{document}

%%% The following commands remove the headers in your paper. For final 
%%% papers, these will be inserted during the pagination process.

\pagestyle{fancy}
\fancyhead{}

%%% The next command prints the information defined in the preamble.

\maketitle 

%%%%%%%%%%%%%%%%%%%%%%%%%%%%%%%%%%%%%%%%%%%%%%%%%%%%%%%%%%%%%%%%%%%%%%%%

\section{Introduction}

% Describe FEM and AMR
The finite element method (FEM) \citep{brenner2007mathematical} is instrumental to numerical simulation of partial differential equations (PDEs) in computational science and engineering \citep{reddy2010finite,monk2003finite}.
For multi-scale systems with large variations in local features, such as combinations of regions with large gradients that require high resolution and regions with flat solutions where coarse resolution is sufficient, an efficient trade-off between solution accuracy and computational cost requires the use of adaptive mesh refinement (AMR).
The goal of AMR is to adjust the finite element mesh resolution dynamically during a simulation, by refining regions that can contribute the most to improvement in accuracy relative to computational cost.

% Difficulty faced by existing approaches
For evolutionary (i.e., time-dependent) PDEs in particular, a long-standing challenge is to find anticipatory refinement strategies that optimize a long-term objective, such as an efficient tradeoff between final solution accuracy and cumulative degrees of freedom (DoF).
Anticipatory refinement strategies would preemptively refine regions of the mesh that will contain solution features (e.g., large gradients) right before these features actually occur.
This is hard for existing approaches to achieve.
% This is hard for existing approaches to achieve for general problems, where optimal AMR methods are not known \citep{bohn2021recurrent}
Traditional methods for AMR rely on estimating local refinement indicators (e.g., local error \citep{zienkiewicz1992superconvergent}) and heuristic marking strategies (e.g., greedy error-based marking) \citep{bangerth2013adaptive, Cerveny2019}.
Recent data-driven methods for mesh refinement apply supervised learning to learn a fast neural network estimator of the solution from a fixed dataset of pre-generated high-resolution solutions \citep{obiols2022nunet,wallwork2022e2n}.
However, greedy strategies based on local information cannot produce an optimal sequence of anticipatory refinement decisions in general, as they do not have sufficient information about features that may occur at subsequent time steps, while supervised methods do not directly optimize a given long-term objective.
These challenges can be addressed by formulating AMR as a sequential decision-making problem and using reinforcement learning (RL) \citep{sutton2018reinforcement} to optimize a given objective directly.
However, current single-agent RL approaches for AMR either do not easily support refinement of multiple elements per solver step \citep{yang2021reinforcement}; faces different definitions of the environment transition at training and test time \citep{foucart2022deep}; or work by selecting a single global error threshold \citep{gillette2022learning}, which poses difficulties for anticipatory refinement.

% Describe MARL approach
In this work, we present the first formulation of AMR as a Markov game \citep{owen1982game,littman1994markov} and propose a novel fully-cooperative deep multi-agent reinforcement learning (MARL) algorithm \citep{foerster2018deep,hernandez2019survey,yang2021cooperation} called Value Decomposition Graph Network (VDGN), shown in \Cref{fig:VDGN_top}, to train a team of independently and simultaneously acting agents, each of which is a decision-maker for an element, to optimize a global performance metric and find anticipatory refinement policies.
Because refinement and de-refinement actions at each step of the AMR Markov game leads to the creation and deletion of agents, we face the posthumous credit assignment problem \citep{cohen2021use}: agents who contributed to a future reward are not necessarily present at the future time to experience it.
We show that VDGN, by virtue of centralized training with decentralized execution \citep{oliehoek2016concise}, addresses this key challenge.
By leveraging graph networks as the inductive bias \citep{battaglia2018relational}, VDGN supports meshes with varying number of elements at each time step and can be applied to meshes of arbitrary size, depth, and geometry.
Moreover, graph attention layers \citep{velivckovic2018graph} in VDGN enable each element to receive information within a large local neighborhood, so as to anticipate incoming or outgoing solution features and learn to take preemptive actions.
Experimentally, using the advection equation as a pedagogical benchmark problem, we show for the first time that MARL: (a) displays anticipatory refinement behavior;
(b) generalizes to different initial conditions, initial mesh resolutions, simulation durations, and mesh geometries;
(c) significantly improves error and cost metrics compared to local error threshold-based policies, and unlocks new regions of the error-cost optimization landscape. 
Furthermore, we augment VDGN with a multi-objective optimization method to train a single policy that discovers the Pareto front of error and cost.

\begin{figure*}[t]
    \centering
    \includegraphics[width=1.0\linewidth]{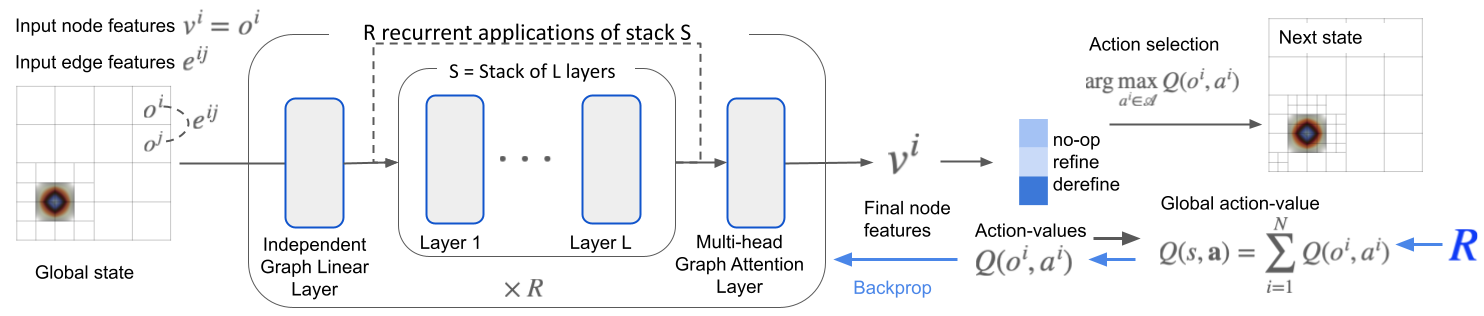}
    \caption{A state-action-next state transition in the Markov game for AMR. Shown in detail are the action selection (rightward) and learning (leftward) processes in the Value Decomposition Graph Network.}
    \label{fig:VDGN_top}
    \vspace{-5pt}
\end{figure*}

\section{Preliminaries}
\label{sec:preliminaries}

\subsection{Finite Element Method}
\label{sec:fem}

The finite element method \citep{brenner2007mathematical} models the domain of a PDE with a mesh that consists of nonoverlapping geometric elements.
Using the weak formulation and choosing basis functions to represent the PDE solution on these elements, one obtains a system of linear equations that can be numerically solved.
The shape and sizes of elements determine solution error, while the computational cost is primarily determined by the number of elements.
Improving the trade-off between error and cost is the objective of AMR.
This work focuses purely on $h$-refinement, in which refinement of a coarse element produces multiple new fine elements (e.g., isotropic refinement of a 2D quadrilateral element produces four new elements), whereas de-refinement of a group of fine elements, restores an original coarse element. 
% produces a new coarse element.
In both cases, the original coarse element or fine elements are removed from the FEM discretization.

\vspace{-5pt}
\subsection{Notation} 
\label{sec:notation}

Each element is identified with an agent, denoted by $i \in \lbrace 1,\dotsc, n_t \rbrace$, where $n_t$ is variable across time step $t=0,1,\dotsc,T_{\max}$ within each episode of length $T_{\max}$, and $n_t \in [N_{\max}]$, where $N_{\max}$ is constrained by $\text{depth}_{\max}$ to be $n_x \cdot n_y \cdot 4^{\text{depth}_{\max}}$ in the case of quadrilateral elements.
Let each element $i$ have its own individual observation space $\Scal^i$ and action space $\Acal^i$.
Let $s$ denote the global state, $o^i$ and $a^i$ denote agent $i$'s individual observation and action, respectively, and $a \defeq (a^1,\dotsc,a^{n_t})$ denote the joint action by $n_t$ agents.
In the case of AMR, all agents have the same observation and action spaces.
Let $R$ denote a single global reward for all agents and let $P$ denote the environment transition function, both of which are well-defined for any number of agents $n_t \in [N_{\max}]$.
Let $\gamma \in (0,1]$ be the discount factor.
The FEM solver time at episode step $t$ is denoted by $\tau$ (we omit the dependency $\tau(t)$ for brevity), which advances by $\tau_{\text{step}}$ in increments of $d\tau=0.002$ during each discrete step $t \rightarrow t+1$ up to final simulation time $\tau_f$.

For element $i$ at time $t$, let $u^i_t$ and $\uhat^i_t$ denote the true and numerical solution, respectively, and let $c^i_t \defeq \lVert u^i_t - \uhat^i_t \rVert_2$ denote the $\mathcal{L}_2$ norm of the error.
Let $c_t \defeq \sqrt{\sum_{i=1}^{n_t} (c^i_t)^2}$ denote the global error of a mesh with $n_t$ elements.
Let $\text{depth}(i) = 0,1,\dotsc,\text{depth}_{\max}$ denote the refinement depth of element $i$.
Let $d_t \defeq \sum_{s=0}^t \text{DoF}_s$ denote the cumulative degrees of freedom (DoF) of the mesh up to step $t$, which is a measure of computational cost, and let $d_{\text{thres}}$ be a threshold (i.e., constraint) on the cumulative DoF seen during training.

\subsection{Value Decomposition Network}
\label{sec:vdn}

In the paradigm of centralized training with decentralized execution (CTDE) \citep{oliehoek2016concise}, global state and reward information is used in centralized optimization of a team objective at training time, while decentralized execution allows each agent to take actions conditioned only on their own local observations, independently of other agents, both at training time and at test time.
One simple yet effective way to implement this in value-based MARL is Value Decomposition Networks (VDN) \citep{sunehag2018value}.
VDN learns within the class of global action-value functions $Q(s,a)$ that decompose additively:
\begin{align}\label{eq:vdn}
    Q(s,a) \defeq \sum_{i=1}^n Q^i(s^i,a^i) \, ,
\end{align}
where $Q^i$ is an individual utility function representing agent $i$'s contribution to the joint expected return.

This decomposition is amenable for use in Q-learning \citep{watkins1992q}, as it satisfies the \textit{individual-global-max} (IGM) condition:
\begin{align}\label{eq:igm}
    \argmax_{\abf \in \Pi_{i=1}^n \Acal^i} Q(s,a) = \left[ 
    \argmax_{a^1 \in \Acal^1} Q^1(s^1,a^1), \dotsc,
    \argmax_{a^n \in \Acal^n} Q^n(s^n,a^n)
    \right] \, .
\end{align}
This means the individual maxima of $Q^i$ provide the global maximum of the joint $Q$ function for the Q-learning update step \citep{watkins1992q,mnih2015human}, which scales linearly rather than exponentially in the number of agents.
Using function approximation for $Q^i_{\theta}$ with parameter $\theta$, the VDN update equations using replay buffer $\Bcal$ are:
\begin{align}
    \theta &\leftarrow \theta - \nabla_{\theta} \Ebb_{(s_t,a_t,r_t,s_{t+1}) \sim \Bcal} \left[ \left( y_{t+1} - Q_{\theta}(s_t,a_t) \right)^2 \right] \label{eq:q-learning} \\
    y_{t+1} &\defeq R_t + \bot \gamma Q_{\theta}(s_{t+1},a)\vert_{a = [\argmax_{a^i} Q^i_{\theta}(s_{t+1}^i, a^i)]_{i=1}^n} \label{eq:td-target} \, ,
    % y_{t+1} &\defeq r_t + \max_{a \in \Pi_{i=1}^n \Acal^i} Q(s_{t+1},a) \label{eq:td-target}
\end{align}
where $\bot$ is the stop-gradient operator.

\section{Agent creation and deletion}
\label{sec:creation}

Each agent's refinement and de-refinement action has long-term impact on the global error (e.g., refining before arrival of a feature would reduce error upon its arrival).
However, since all elements that refine or de-refine are removed immediately and their individual trajectories terminate, they do not observe future states and future rewards.
This is the posthumous multi-agent credit assignment problem \citep{cohen2021use}, which we propose to address using centralized training.
First, we show that an environment with variable but bounded number of agents can be written as a Markov game \citep{owen1982game} (see proof in \Cref{app:proofs}).
\begin{restatable}{proposition}{AgentCreationMarkov}
\label{prop:agent-creation-markov-game}
Let $\Mcal$ denote a multi-agent environment where the number of agents $n_t = 1,\dotsc,N_{\text{max}}$ can change at every time step $t = 0,1,\dotsc,T_{\text{max}}$ due to agent-creation and agent-deletion actions in each agent's action space.
At each time $t$, the environment is defined by the tuple 
$\left( \lbrace \Scal^i \rbrace_{i=1}^{n_t}, \lbrace \Acal^i \rbrace_{i=1}^{n_t}, R, P, \gamma, n_t, N_{\max} \right) $.
$\Mcal$ can be written as a Markov game with a stationary global state space and joint action space that do not depend on the number of currently existing agents.
\end{restatable}

\vspace{-5pt}
\textbf{Centralized training for posthumous multi-agent credit assignment.}
Our key insight for addressing the posthumous credit assignment problem stems from \Cref{prop:agent-creation-markov-game}: because the environment is Markov and stationary, we can use centralized training with a global reward to train a global state-action value function $Q(s,a)$ that (a) persists across time and (b) evaluates the expected future return of any $(s_t,a_t)$.
Crucially, these two properties enable $Q(s,a)$ to sidestep the issue of posthumous credit assignment, since the value estimate of a global state will be updated by future rewards via temporal difference learning regardless of agent deletion and creation.
To arrive at a truly multi-agent approach, we factorize the global action space so that each element uses its individual utility function $Q^i(o^i,a^i)$ to choose its own action from $\lbrace \text{no-op, refine, de-refine} \rbrace$.
This immediately leads to the paradigm of centralized training with decentralized execution \citep{oliehoek2016concise}, of which VDN (\cref{eq:vdn}) is an example.
We discuss the pros and cons of other formulations in \Cref{app:other_formulations}.

\textbf{Effective space.}
Agent creation and deletion means that the accessible region of the global state-action space changes over time during each episode.
While this is not a new phenomenon,
AMR is special in that the sizes of the informative subset of the global state and the available action set depend directly on the current number of existing agents.
Hence, a key observation for algorithm development is that a model-free multi-agent reinforcement learning algorithm only needs to account for the accessible state-action space $\prod_{i=1}^{n_t} \Scal^i \times \Acal^i$ at each time step, since the expansion or contraction of that space is part of the environment dynamics that are accounted implicitly by model-free MARL methods.
In practice, this means all dummy states $s_{\nexists}$ do not need to be input to policies, and policies do not need to output the  (mandatory) no-op actions for the $N_{\max} - n_t$ nonexistent elements at time $t$.
This informs our concrete definition of the Markov game for AMR in \Cref{sec:markov_game}.

\section{AMR as a Markov game}
\label{sec:markov_game}

\textbf{State.}
The global state is defined by the collection of all individual element observations and pairwise relational features.
The individual observation $o^i$ of element $i$ consists of:
\begin{itemize}
    \item $\log(c^i_t)$: logarithm of error at element $i$ at time $t$.
    \item 1-hot representation of element refinement depth.
\end{itemize}
Relational features $e^{ij}$ are defined for each pair of spatially-adjacent elements $i,j$ that form edge $e = (i,j)$ (directed from $j$ to $i$) in the graph representation of the mesh (see \Cref{sec:methods}), as a 1-dimensional vector concatenation of the following:
\begin{itemize}
    \item $\mathbf{1}[\text{depth}(i) - \text{depth}(j)]$: 1-hot vector indicator of the difference in refinement depth between $i$ and $j$.
    \item $\langle u_j, \frac{\Delta x}{\lVert \Delta x \rVert^2_2} \rangle \cdot \tau_{\text{step}}$: Dimensionless inner product of velocity $u_j$ at element $j$ with the displacement vector between $i$ and $j$. Here $\Delta x \defeq (x_{i} - x_{j})$, where $x_i$ is the center of element $i$.
\end{itemize}
We use the velocity $u_j$ at the sender element so that the receiver element is informed about incoming or outgoing PDE features and can act preemptively.

\textbf{Action.}
All elements have the same action space:
\begin{align*}
\vspace{-5pt}
    \Acal \defeq \lbrace \texttt{no-op}, \texttt{refine}, \texttt{de-refine} \rbrace
    ,
\vspace{-5pt}
\end{align*}
where \texttt{no-op} means the element persists to the next decision step; \texttt{refine} means the element is equipartitioned into four smaller elements; \texttt{de-refine} means that the element opts to coalesce into a larger coarse element, subject to feasibility constraints specified by the transition function (see below).

\textbf{Transition.}
Given the current state and agents' joint action, which is chosen simultaneously by all agents, the transition $P \colon \Scal \times \Acal \mapsto \Scal$ is defined by these steps:
\begin{enumerate}[leftmargin=*]
    \item Apply de-refinement rules to each element $i$ whose action is \texttt{de-refine}: (a) if, within its group of sibling elements, a majority (or tie) of elements chose \texttt{de-refine}, then the whole group is de-refined; (b) if it is at the coarsest level, i.e., $\text{depth}(i)=0$, or it belongs to a group of sibling elements in which \textit{any} element chose to refine, then its choice is overridden to be \texttt{no-op}.
    % \item Apply de-refinement rules to each element $i$ whose action is \texttt{de-refine}: (a) if it is at the coarsest level, i.e., $\text{depth}(i)=0$, or it belongs to a group of sibling elements in which \textit{any} element chose to refine, then its choice is overridden to be \texttt{no-op}; (b) if, within its group of sibling elements, a majority (or tie) of elements chose \texttt{de-refine}, then the whole group is de-refined.
    \item Apply refinement to each agent who chose $\texttt{refine}$.
    \item Step the FEM simulation forward in time by $\tau_{\text{step}}$ and compute a new solution on the updated mesh. An episode terminates when $\tau + d\tau > \tau_f$ or $d_t > d_{\text{thres}}$. This follows standard procedure in FEM \citep{anderson2019mfem} and knowledge of the transition dynamics is not used by the proposed model-free MARL approach.
\end{enumerate}

\textbf{Reward.}
We carefully design a shaped reward \citep{ng1999policy} that encourages agents to minimize the final global error.
Let $c_0 = 1.0$ be a dummy initial global error.
The reward at step $t = 1,2,\dotsc$ is
\begin{align}
R_t = \begin{cases}
        p \cdot (\tau - \tau_f) + \log(c_{t-1}), \text{ if $d_t > d_{\text{thres}} \wedge \tau + d\tau < \tau_f$} \\
        \log(c_{t-1} / c_t) - p \cdot \max(0, \frac{d_t}{d_{\text{thres}}} - 1) \text{, if $ \tau + d\tau \geq \tau_f$} \\
        \log(c_{t-1}/c_t), \quad \text{ else}
    \end{cases}
\end{align}
The first case applies a penalty $p$ when the cumulative DoF exceeds the DoF threshold before reaching the simulation final time.
The second case applies a penalty based on the amount by which the DoF threshold is exceeded when the final time is reached.
The last case provides the agents with a dense learning signal based on potential-based reward shaping \citep{ng1999policy}, so that the episode return (i.e., cumulative reward at the end of the episode) is $R = -\log(c_{T_{\max}})$ in the absence of any of the other penalties.

\textbf{Objective.}
We consider a fully-cooperative Markov game in which all agents aim to find a shared parameterized policy $\pi_{\theta} \colon \Scal \times \Acal \mapsto [0,1]$ to maximize the objective
\begin{align}
    \max_{\theta} J(\theta) \defeq \Ebb_{s_{t+1} \sim P(\cdot|a,s_t), a \sim \pi_{\theta}(\cdot|s)} \left[ \sum_{t=0}^T \gamma^t R_t \right]
\end{align}

\begin{remark}
By using time limit $\tau_f$ and DoF threshold $d_{\text{thres}}$ in the reward at training time, while not letting agents observe absolute time and cumulative DoF, agents must learn to make optimal decisions based on local observations, which enables generalization to longer simulation duration and DoF budgets at test time.
\end{remark}

\begin{remark}
For problems without an easily computable analytical solution to compute the true error, one may use an error estimator for the element observations.
The reward, which is only needed at training time and not at test time, can still be based on error with respect to a highly resolved reference simulation.
Empirical results on this configuration are provided in \Cref{app:results}.
\end{remark}

\begin{figure}[t]
    \centering
    \includegraphics[width=1.0\linewidth]{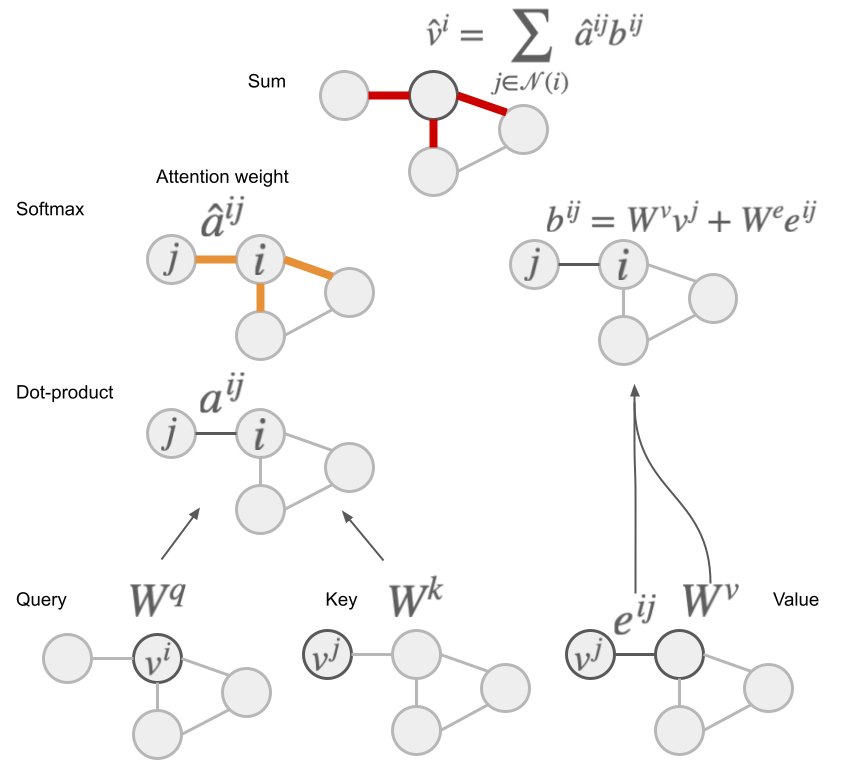}
    \caption{Graph Attention Layer. A \textcolor{orange}{softmax} over all edges connected to node $i$ produces attention weights $\ahat^{ij}$ for edge $(i,j)$ (\cref{eq:attention_weights}). A \textcolor{red}{weighted sum} over values $b^{ij}$ with weight $\ahat^{ij}$ produces the updated node feature $\vhat^i$ (\cref{eq:node_update}).}
    \label{fig:VDGN_graph_layer}
    \vspace{-10pt}
\end{figure}

\section{Value Decomposition Graph Network}
\label{sec:methods}

To enable anticipatory refinement in the time-dependent case, an element must observe a large enough neighborhood around itself.
% \jc{Add figure of domain of dependency}
However, it is difficult to define a fixed local observation window for general mesh geometries, element shapes, and boundaries.
Instead, we use Graph Networks \citep{scarselli2008graph,battaglia2018relational} as a general inductive bias to learn representations of element interactions on arbitrary meshes.

Specifically, we construct a policy based on Graph Attention Networks \citep{velivckovic2018graph}, which incorporates self-attention \citep{vaswani2017attention} into graph networks.
At each step $t$, the mesh is represented as a graph $\Gcal = (V_t, E_t)$.
Each node $v^i$ in $V = \lbrace v^i \rbrace_{i=1:n_t}$ corresponds to element $i$ and its feature is initialized to be the element observation $o^i$.
$E = \lbrace e^k = (r^k,s^k) \rbrace_{k=1:N^e}$
is a set of edges, where an edge $e^k$ exists between sender node $s^k$ and receiver node $r^k$ if and only if they are spatially adjacent (i.e., sharing either a face or a vertex in the mesh).
Its feature is initialized to be the relational feature $e^{r^ks^k}$.

\subsection{Graph Attention Layer}
\label{sec:graph_attention_layer}

In a graph attention layer, each node is updated by a weighted aggregation over its neighbors: weights are computed by self-attention using node features as queries and keys, then applied to values that are computed from node and edge features.

Self-attention weights $\ahat^{ij}$ for each node $i$ are computed as follows (see \Cref{fig:VDGN_graph_layer}): 1) we define queries, keys, and values as linear projections of node features, via weight matrices $W^q$, $W^k$, and $W^v$ (all $\in \Rbb^{d \times \dim(v)}$) shared for all nodes; 2) for each edge $(i,j)$, we compute a scalar pairwise interaction term $a^{ij}$ using the dot-product of queries and keys; 3) for each receiver node $i$ with sender node $j \in \Ncal_i$, we define the attention weight as the $j$-th component of a \texttt{softmax} over all neighbors $k \in \Ncal_i$:
\begin{align}
\vspace{-10pt}
    a^{ij} &\defeq W^q v^i \cdot W^k v^j && \text{for } (i,j) \in E \, , \label{eq:attention_dot_product} \\
    \ahat^{ij} &\defeq \texttt{softmax}_j(\lbrace a^{ik} \rbrace_k) = \frac{\exp(a^{ij})}{\sum_{k \in \Ncal_i} \exp(a^{ik})} && \text{for } j \in \Ncal_i \, . \label{eq:attention_weights}
\vspace{-10pt}
\end{align}
We use these attention weights to compute the new feature for each node $i$ as a linear combination over its neighbors $j \in \Ncal_i$ of projected values $W^v v^j$.
Edge features $e^{ij}$, with linear projection using $W^e \in \Rbb^{d \times \dim(e)}$,  capture the relational part of the observation:
\begin{align}
\vspace{-10pt}
    b^{ij} &\defeq W^v v^j + W^e e^{ij} && \text{for } (i,j) \in E \, , \label{eq:attention_value} \\
    \vhat^i &\defeq \sum_{j \in \Ncal_i} \ahat^{ij} b^{ij} && \text{for } i \in V \, . \label{eq:node_update}
\vspace{-10pt}
\end{align}
Despite being a special case of the most general message-passing flavor of graph networks \citep{battaglia2018relational}, graph attention networks separate the learning of $\ahat^{ij}$, the scalar importance of interaction between $i$ and $j$ relative to other neighbors, from the learning of $b^{ij}$, the vector determining how $j$ affects $i$.
This additional inductive bias reduces the functional search space and can improve learning with large receptive fields around each node, just as attention is useful for long-range interactions in sequence data \citep{vaswani2017attention}.

\textbf{Multi-head graph attention layer.}
We extend graph attention by building on the mechanism of multi-head attention \citep{vaswani2017attention}, which uses $H$ independent linear projections $(W^{q,h}, W^{k,h}, W^{v,h}, W^{e,h})$ for queries, keys, values and edges (all projected to dimension $d/H$), and results in $H$ independent sets of attention weights $\ahat^{ij,h}$, with $h = 1,2,\dotsc,H$ .
This enables attention to different learned representation spaces and was found to stabilize learning \citep{velivckovic2018graph}.
The new node representation with multi-head attention is the concatenation of all output heads, with an output linear projection $W^o \in \Rbb^{d \times d}$.
\Cref{app:multi-head} shows the multi-head versions of \cref{eq:attention_dot_product,eq:attention_weights,eq:attention_value,eq:node_update}.

\subsection{Value Decomposition Graph Network}
\label{sec:vdgn}

We use the graph attention layer to define the Value Decomposition Graph Network (VDGN), shown in \Cref{fig:VDGN_top}.
Firstly, an independent graph network block \citep{battaglia2018relational} linearly projects nodes and edges independently to $\Rbb^d$.

\begin{wrapfigure}{l}{0.22\textwidth}
% \begin{figure}
    \vspace{-15pt}
    \centering
    \includegraphics[width=0.22\textwidth]{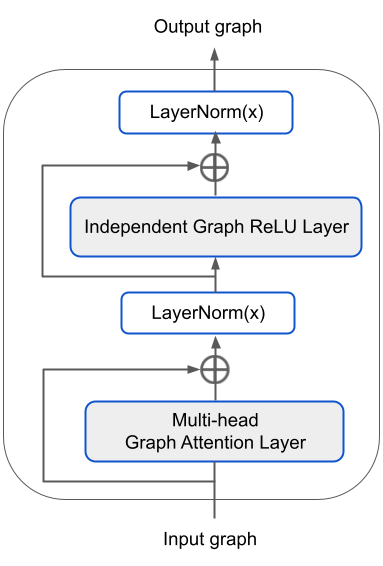}
    % \vspace{-5pt}
    \caption{VDGN layer}
    \label{fig:VDGN_layer}
    \vspace{-10pt}
% \end{figure}
\end{wrapfigure}Each layer $F$ of VDGN involves two sub-layers: a multi-head graph attention layer followed by a fully-connected independent graph network block with \texttt{ReLU} nonlinearity.
For each of these sub-layers, we use residual connections \citep{he2016deep} and layer normalization \citep{ba2016layer}, so that the transformation of input graph $g$ is $\text{LayerNorm}(g + \text{SubLayer}(g))$.
This is visualized in \Cref{fig:VDGN_layer}.
We define a VDGN \textit{stack} as $S \defeq F_L(F_{L-1}(\dotsm F_1(g)\dotsm))$ with $L$ unique layers without weight sharing, then define a single forward pass by $r$ recurrent applications of $S$:
$P(g) = S(S(\dotsm S(g)\dotsm))$ with $r$ instances of $S$.
Finally, to produce $|\Acal|$ action-values for each node (i.e., element) $i$, we apply a final graph attention layer whose output linear projection is $W^o_{\text{out}} \in \Rbb^{|\Acal| \times d}$, so that each final node representation is $\vhat^i \in \Rbb^{|\Acal|}$, interpreted as the individual element utility $Q(o^i,a^i)$ for all possible actions $a^i \in \Acal$.

VDGN is trained using \cref{eq:vdn} in a standard Q-learning algorithm \citep{mnih2015human} (see the leftward process in \Cref{fig:VDGN_top}).
Since VDGN fundamentally is a $Q$ function-based method, we can extend it with a number of independent algorithmic improvements \citep{hessel2018rainbow} to single-agent Deep Q Network \citep{mnih2015human} that provide complementary gains in learning speed and performance.
% Since VDGN fundamentally is a $Q$ function-based method, we can extend it with
These include
double Q-learning \citep{van2016deep}, dueling networks \citep{wang2016dueling}, and prioritized replay \citep{schaul2015prioritized}.
Details are provided in \Cref{app:value_learning_improvements}.
We did not employ the other improvements such as noisy networks, distributional Q-learning, and multi-step returns because the AMR environment dynamics are deterministic for the PDEs we consider and the episode horizon at train time is short.
% VDGN is trained using \cref{eq:vdn} in a standard Q-learning algorithm \citep{mnih2015human}, along a number of value-function based algorithmic improvements detailed in \Cref{sec:improvements}.

\subsection{Symmetries}
\label{sec:symmetries}

Methods for FEM and AMR should respect symmetries of the physics being simulated.
For the simulations of interest in this work, we require a refinement policy to satisfy two properties: a) spatial equivariance: given a spatial rotation or translation of the PDE, the mesh refinement decisions should also rotate or translate in the same way; b) time invariance: the same global PDE state at different absolute times should result in the same refinement decisions.
By construction of node and edge features, and the fact that graph neural networks operate on individual nodes and edges independently, we have the following result, proved in \Cref{app:proofs}.
\begin{restatable}{proposition}{Symmetry}\label{prop:symmetry}
% \begin{proposition}\label{prop:symmetry}
VDGN is equivariant to global rotations and translations of the error and velocity field, and it is time invariant.
% \end{proposition}
\end{restatable}

% \subsection{Improvements to VDGN}
% \label{sec:improvements}

% A number of independent improvements \citep{hessel2018rainbow} to single-agent Deep Q Network \citep{mnih2015human} can provide complementary gains in learning speed and performance.
% Since VDGN fundamentally is a $Q$ function-based method, we can extend it with
% double Q-learning \citep{van2016deep}, dueling networks \citep{wang2016dueling}, and prioritized replay \citep{schaul2015prioritized}.
% Details are provided in \Cref{app:value_learning_improvements}.
% We did not employ the other improvements contained in Rainbow \citep{hessel2018rainbow} (noisy networks, distributional Q-learning, and multi-step returns) because the environment dynamics are deterministic, and the episode horizon at train time is short.

\subsection{Multi-objective VDGN}
\label{sec:multi-objective}

In applications where a user's preference between minimizing error and reducing computational cost is not known until test time, one cannot \textit{a priori} combine error and cost into a single scalar reward at training time.
Instead, one must take a multi-objective optimization viewpoint \citep{hayes2022practical} and treat cost and error as separate components of a vector reward function $\Rbf = (R^c, R^e)$.
The components encourage lower DoFs and lower error, respectively, and are defined by
\begin{align}\label{eq:multi_objective_reward}
\vspace{-5pt}
R^c_t \defeq \frac{d_{t-1} - d_t}{d_{\text{thres}}} ; \quad \quad 
R^e_t \defeq \log(c_{t-1}) - \log(c_t)
.
\vspace{-5pt}
\end{align}
The objective is $\Rbb^2 \ni \Jbf(\theta) \defeq \Ebb_{s_{t+1} \sim P(\cdot|a,s_t), a \sim \pi_{\theta}(\cdot|s)} [ \sum_{t=0}^T \gamma^t \Rbf_t ]$, which is vector-valued.
We focus on the widely-applicable setting of linear preferences, whereby a user's scalar utility based on preference vector $\omega$ is $\omega^T \Rbf$ (e.g., $\omega = [0.5,0.5]$ implies the user cares equally about cost and error).
At training time, we randomly sample $\omega \in \Omega$ in each episode and aim to find an optimal action-value function
$\Qbf^*(s,a,\omega) \defeq \arg_Q \sup_{\pi \in \Pi} \omega^T \Ebb_{\pi} [ \sum_{t=0}^T \gamma^T \Rbf_t ]$, where $\arg_Q$ extracts the vector $E_{\pi}[\dotsc]$ corresponding to the supremum.
We extend VDGN with Envelop Q-learning \citep{yang2019generalized}, a multi-objective RL method that efficiently finds the convex envelope of the Pareto front in multi-objective MDPs; see \Cref{app:multi-objective} for details.
Once trained, $\Qbf^*$ induces the optimal policy for any preference $\omega$ according to the greedy policy $a^* = \argmax_{a} \omega^T \Qbf^*(s,a,\omega)$.

\section{Experimental setup}
\label{sec:experimental_setup}

We designed experiments to test the ability of VDGN to find generalizable AMR strategies that display anticipatory refinement behavior, and benchmark these policies against standard baselines on error and DoF metrics.
We define the FEM environment in \Cref{subsec:amr_environment},
% the train-test procedure in \Cref{subsec:procedure}, 
and the implementation of our method and baselines in \Cref{subsec:implementation} and \Cref{app:implementation}.
Results are analyzed in \Cref{sec:results}.
% \Cref{app:setup} contains precise implementation details not

\subsection{AMR Environment}
\label{subsec:amr_environment}

We use MFEM \citep{anderson2019mfem, mfem-web} and PyMFEM \citep{mfem-web}, a modular open-source library for FEM, to implement the Markov game for AMR.
We ran experiments on the linear advection equation $\frac{\partial u}{\partial t} + \nu \nabla {\cdot} u = 0$ with random initial conditions (ICs) for velocity $\nu$ and solution $u(0)$, solving it using the FEM framework on a two-dimensional $L^2$ finite element space with periodic boundary conditions.
Each discrete step of the Markov game is a mesh re-grid step, with $\tau_{\text{step}}$ FEM simulation time elapsing between each consecutive step.
The solution is represented using discontinuous first-order Lagrange polynomials, and the initial mesh is partitioned into $n_x \times n_y$ quadrilateral elements.
\Cref{app:implementation} contains further FEM details on the mesh partition and the construction of element observations.

Linear advection is a useful benchmark for AMR despite its seeming simplicity because the challenge of anticipatory refinement can be made arbitrarily hard by increasing the $\tau_{\text{step}}$ of simulation time that elapses between two consecutive steps in the Markov game (i.e., between each mesh update step).
Intuitively, an optimal refinement strategy must refine the entire connected region that covers the path of propagation of solution features with large solution gradients (i.e., high error on a coarse mesh), and maintain coarse elements everywhere else.
Hence, the larger the $\tau_{\text{step}}$, the harder it is for distant elements, which currently have low error but will experience large solution gradients later, to refine preemptively.
But such long-distance preemptive refinement capability is exactly the key for future applications in which one prefers to have few re-meshing steps during a simulation due to its computational cost.
Moreover, the existence of an analytic solution enables us to benchmark against error threshold-based baselines under the ideal condition of having access to perfect error estimator.

\textbf{Metric.}
Besides analyzing performance via error $c$ and cumulative DoFs $d$ individually, we define an efficiency metric as
$\eta = 1 - \sqrt{\tilde{c}^2 + \tilde{d}^2} \in [0, 1]$, where a higher value means higher efficiency.
Here, 
$\tilde{c} \defeq \frac{c - c_{\text{fine}}}{c_{\text{coarse}} - c_{\text{fine}}}$
is a normalized solution error and
$\tilde{d} \defeq \frac{d - d_{\text{coarse}}}{d_{\text{fine}} - d_{\text{coarse}}}$
is normalized cumulative degrees of freedom (a measure of computational cost).
Here, the subscripts ``fine`` and ``coarse'' indicate that the quantity is computed on a constant uniformly fine and coarse mesh, respectively, that are held static (not refined/de-refined) over time.
The uniformly fine and coarse meshes themselves attain efficiency $\eta = 0$.
Efficiency $\eta = 1$ is unattainable in principle, since non-trivial problems require $d > d_{\text{coarse}}$.

\vspace{-5pt}
\subsection{Implementation and Baselines}
\label{subsec:implementation}

The graph attention layer of VDGN was constructed using the Graph Nets library \citep{battaglia2018relational}.
We used hidden dimension $64$ for all VDGN layers (except output layer of size $|\Acal|$), and $H=2$ attention heads.
For $\text{depth}_{\max}=1$, with initial $n_x=n_y=16$, we chose $L=2$ internal layers, with $R=3$ recurrent passes.
Each Markov game step has $\tau_{\text{step}}=0.25$, $\tau_f=0.75$ (hence $T_{\max}=3$).
For $\text{depth}_{\max}=2$, with $n_x=n_y=8$, we used $L=R=2$.
Each Markov game step has $\tau_{\text{step}}=0.2$, $\tau_f=0.8$ (hence $T_{\max}=4$).
For each training episode, we uniformly sampled the starting position and velocity of a 2D isotropic Gaussian wave as the initial condition.
The FEM solver time discretization was $d\tau=0.002$ throughout.
See \Cref{app:implementation} for further architectural details and hyperparameters.

We compare with the class of local error-based \textbf{Threshold} policies, each member of which is defined by a tuple $(\theta_r,\theta_d, \tau_{\text{step}})$ as follows:
every $\tau_{\text{step}}$ of simulation time, all elements whose true error exceed $\theta_r$ are refined, while those with true error below $\theta_d$ are de-refined.
These policies represent the ideal behavior of widely-used AMR methods based on local error estimation, in the limit of perfectly accurate error estimation.

\begin{remark}
Crucially, note that the Threshold policy class does not necessarily contain the global optimal policy for all AMR problems
% ---and hence can be outperformed by MARL policies as results below show---
because such policies are incapable of anticipatory refinement and cannot access the full error-cost objective landscape.
Suppose an element $i$ with flat features and negligible error $c^i_t \ll 1$ at time $t$ needs to refine before the arrival of complex PDE features at $t+1$.
If $\theta_r > c^i_t$, then element $i$ is not refined preemptively and large error is incurred at $t+1$.
If $\theta_r < c^i_t$, then many other elements $j$ with error $c^j_t > \theta_r$ are also refined at $t$ but they may not contain complex features at $t+1$, so DoF cost is unnecessarily increased.
\end{remark}

\section{Experimental results}
\label{sec:results}

Overall, we find that VDGN policies display anticipatory refinement, generalize to different initial conditions, mesh resolutions and simulation durations, thereby uncovering Pareto-efficient regions of the error-cost trade-off that were previously inaccessible by traditional error-estimator-based methods.
VDGN policy runtimes are comparable to Threshold policies (see \Cref{tab:runtime})

%Config (old):
% a: iso_vel1_nx16_ny16_depth1_tstep0p5_vdn_graphnet_parallel2
% b: iso_velunif_nx16_ny16_depth1_tstep0p25_vdn_graphnet_diffepsilon
\begin{figure}[t]
    \centering
    \begin{subfigure}[t]{0.49\linewidth}
    \includegraphics[width=1.0\linewidth]{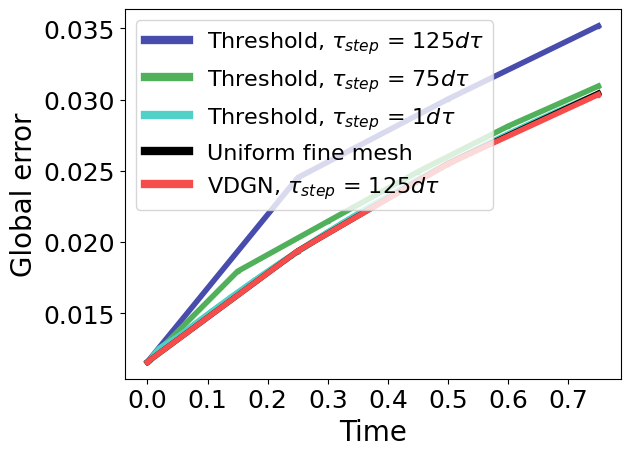}
    \caption{$\text{depth}_{\max} = 1$}
    \label{fig:error_vs_time_depth1}
    \end{subfigure}
    \begin{subfigure}[t]{0.49\linewidth}
    \includegraphics[width=1.0\linewidth]{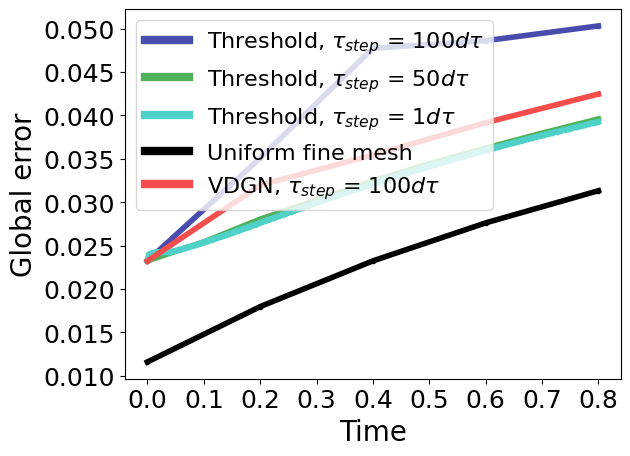}
    \caption{$\text{depth}_{\max} = 2$}
    \label{fig:error_vs_time_depth2}
    \end{subfigure}
    \vspace{-5pt}
    \caption{Global error versus simulation time of VDGN, compared with Thresholld policies with different $\tau_{\text{step}}$ between each mesh update step. (a) VDGN with the longest duration $\tau_{\text{step}} = 125 d\tau$ has error growth comparable to Threshold with the shortest duration $\tau_{\text{step}} = 1d\tau$. (b) VDGN significantly outperforms its Threshold counterpart with $\tau_{\text{step}}=100d\tau$.}
    \label{fig:error_vs_time}
    \vspace{-5pt}
\end{figure}

\begin{figure}[t]
\centering
\begin{subfigure}[t]{0.49\linewidth}
    \centering
    \includegraphics[width=0.49\linewidth]{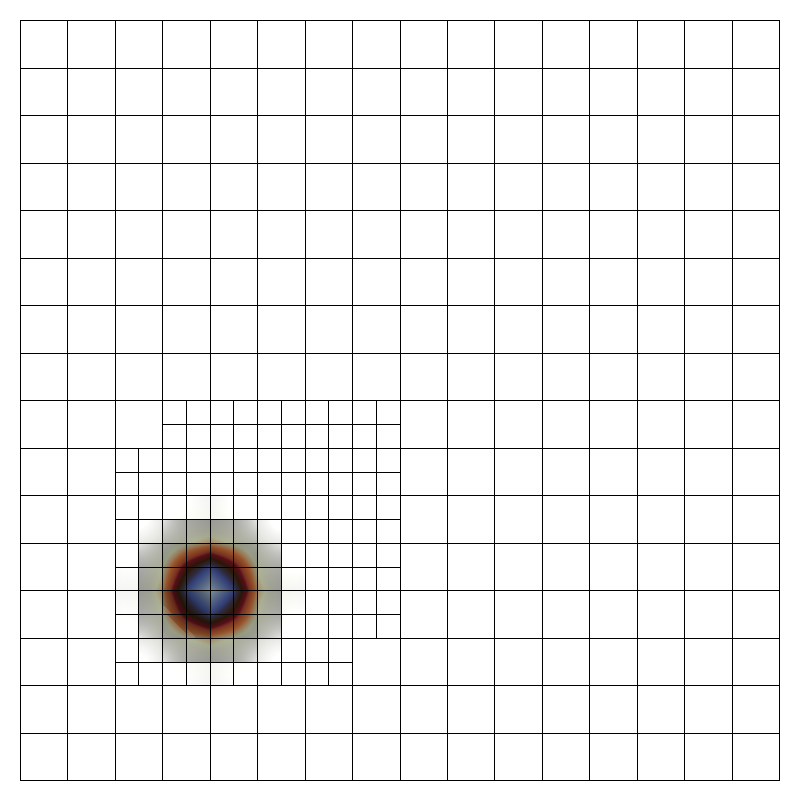}
    \includegraphics[width=0.49\linewidth]{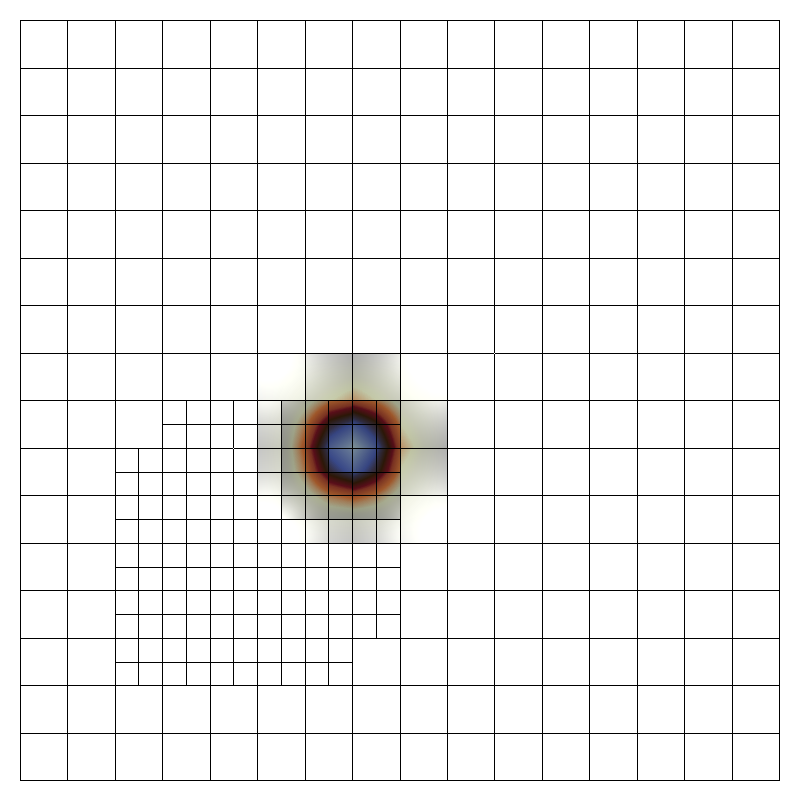}
    % \captionsetup{labelformat=empty}
    % \caption{$t=1$}
\end{subfigure}
\hfill
\begin{subfigure}[t]{0.49\linewidth}
    \centering
    \includegraphics[width=0.49\linewidth]{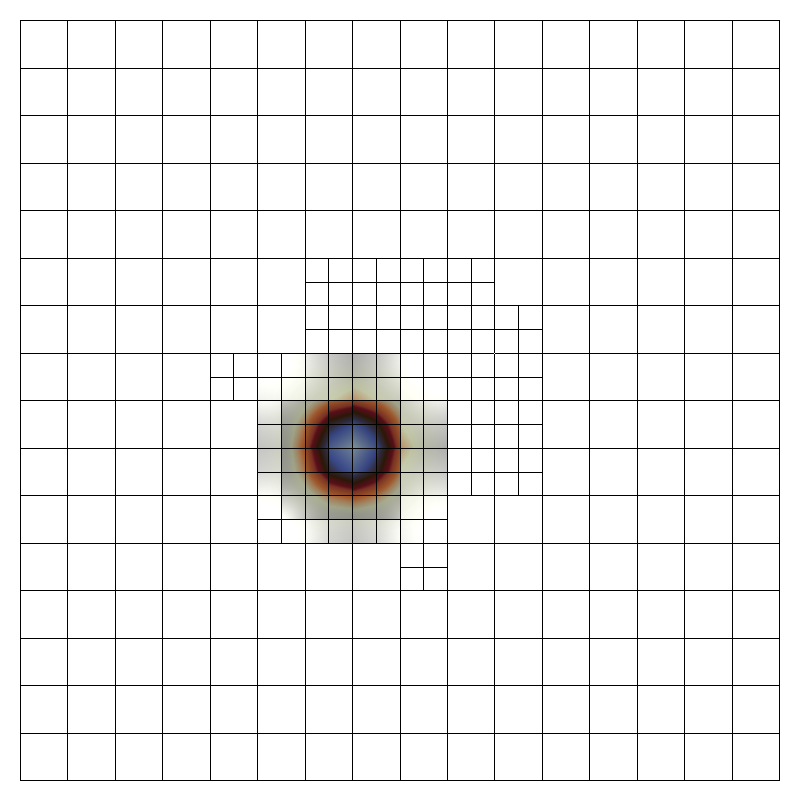}
    \includegraphics[width=0.49\linewidth]{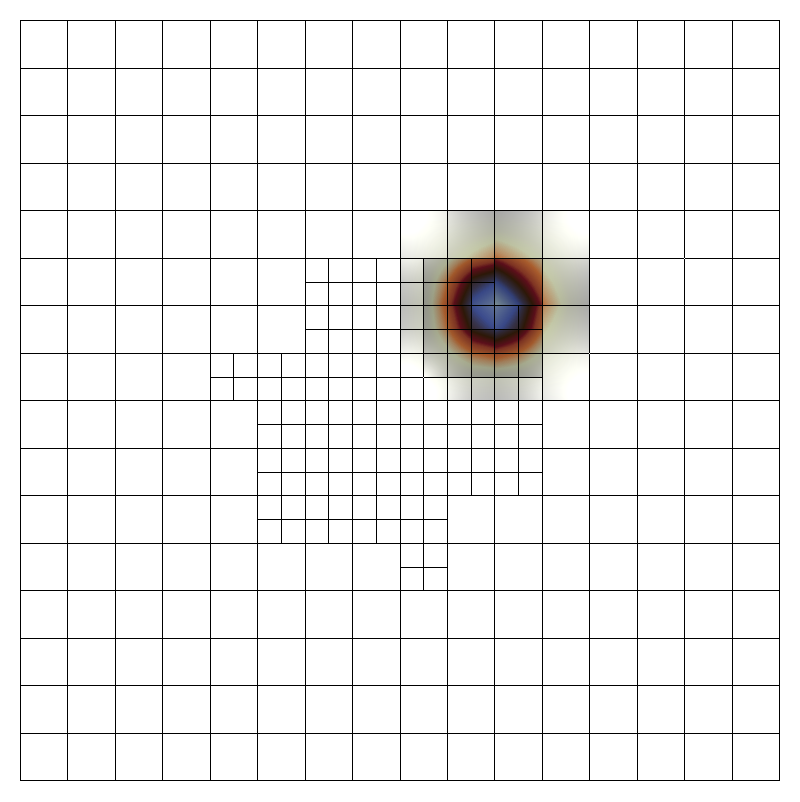}
    % \captionsetup{labelformat=empty}
    % \caption{$t=2$}
\end{subfigure}
\caption{At each RL step, VDGN refines the full path segment that will be traversed by the wave over many solver steps into the future. Here, and for all subsequent mesh visualizations, we show the process: refinements, solution after $\tau_{\text{step}}$, refinements, and so on.}
\label{fig:iso_periodic_velunif_nx16_ny16_depth1_tstep0p25_vdn_graphnet_nodoftime_3_ep210800}
\vspace{-10pt}
\end{figure}

\begin{figure}[t]
\centering
\begin{subfigure}[t]{0.49\linewidth}
    \centering
    \includegraphics[width=0.49\linewidth]{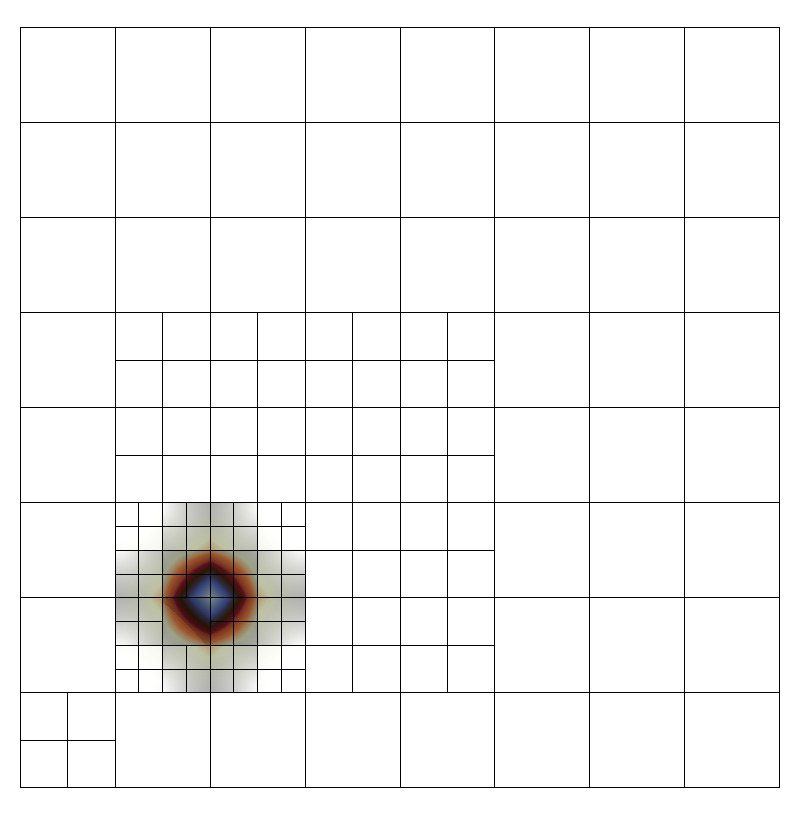}
    \includegraphics[width=0.49\linewidth]{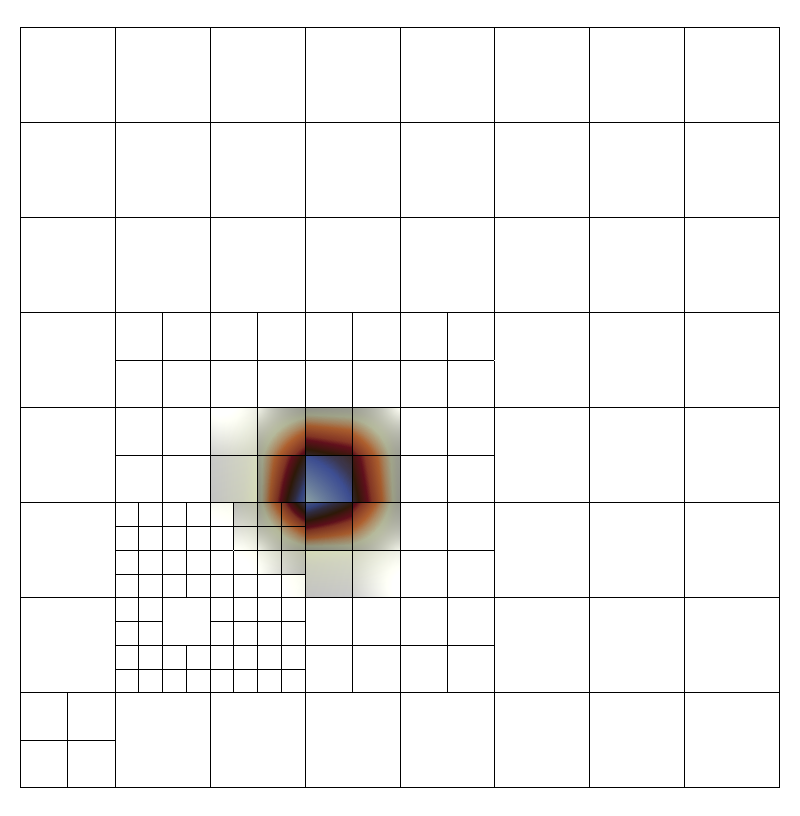}
    % \captionsetup{labelformat=empty}
    % \caption{$t=1$}
\end{subfigure}
\hfill
\begin{subfigure}[t]{0.49\linewidth}
    \centering
    \includegraphics[width=0.49\linewidth]{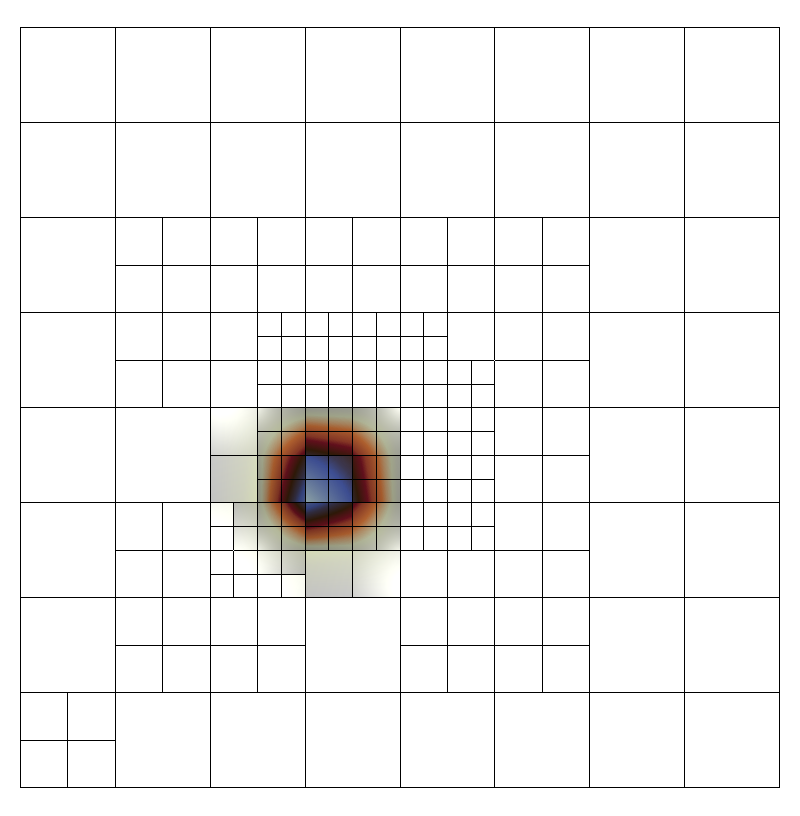}
    \includegraphics[width=0.49\linewidth]{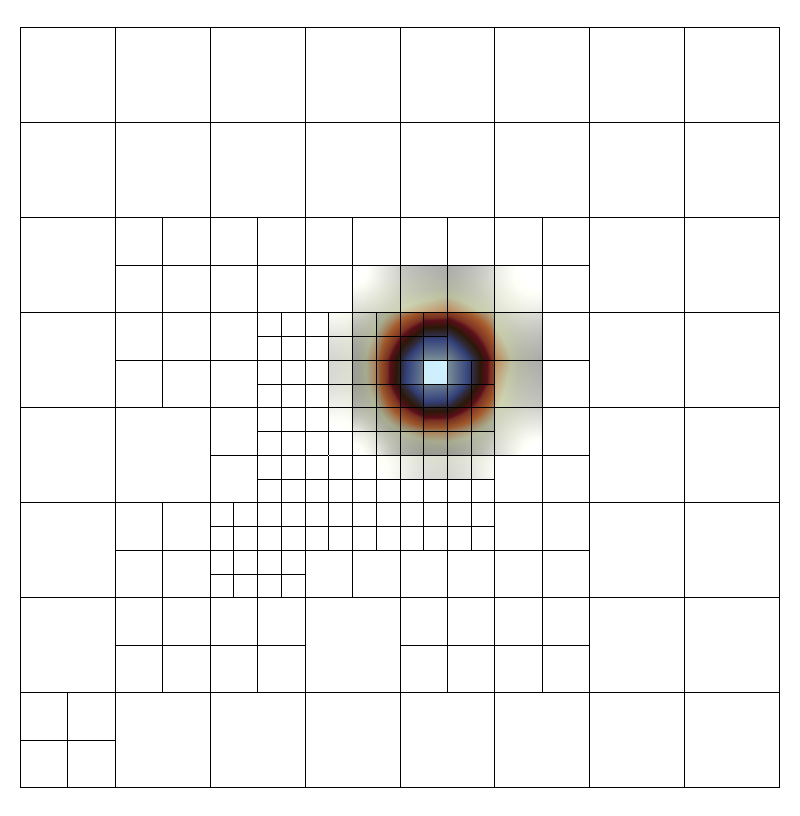}
    % \captionsetup{labelformat=empty}
    % \caption{$t=2$}
\end{subfigure}
\caption{VDGN chooses more level-1 refinement than necessary for the wave's location at $t+1$, so that level-2 refinement is possible for the wave's location at $t+2$.}
\label{fig:iso_periodic_velunif_nx8_ny8_depth2_tstep0p2_vdn_graphnet_nodoftime_3_ep265600}
\vspace{-15pt}
\end{figure}

\subsection{Anticipatory Refinement}

As discussed in \Cref{subsec:amr_environment}, each mesh update incurs a computational cost, which means AMR practitioners prefer to have a long duration of simulation time between each mesh update step.
\Cref{fig:error_vs_time} shows the growth of global error versus simulation time of VDGN and Threshold policies with different $\tau_{\text{step}}$ between each mesh update step.
In the case of $\text{depth}_{\max}=1$, VDGN was trained and tested using the longest duration $\tau_{\text{step}} = 125 d\tau$ (i.e., it has the fewest mesh updates), but it matches the error of the most expensive threshold policy that updates the mesh after each $\tau_{\text{step}} = 1 d\tau$ (see \Cref{fig:error_vs_time_depth1}).
This is possible only because VDGN preemptively refines the contiguous region that will be traversed by the wave within $125 d\tau$ (e.g., see 
\Cref{fig:iso_periodic_velunif_nx16_ny16_depth1_tstep0p25_vdn_graphnet_nodoftime_3_ep210800}).
In contrast, Threshold must update the mesh every $1d\tau$ to achieve this performance, since coarse elements that currently have negligible error due to their distance from the incoming feature do not refine before the feature's arrival.

Moreover, in the case of $\text{depth}_{\max}=2$, the agents learned to choose level-1 refinement at $t=1$ for a region much larger than the feature's periphery, so that these level-1 elements can preemptively refine to level 2 at $t=2$ before the feature passes over them. This is clearly seen in
\Cref{fig:iso_periodic_velunif_nx8_ny8_depth2_tstep0p2_vdn_graphnet_nodoftime_3_ep265600}.
This enabled VDGN with $\tau_{\text{step}} =100d\tau$ (fewest update steps) to have error growth rate close to that of Threshold with $\tau_{\text{step}} =1d\tau$ (see \Cref{fig:error_vs_time_depth2}).

\textbf{Symmetry.}
Comparing \Cref{fig:iso_periodic_velunif_nx16_ny16_depth1_tstep0p25_vdn_graphnet_nodoftime_3_ep210800_degree225} with \Cref{fig:iso_periodic_velunif_nx16_ny16_depth1_tstep0p25_vdn_graphnet_nodoftime_3_ep210800}, we see that VDGN policies are equivariant to rotation of initial conditions.
Reflection equivariance is also visible for the opposite moving waves in \Cref{fig:iso_periodic_velunif_nx16_ny16_depth1_tstep0p25_vdn_graphnet_nodoftime_3_ep210800_opposite}.
Translation equivariance can be seen in \Cref{fig:iso_periodic_velunif_nx16_ny16_depth1_tstep0p25_vdn_graphnet_nodoftime_3_fourbumps}.
Note that perfect symmetry holds only for rotation by integer multiples of $\pi/2$ and translation by integer multiples of the width of a level-0 element. Symmetry violation from mesh discretization is unavoidable for other values.

%Config: velocity [1,0], x0 0.5, y0 0.5, nx64, ny1, depth1, dof threshold 1040, deref, tstep0.5 num recurrent passes 14
\begin{figure}[t]
\centering
    \begin{subfigure}[t]{0.49\linewidth}
    \includegraphics[width=1.0\linewidth]{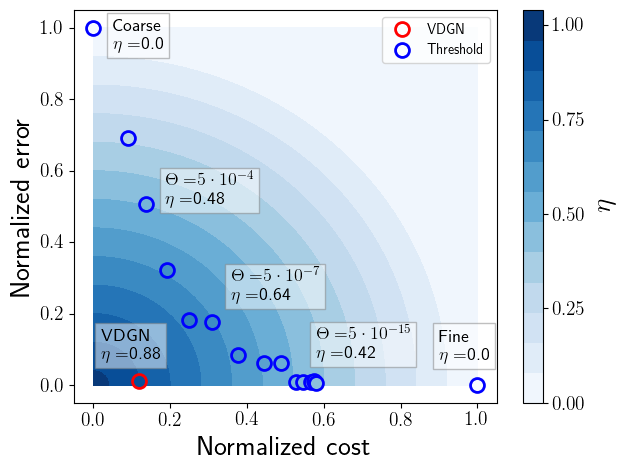}
    \caption{$n_x = n_y = 16$, max depth 1, 500 solver steps}
    \label{fig:pareto_iso_periodic_velunif_nx16_ny16_depth1_tstep0p25_vdn_graphnet_nodoftime_3_ep210800}
    \end{subfigure}
    \hfill
    \begin{subfigure}[t]{0.49\linewidth}
    \includegraphics[width=1.0\linewidth]{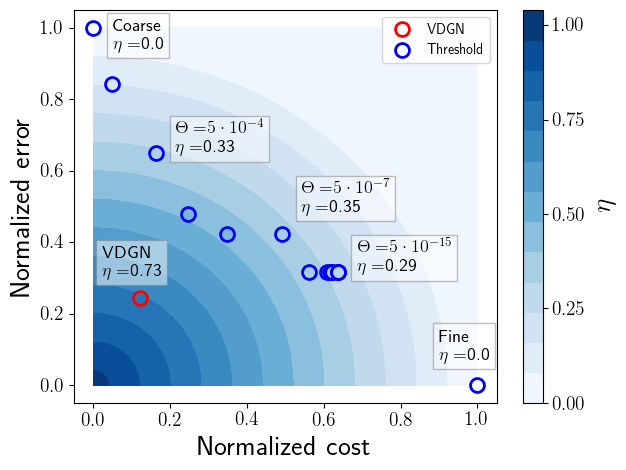}
    \caption{$n_x = n_y = 8$, max depth 2}
    \label{fig:pareto_iso_periodic_velunif_nx8_ny8_depth2_tstep0p2_vdn_graphnet_nodoftime_3_ep265600}
    \end{subfigure}
    \hfill
    \begin{subfigure}[t]{0.49\linewidth}
    \includegraphics[width=1.0\linewidth]{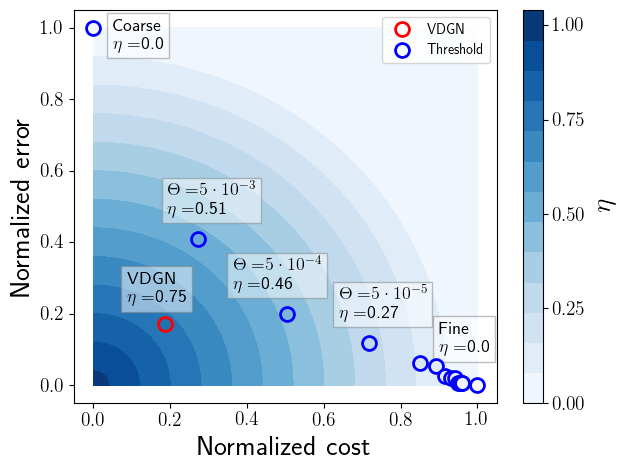}
    \caption{$n_x = n_y = 16$, max depth 1, 2500 solver steps.}
    \label{fig:pareto_iso_periodic_velunif_nx16_ny16_depth1_tstep0p25_vdn_graphnet_nodoftime_3_ep210800_tfinal5}
    \end{subfigure}
    \hfill
    \begin{subfigure}[t]{0.49\linewidth}
    \includegraphics[width=1.0\linewidth]{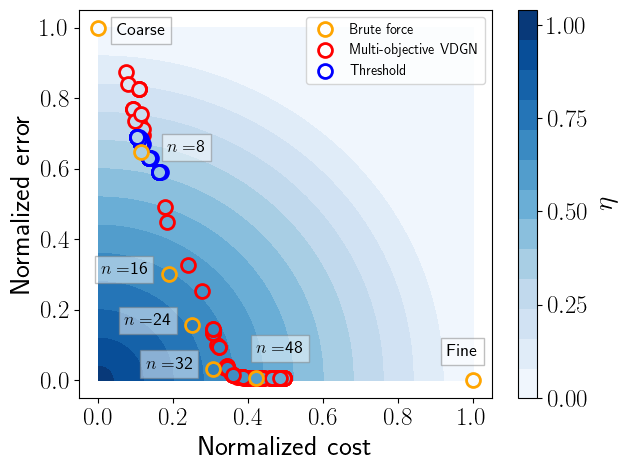}
    \caption{Pareto front of multi-objective VDGN.}
    \label{fig:pareto_aniso_nx64_depth1}
    \end{subfigure}
\vspace{-10pt}
\caption{Trade-off between error and computational cost. VDGN unlocks regions inaccessible by threshold policies.
% \bp{Font is a bit too small for me in the legend and the individual point labels.}\jc{Yup, I increased the sizes.}
}
\label{fig:pareto}
\vspace{-10pt}
\end{figure}

\begin{table*}[t]
    \centering
    \caption{Mean and standard error of efficiency of VDGN versus Threshold policies, over 100 runs with uniform random ICs.
    Policies used in generalization tests were trained on isotropic Gaussian features on a square mesh with quadrilateral elements for less than $500d\tau$ per episode. Last row shows the efficiency ratio of VDGN to the best threshold policy.}
    \vspace{-5pt}
    \begin{tabular}{crrrrrrrrr}
    \toprule
    \multicolumn{1}{}{}&
    \multicolumn{3}{c}{In-distribution} &
    \multicolumn{6}{c}{Generalization}\\
    \cmidrule(r){2-4}
    \cmidrule(r){5-10}
    $\theta_r$ & \shortstack{Depth 1 \\ $375$ steps} & \shortstack{Depth 2 \\ $400$ steps} & \shortstack{Triangular \\ $1250$ steps} & \shortstack{Depth 1 \\ $2500$ steps} & \shortstack{Depth 2 \\ $2500$ steps} & \shortstack{Anisotropic \\ $2500$ steps} & \shortstack{Ring \\ $2500$ steps} & \shortstack{Opposite \\ $2500$ steps} & \shortstack{Star \\ $750$ steps} \\
    \midrule
    $5\times10^{-3}$ & 0.35 (0.01) & 0.35 (0.01) & 0.69 (0.02) & 0.52 (0.02) & 0.41 (0.01) & 0.62 (0.02) & 0.61 (0.02) & 0.56 (0.002) & 0.85 (0.01) \\
    $5\times10^{-4}$ & 0.74 (0.02) & 0.51 (0.01) & 0.76 (0.01) & 0.57 (0.02) & 0.43 (0.02) & 0.64 (0.02) & 0.56 (0.02) & 0.61 (0.01) & 0.92 (0.01) \\
    $5\times10^{-5}$ & 0.84 (0.01) & 0.56 (0.01) & 0.66 (0.02)& 0.46 (0.02) & 0.34 (0.02) & 0.53 (0.02) & 0.48 (0.02) & 0.50 (0.01) & 0.87 (0.01) \\
    $5\times10^{-6}$ & 0.82 (0.01) & 0.55 (0.01) & 0.56 (0.02)& 0.37 (0.02) & 0.26 (0.01) & 0.42 (0.02) & 0.39 (0.02) & 0.42 (0.01) & 0.83 (0.01) \\
    $5\times10^{-7}$ & 0.77 (0.01) & 0.50 (0.01) & 0.47 (0.02) & 0.30 (0.02) & 0.20 (0.01) & 0.33 (0.02) & 0.31 (0.02) & 0.32 (0.01) & 0.79 (0.01) \\
    $5\times10^{-8}$ & 0.70 (0.01) & 0.44 (0.01) & 0.38 (0.02) & 0.24 (0.02) & 0.15 (0.01) & 0.26 (0.02) & 0.25 (0.02) & 0.26 (0.01) & 0.74 (0.01) \\
    $5\times10^{-15}$ & 0.34 (0.01) & 0.27 (0.01) & 0.10 (0.01) & 0.08 (0.01) & 0.05 (0.003) & 0.07 (0.01) & 0.07 (0.01) & 0.07 (0.01) & 0.45 (0.02) \\
    \midrule
    VDGN & \textbf{0.92} (0.01) & \textbf{0.66} (0.01) & \textbf{0.80} (0.01) & \textbf{0.84} (0.004) & \textbf{0.73} (0.01) & \textbf{0.78} (0.02) & \textbf{0.82} (0.01) & \textbf{0.63} (0.03) & \textbf{0.93} (0.01) \\
    VDGN/best $\theta_r$ & 1.10 & 1.18 & 1.05 & 1.47 & 1.70 & 1.22 & 1.34 & 1.03 & 1.13 \\
    \bottomrule
    \end{tabular}
    \label{tab:efficiency}
    \vspace{-10pt}
\end{table*}

\subsection{Pareto Optimality}

\Cref{fig:pareto} shows that VDGN unlocks regions of the error-cost landscape that are inaccessible to the class of Threshold policies in all of the mesh configurations that were tested.
We ran a sweep over refinement threshold $\theta_r \in [5\times 10^{-3},\dotsc,5\times 10^{-8}, 5\times 10^{-15}]$ with de-refinement threshold $\theta_d = 4\times 10^{-15}$.
In the case of $\text{depth}_{\max} = 1, 2$ with $500$ solver steps, and $\text{depth}_{\max}=1$ with $2500$ solver steps, \Cref{fig:pareto_iso_periodic_velunif_nx16_ny16_depth1_tstep0p25_vdn_graphnet_nodoftime_3_ep210800,fig:pareto_iso_periodic_velunif_nx8_ny8_depth2_tstep0p2_vdn_graphnet_nodoftime_3_ep265600,fig:pareto_iso_periodic_velunif_nx16_ny16_depth1_tstep0p25_vdn_graphnet_nodoftime_3_ep210800_tfinal5} show that VDGN lies outside the empirical Pareto front formed by threshold-based policies, and that VDGN Pareto-dominates those policies for almost every value of $\theta_r$: given a desired error (cost), VDGN has much lower cost (error).
The ``In-distribution'' group in \Cref{tab:efficiency} shows that VDGN has significantly higher efficiency than Threshold policies for all tested threshold values, for $\text{depth}_{\max} = 1,2$.

To understand the optimality of VDGN policies, we further compared multi-objective VDGN to brute-force search for the best sequence of refinement actions in an anisotropic 1D advection problem with $n_x =64, n_y=1$ and two mesh update steps.
To make brute-force search tractable, we imposed the constraint that a contiguous region of $n$ elements are refined at each step (while all elements outside the region are de-refined).
We searched for the starting locations of the region that resulted in lowest final global error.
By varying $n$, this procedure produces an empirical Pareto front of such brute-force policies in the error-cost landscape, which we plot in \Cref{fig:pareto_aniso_nx64_depth1}.
For multi-objective VDGN, we trained a single policy and evaluated it with $100$ randomly sampled preferences $\omega = [\alpha, 1-\alpha]$ where $\alpha \sim \text{Unif}[0,1]$.
% , thereby producing $100$ points in the error-cost landscape.
\Cref{fig:pareto_aniso_nx64_depth1} shows that a single multi-objective VDGN policy produces a Pareto front (\textcolor{red}{o}) that approaches the Pareto front formed by brute force policies (\textcolor{orange}{o}).
Moreover, we see that Threshold policies with various refinement thresholds (\textcolor{blue}{o}) are limited to a small section of the objective landscape, whereas VDGN unlocks previously-inaccessible regions.
% Again, we see that threshold-based policies are limited to a small section of the objective landscape, whereas VDGN unlocks previously-inaccessible regions.

\begin{figure}[ht]
\centering
\begin{subfigure}[t]{0.49\linewidth}
    \centering
    \includegraphics[width=0.49\linewidth]{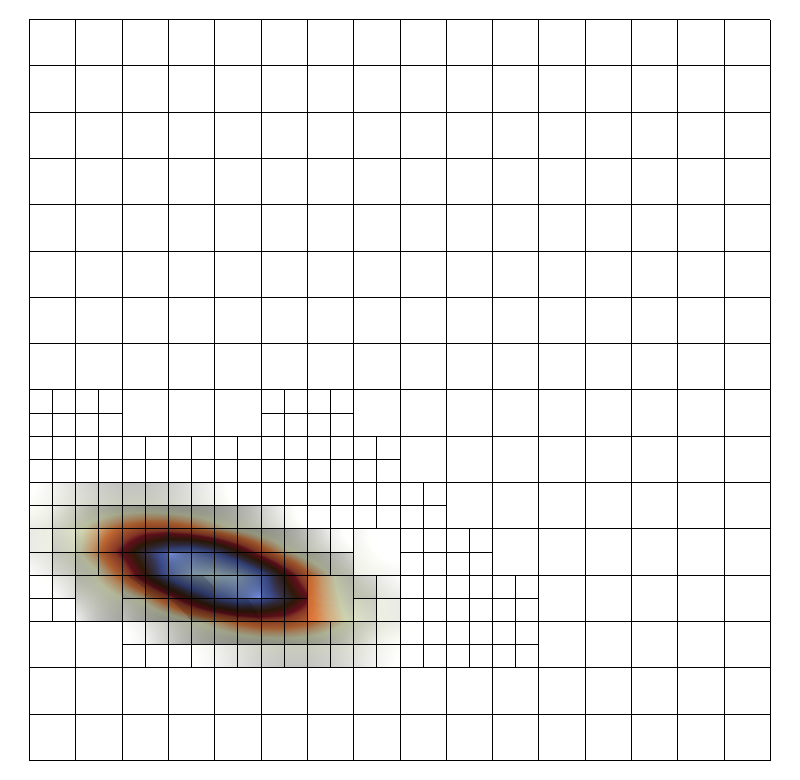}
    \includegraphics[width=0.49\linewidth]{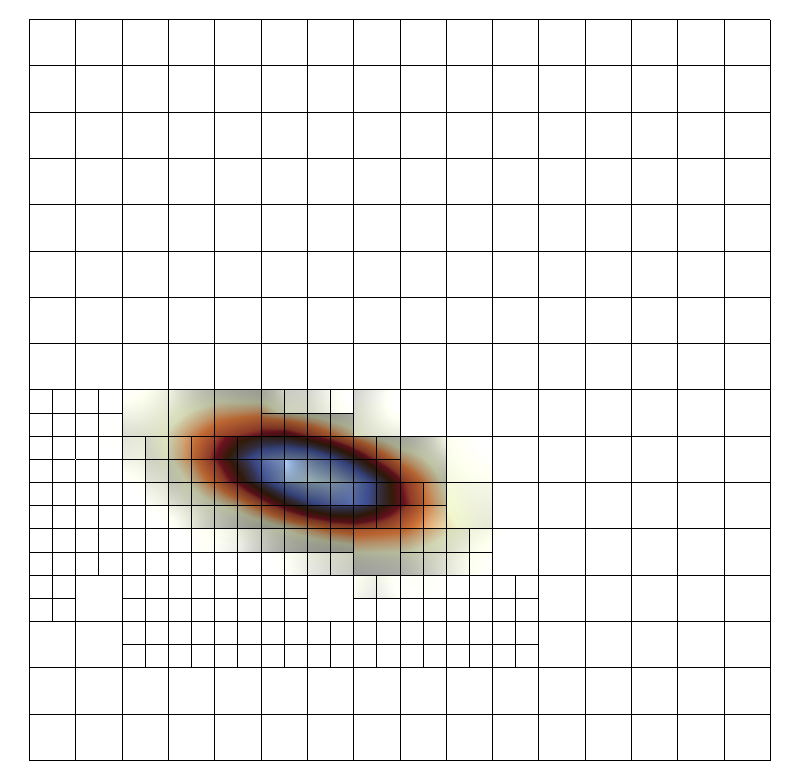}
    % \captionsetup{labelformat=empty}
    % \caption{$t=1$}
\end{subfigure}
\hfill
\begin{subfigure}[t]{0.49\linewidth}
    \centering
    \includegraphics[width=0.49\linewidth]{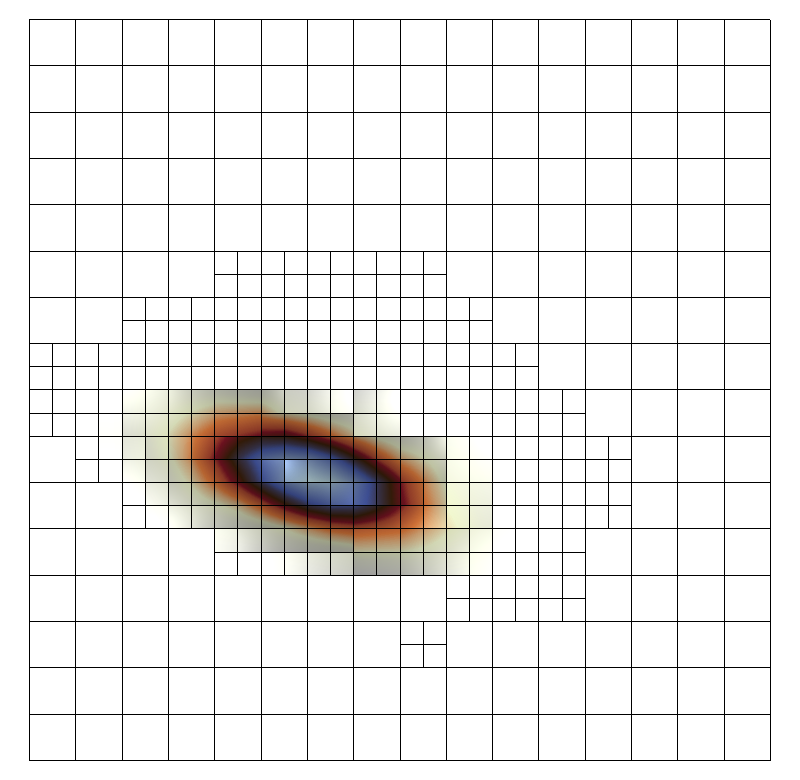}
    \includegraphics[width=0.49\linewidth]{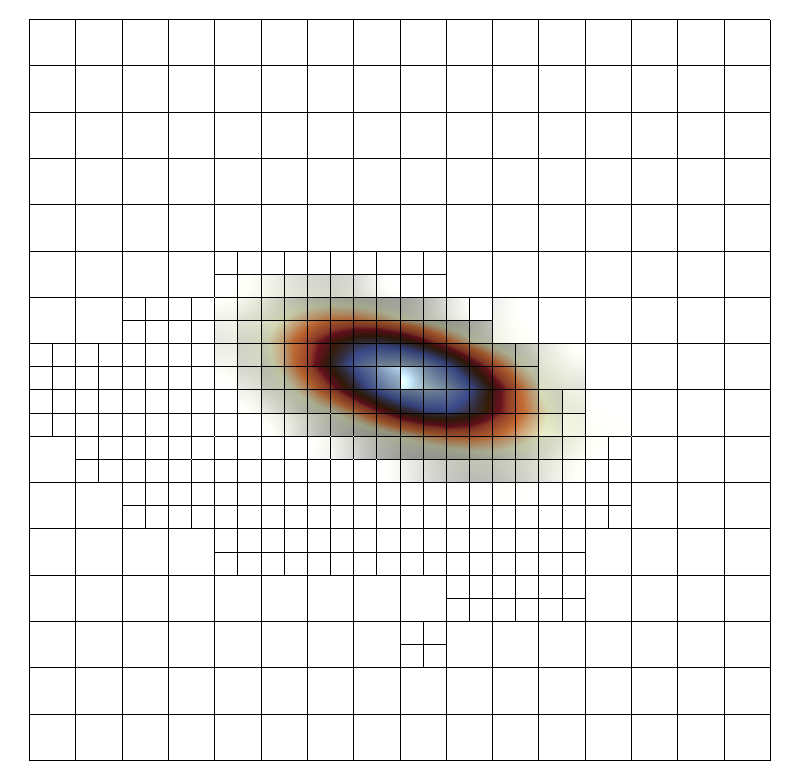}
    % \captionsetup{labelformat=empty}
    % \caption{$t=2$}
\end{subfigure}
\caption{Policy trained on isotropic 2D Gaussian can be applied to anisotropic 2D Gaussian.}
\label{fig:iso_periodic_velunif_nx16_ny16_depth1_tstep0p25_vdn_graphnet_nodoftime_3_ep210800_aniso}
\vspace{-15pt}
\end{figure}

\begin{figure}[ht]
\centering
\begin{subfigure}[t]{0.49\linewidth}
    \centering
    \includegraphics[width=0.49\linewidth]{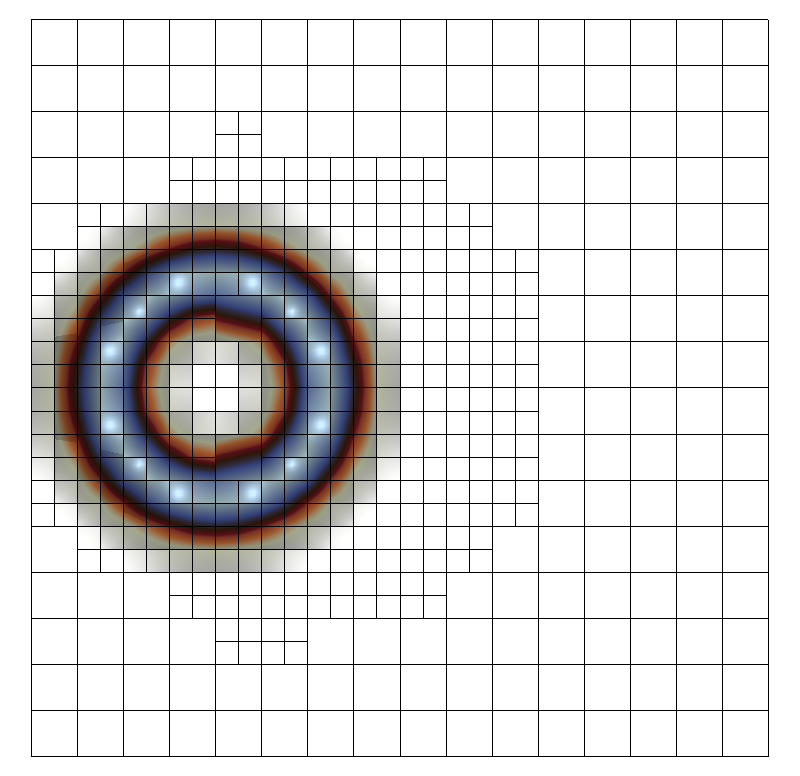}
    \includegraphics[width=0.49\linewidth]{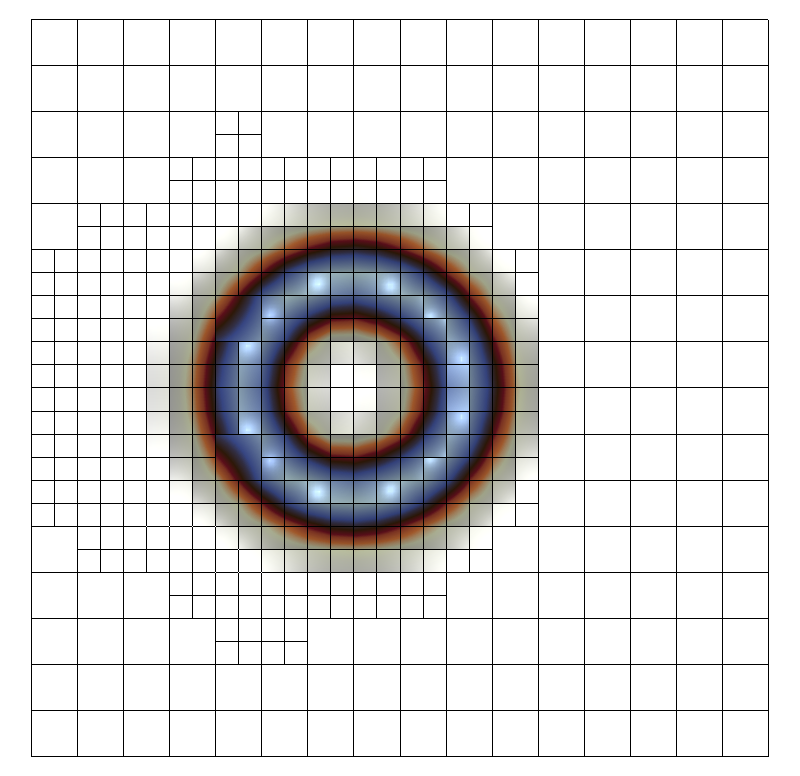}
    % \captionsetup{labelformat=empty}
    % \caption{$t=1$}
\end{subfigure}
\hfill
\begin{subfigure}[t]{0.49\linewidth}
    \centering
    \includegraphics[width=0.49\linewidth]{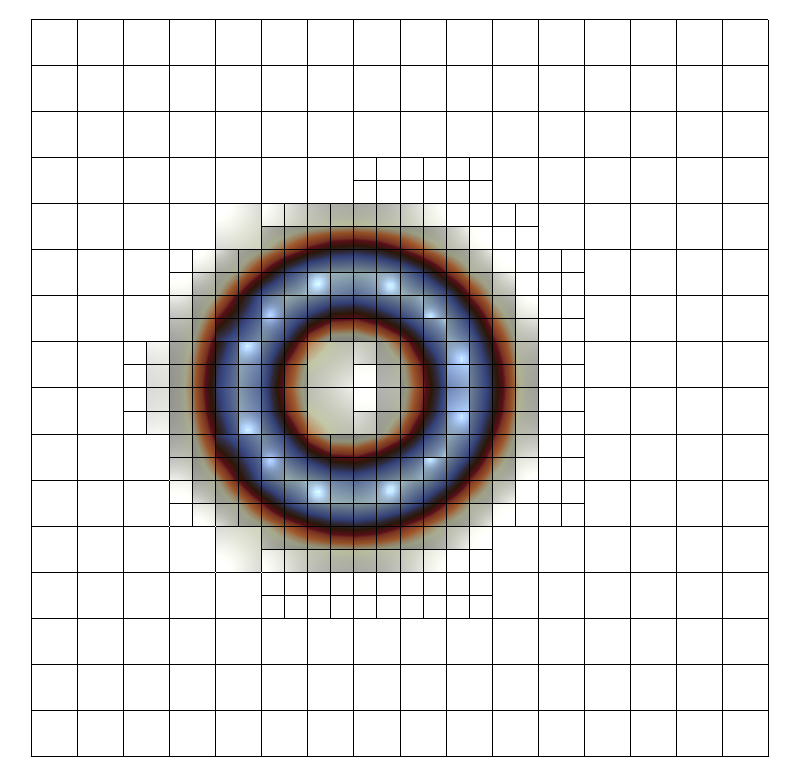}
    \includegraphics[width=0.49\linewidth]{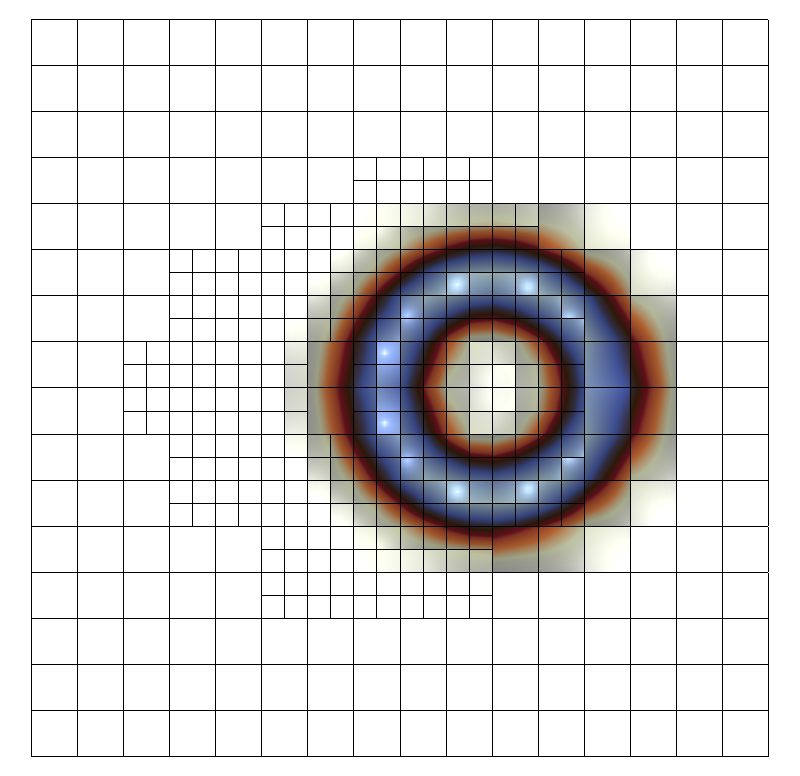}
    % \captionsetup{labelformat=empty}
    % \caption{$t=2$}
\end{subfigure}
\caption{Policy trained on isotropic 2D Gaussian can be applied to ring functions.}
\label{fig:iso_periodic_velunif_nx16_ny16_depth1_tstep0p25_vdn_graphnet_nodoftime_3_ep210800_ring}
\vspace{-15pt}
\end{figure}

\begin{figure}[ht]
\centering
\begin{subfigure}[t]{0.49\linewidth}
    \centering
    \includegraphics[width=0.49\linewidth]{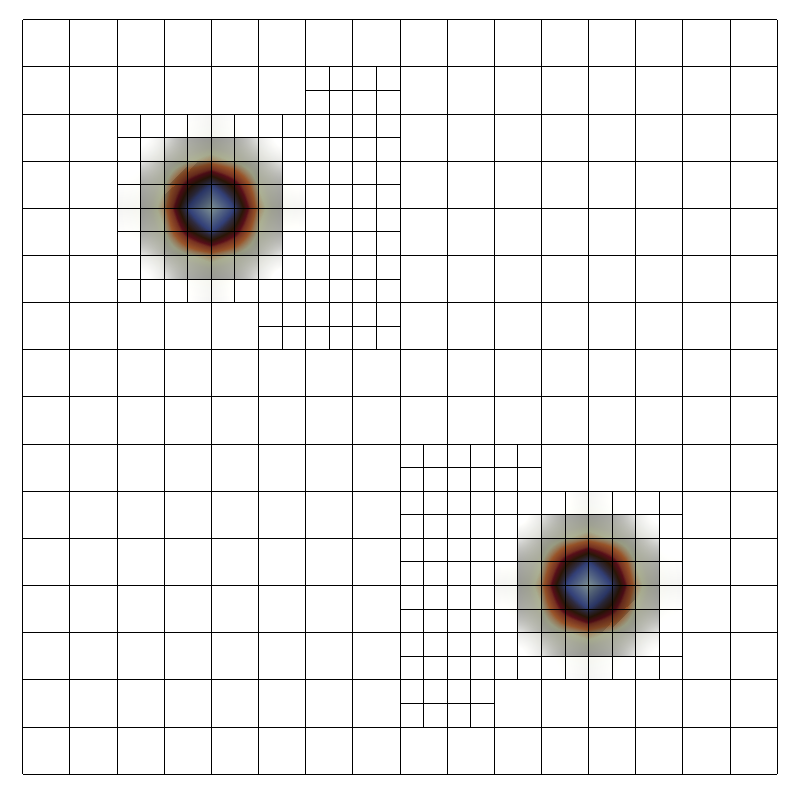}
    \includegraphics[width=0.49\linewidth]{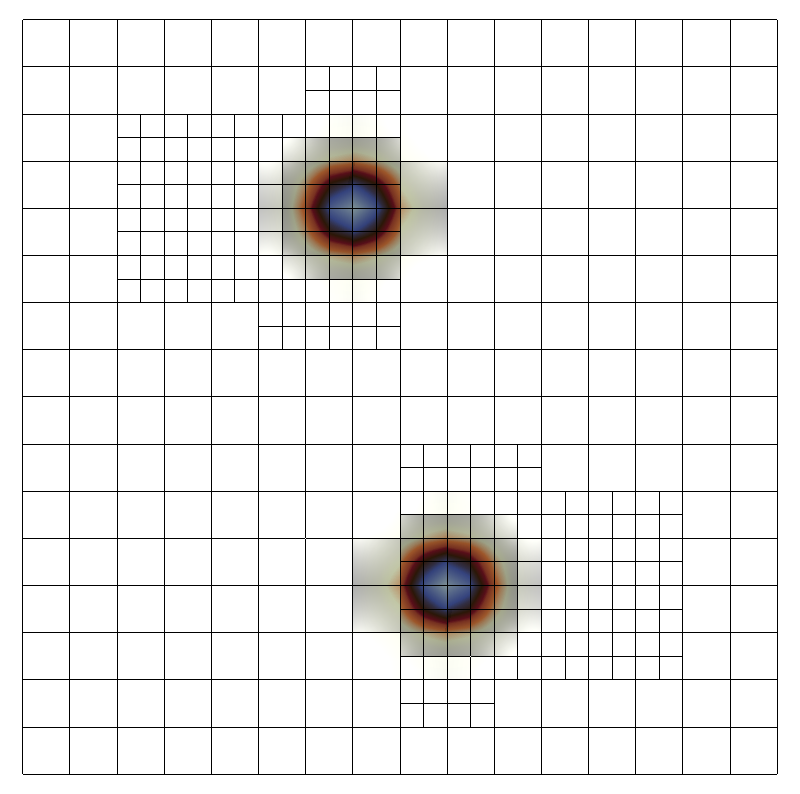}
    % \captionsetup{labelformat=empty}
    % \caption{$t=1$}
\end{subfigure}
\hfill
\begin{subfigure}[t]{0.49\linewidth}
    \centering
    \includegraphics[width=0.49\linewidth]{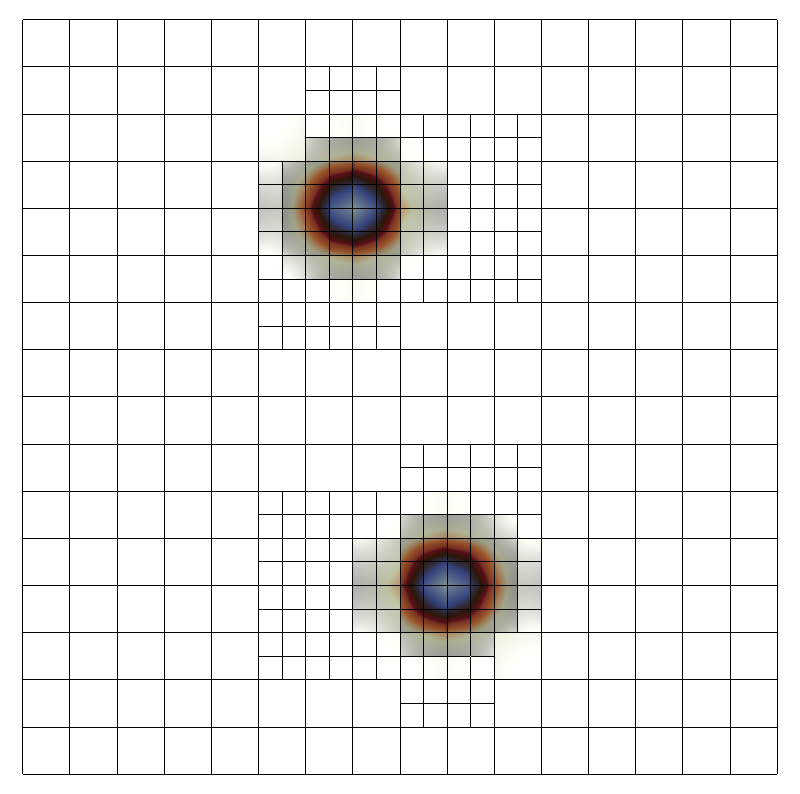}
    \includegraphics[width=0.49\linewidth]{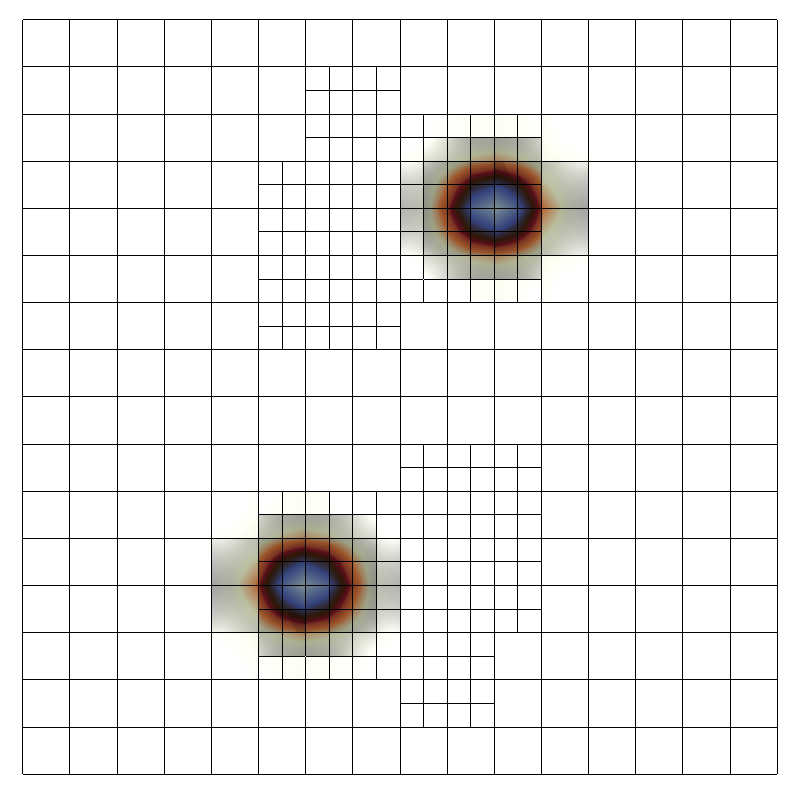}
    % \captionsetup{labelformat=empty}
    % \caption{$t=2$}
\end{subfigure}
\caption{Policy trained on one bump with one velocity can be applied to two bumps with opposite velocities.}
\label{fig:iso_periodic_velunif_nx16_ny16_depth1_tstep0p25_vdn_graphnet_nodoftime_3_ep210800_opposite}
\vspace{-15pt}
\end{figure}

\begin{figure}[ht]
\centering
\begin{subfigure}[t]{0.49\linewidth}
    \centering
    \includegraphics[width=0.49\linewidth]{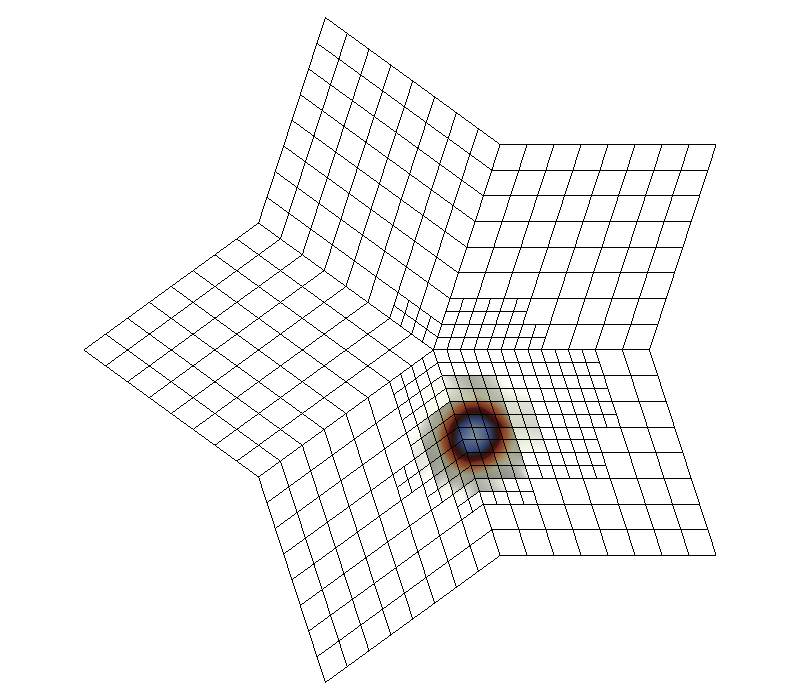}
    \includegraphics[width=0.49\linewidth]{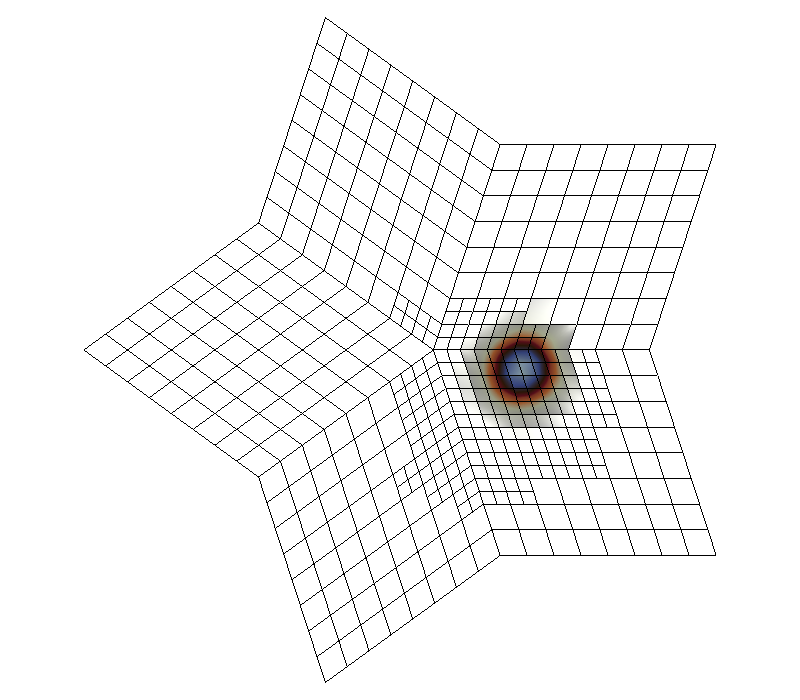}
    % \captionsetup{labelformat=empty}
    % \caption{$t=1$}
\end{subfigure}
\hfill
\begin{subfigure}[t]{0.49\linewidth}
    \centering
    \includegraphics[width=0.49\linewidth]{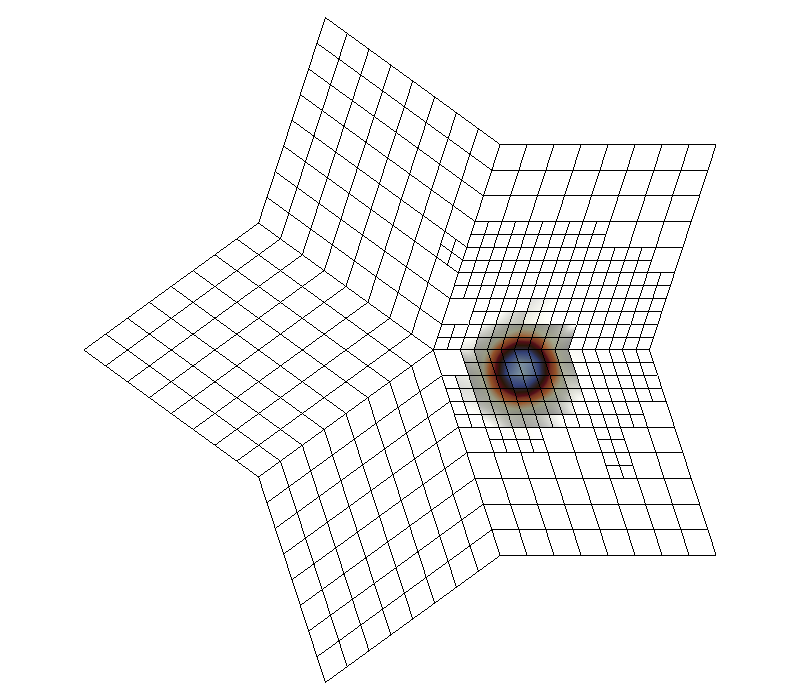}
    \includegraphics[width=0.49\linewidth]{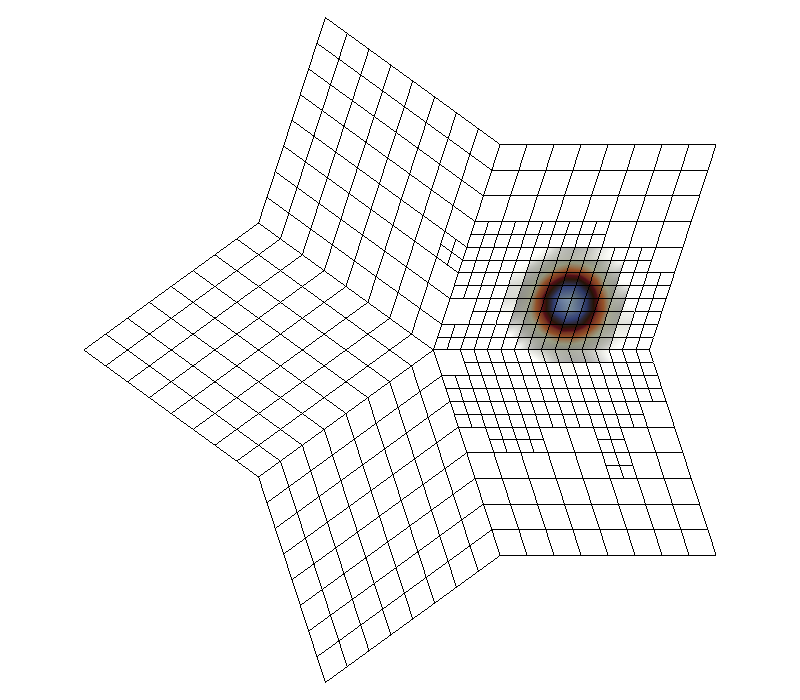}
    % \captionsetup{labelformat=empty}
    % \caption{$t=2$}
\end{subfigure}
\caption{Policy trained on square mesh, run on a star mesh.}
\label{fig:iso_periodic_velunif_nx16_ny16_depth1_tstep0p25_vdn_graphnet_nodoftime_3_ep210800_star}
\vspace{-15pt}
\end{figure}

\vspace{-5pt}
\subsection{Generalization}

% \textbf{From static functions to time-dependent problems}
% May not be necessary, if can show generalization between different time-dependent problem.

\textbf{Longer time}.
At training time for VDGN, each episode consisted of approximately 400-500 FEM solver steps.
We tested these policies on episodes with 2500 solver steps, which presents the agents with features outside of its training distribution due to accumulation of numerical error over time.
\Cref{tab:efficiency} shows that: a) VDGN maintain highest top performance; b) the performance ratio between VDGN and the best Threshold policy actually increases in comparison to the case with shorter time
(e.g., 1.47 in the ``Depth 1 2500 steps'' column, as opposed to 1.10 in the ``Depth 1'' column).
% (e.g., 0.840 vs. 0.565 in the ``Depth 1 2500 steps'' column, as opposed to 0.915 vs. 0.840 in the ``Depth 1'' column).
This is because the error of threshold-based policies accumulates quickly over time due to the lack of anticipatory refinement, whereas VDGN mitigates the effect.
\Cref{fig:iso_periodic_velunif_nx16_ny16_depth1_tstep0p25_vdn_graphnet_nodoftime_3_ep210800_tfinal5} shows that VDGN sustains anticipatory refinement behavior in test episodes longer than training episodes.

\textbf{Out-of-distribution test problems}.
Even though VDGN policies were trained on square meshes with 2D isotropic Gaussian waves, we find that they generalize well to initial conditions and mesh geometries that are completely out of the training distribution.
On anisotropic Gaussian waves (\Cref{fig:iso_periodic_velunif_nx16_ny16_depth1_tstep0p25_vdn_graphnet_nodoftime_3_ep210800_aniso}), ring-shaped features (\Cref{fig:iso_periodic_velunif_nx16_ny16_depth1_tstep0p25_vdn_graphnet_nodoftime_3_ep210800_ring}), opposite-moving waves (\Cref{fig:iso_periodic_velunif_nx16_ny16_depth1_tstep0p25_vdn_graphnet_nodoftime_3_ep210800_opposite}), star-shaped meshes (\Cref{fig:iso_periodic_velunif_nx16_ny16_depth1_tstep0p25_vdn_graphnet_nodoftime_3_ep210800_star}), VDGN significantly outperforms Threshold policies without any additional fine-tuning or training (see the ``Generalization'' group in \Cref{tab:efficiency}).
% Furthermore, VDGN is still competitive with Threshold policies on 3D meshes with spherical Gaussian waves, even though the 2D-trained policies do not observe any vectorial 3D information (recall that relational edge information contains the scalar dot product between velocity and element displacement).
\Cref{fig:iso_periodic_velunif_nx16_ny16_depth1_tstep0p25_vdn_graphnet_nodoftime_3_gen_triangle} shows qualitatively that a policy trained on quadrilateral elements shows rational refinement decisions when deployed on triangular elements.

\textbf{Scaling.}
Since VDGN is defined by individual node and edge computations with parameter-sharing across nodes and edges, it is a local model that is agnostic to size and scale of the global mesh.
\Cref{fig:iso_periodic_velunif_nx16_ny16_depth1_tstep0p25_vdn_graphnet_nodoftime_3_ep210800_nx64_ny64,fig:iso_periodic_velunif_nx16_ny16_depth1_tstep0p25_vdn_graphnet_nodoftime_3_ep210800_nx64_ny64_continued} show that a policy trained on $n_x = n_y = 16$ can be run with rational refinement behavior on an $n_x = n_y = 64$ mesh.

% \vspace{-8pt}
\section{Related work}
\label{sec:related_work}

A growing body of work leverage machine learning and deep neural networks \citep{goodfellow2016deep} to improve the trade-off between computational cost and accuracy of numerical methods: e.g., reinforcement learning for generating a fixed (non-adaptive) mesh \citep{pan2022reinforcement}, unsupervised clustering for marking and $p$-refinement \citep{tlales2022machine}, and
supervised learning for target resolution prediction \citep{obiols2022nunet}, error estimation \citep{wallwork2022e2n}, and mesh movement \citep{song2022m2n}.
The following three are the closest work to ours.
\citet{yang2021reinforcement} proposed a global single-agent RL approach for h-adaptive AMR.
It does not naturally support refining multiple elements per mesh update step, and anticipatory refinement was not conclusively demonstrated.
% \citep{yang2021reinforcement} formulated AMR with h-refinement as a sequential decision-making problem.
% The proposed global single-agent reinforcement learning approach acts on one element per step and does not easily extend to refining multiple elements at each mesh update step.
\citet{gillette2022learning} work within the class of marking policies parameterized by an error threshold and showed that single-agent RL finds robust policies that dynamically choose the error threshold and outperform fixed-threshold policies in elliptic problems.
However, threshold-based policies may not contain the optimal policy for time-dependent problems that require anticipatory refinement.
\citet{foucart2022deep} proposed a local single-agent RL approach whereby the agent makes a decision for one randomly-selected element at each step.
At training time, the global solution is updated every time a single element action occurs; at test time, the agent faces a different environment transition since the global solution is updated only after it has acted for all elements.
Our multi-agent approach enables the definition of the environment transition to be the same at training and test time.

\vspace{-5pt}
\section{Conclusion}
\label{sec:conclusion}

We have formulated a Markov game for adaptive mesh refinement, shown that centralized training addresses the posthumous credit assignment problem, and proposed a novel multi-agent reinforcement learning method called Value Decomposition Graph Network (VDGN) to train AMR policies directly from simulation.
VDGN displays anticipatory refinement behavior, enabling it to unlock new regions of the error-cost objective landscape that were inaccessible by previous threshold-based AMR methods.
We verified that trained policies work well on out-of-distribution test problems with PDE features, mesh geometries, and simulation duration not seen in training.
Our work serves as a stepping stone to apply multi-agent reinforcement learning to more complex problems in AMR.

%%%%%%%%%%%%%%%%%%%%%%%%%%%%%%%%%%%%%%%%%%%%%%%%%%%%%%%%%%%%%%%%%%%%%%%%

%%% The acknowledgments section is defined using the "acks" environment
%%% (rather than an unnumbered section). The use of this environment 
%%% ensures the proper identification of the section in the article 
%%% metadata as well as the consistent spelling of the heading.

\begin{acks}
The authors thank Andrew Gillette for helpful discussions in the project overall.
This work was performed under the auspices of the U.S. Department of Energy by Lawrence Livermore National Laboratory under contract DE-AC52-07NA27344 and the LLNL-LDRD Program with project tracking \#21-SI-001.
LLNL-PROC-841328.
\end{acks}

%%%%%%%%%%%%%%%%%%%%%%%%%%%%%%%%%%%%%%%%%%%%%%%%%%%%%%%%%%%%%%%%%%%%%%%%

%%% The next two lines define, first, the bibliography style to be 
%%% applied, and, second, the bibliography file to be used.
\bibliographystyle{ACM-Reference-Format}
\balance
% \bibliography{sample}
\bibliography{citation_marl.bib, citation_ml_math.bib, citation_nn.bib, citation_rl.bib}

%%%%%%%%%%%%%%%%%%%%%%%%%%%%%%%%%%%%%%%%%%%%%%%%%%%%%%%%%%%%%%%%%%%%%%%%

\clearpage

\appendix

\section{Additional details on methods}
\label{app:methods}

\subsection{Multi-head attention}
\label{app:multi-head}

\textbf{Multi-head graph attention layer.}
We extend graph attention by building on the mechanism of multi-head attention \citep{vaswani2017attention}, which uses $H$ independent linear projections $(W^{q,h}, W^{k,h}, W^{v,h}, W^{e,h})$ for queries, keys, values and edges (all projected to dimension $d/H$), and results in $H$ independent sets of attention weights $\ahat^{ij,h}$, with $h = 1,2,\dotsc,H$ .
This enables the model to attend to different learned representation spaces and was previously found to stabilize learning \citep{velivckovic2018graph}.
The new node representation with multi-head attention is the concatenation of all output heads, with an output linear projection $W^o \in \Rbb^{d \times d}$.
Hence, \cref{eq:attention_dot_product,eq:attention_weights,eq:attention_value,eq:node_update} are extended to be:
\begin{align}\label{eq:node_update_multihead}
    \ahat^{ij,h} &\defeq \frac{\exp(W^{q,h}v^i \cdot W^{k,h}v^j)}{\sum_{l \in \Ncal_i} \exp(W^{q,h}v^i \cdot W^{k}v^l)} && \text{for } (i,j) \in E, h \in [H] \, ,\\
    b^{ij,h} &\defeq W^{v,h} v^j + W^{e,h} e^{ij} && \text{for } (i,j) \in E, h \in [H] \, , \\
    \vhat^i &\defeq W^o \left( \Big\Vert_{h=1}^H \sum_{j \in \Ncal_i} \ahat^{ij,h} b^{ij,h} \right) && \text{for } i \in V \, .
\end{align}

\subsection{Improvements to value-based learning}
\label{app:value_learning_improvements}

Standard Deep Q Network (DQN) \citep{mnih2015human} can be improved with double Q-learning, dueling network, and prioritized replay.
Similarly, these improvements can be applied to VDN and VDGN as well.

\textbf{Double Q-learning.}
DQN maintains a main network $Q_{\theta}$ with trainable parameters $\theta$ and a target network $Q_{\theta'}$ with separate trainable parameters $\theta'$, which is used in the TD-target
\begin{align}
    y_{t+1} \defeq R_t + \gamma \max_{a} Q_{\theta'}(s_{t+1},a) \, .
\end{align}
After every update of the main parameters $\theta$, the target parameters $\theta'$ are slowly updated via $\theta' \leftarrow \lambda \theta + (1-\lambda) \theta'$.
In contrast, Double DQN \citep{van2016deep,hasselt2010double} uses the main network to compute the greedy policy, then uses the target network to evaluate the value of the greedy policy:
\begin{align}
    y_{t+1} := R_t + \gamma Q_{\theta'}(s_{t+1}, \argmax_{a} Q_{\theta}(s_{t+1},a)) \, .
\end{align}
This alleviates the overestimation issue caused by using a single network for both selecting and evaluating the greedy policy.
In the context of VDN and VDGN, we compute the TD target by using the target network $Q^i_{\theta'}$ to evaluate the greedy joint actions, which are selected by the main network $Q^i_{\theta}$:
\begin{align}\label{eq:ddqn-vdn}
    y_{t+1} &\defeq R_t + \gamma Q_{\theta'}(s_{t+1}, a)\vert_{a = [\argmax_{a^i} Q^i_{\theta}(s^i_{t+1},a^i)]} \\
    Q_{\theta}(s,a) &\defeq \sum_{i=1}^n Q^i_{\theta}(s^i,a^i) \quad Q_{\theta'}(s,a) \defeq \sum_{i=1}^n Q^i_{\theta'}(s^i,a^i) \, .
\end{align}

\textbf{Dueling Network.}
In DQN, the neural network takes the state $s$ as input and has $i=1,\dotsc,\lvert \Acal \rvert$ output nodes, each of which represents the scalar $Q(s,a=i)$ for $\lvert \Acal \rvert$ discrete actions.
Dueling networks \citep{wang2016dueling} use the fact that sometimes it suffices to estimate the value of a state without estimating the value of each action at the state, which implies that learning a value function $V^{\pi}(s,a)$ is enough.
Practically, this separation of state-action values from state values is achieved by parameterizing $Q_{\phi,\alpha,\beta}(s,a)$ as a sum of value function $V_{\phi,\beta}(s)$ and advantage function $A_{\phi,\alpha}(s,a)$:
\begin{align}
    Q_{\phi,\alpha,\beta}(s,a) = V_{\phi,\beta}(s) + A_{\phi,\alpha}(s,a) \, ,
\end{align}
where $\phi$ are shared parameters while $\alpha$ and $\beta$ are separate parameters specific to $A$ and $V$, respectively.
In practice, the expression used is
\begin{align}\label{eq:dueling}
    Q_{\phi,\alpha,\beta}(s,a) = V_{\phi,\beta}(s) + \left( A_{\phi,\alpha}(s,a) - \frac{1}{\lvert \Acal \rvert} \sum_{a'} A_{\phi,\alpha}(s,a') \right) \, ,
\end{align}
for stability and identifiability reasons.
\Cref{eq:dueling} is used as the form for both the main network with parameters $\theta \defeq (\phi, \alpha, \beta)$ and target network with parameters $\theta' \defeq (\phi',\alpha',\beta')$.
\Cref{eq:dueling} is directly compatible with Double DQN \Cref{eq:ddqn-vdn}.
It is also directly compatible with VDGN, where we use two separate final graph attention layers to output $V(s^i)$ and $A(s^i,a^i)$ for each node $i$.

\textbf{Prioritized replay.}
DQN stores the online experiences into a finite rolling replay buffer $\Bcal$ and periodically trains on batches of transitions that are sampled uniformly at random from $\Bcal$.
However, the frequency of experiencing a transition (which determines the count of occurrence in $\Bcal$) should not necessarily determine the probability of sampling and training on that sample.
This is because some transitions may have high occurrence frequency but low TD-error (if the Q values of the state are already well approximated), whereas other important transitions with high TD-error (due to insufficient training) may have low occurrence frequency in $\Bcal$.
Prioritized replay \citep{schaul2015prioritized} defines the priority $p_i$ of a transition $i = (s_i,a_i,r_i,s'_i)$ based on its TD-error $\delta_i \defeq y_i - Q(s_i,a_i)$, and defines the probability of sampling transition $i$ as
\begin{align}
    P(i) = \frac{p^{\alpha}_i}{\sum_k p^{\alpha}_k}
\end{align}
where $\alpha$ is a hyperparameter that anneals between uniform sampling ($\alpha=0$) and always sampling the highest priority transition ($\alpha \rightarrow \infty$).
We use the rank-based variant of prioritized replay, whereby $p_i \defeq \frac{1}{\text{rank}(i)}$ and $\text{rank}(i)$ is the rank of $i$ when buffer $\Bcal$ is sorted according to $\lvert \delta_i \rvert$.
We also use importance sampling correction weights $w_i \defeq (1/N 1/P(i))^{\beta}$, so that the agents learn with $w_i \delta_i$ rather than $\delta_i$.
This is an orthogonal improvement to any replay-based off-policy reinforcement learning method, and hence is directly compatible with both Double DQN and dueling network.

\subsection{Multi-objective VDGN}
\label{app:multi-objective}

Working in the context of linear preferences,
the goal is to find a set of policies $\Pi^*$ corresponding to the convex coverage set of the Pareto front, i.e., $\Pi^* \defeq \lbrace \pi  \in \Pi \colon \exists \omega \in \Omega, \forall \pi' \in \Pi, \omega^T J^{\pi} \geq \omega^T J^{\pi'} \rbrace$.
\begin{align}
    \Pi^* \defeq \lbrace \pi  \in \Pi \colon \exists \omega \in \Omega, \forall \pi' \in \Pi, \omega^T J^{\pi} \geq \omega^T J^{\pi'} \rbrace
\end{align}
To do so, we extend VDGN with Envelop Q-learning \citep{yang2019generalized}, a multi-objective RL method that efficiently finds the convex envelope of the Pareto front in multi-objective MDPs.
Envelop Q-learning searches over the space of already-learned preferences $\omega' \in \Omega$ to find the best TD target $\ybf(\omega)$ for a given preference $\omega$.
The single-objective VDN update equations \cref{eq:q-learning,eq:td-target} become
\begin{align}\label{eq:multiobj_vdn}
    \theta &\leftarrow \theta - \nabla_{\theta}\Ebb_{s_t,a_t,s_{t+1}}\left[ \lVert \ybf_{t+1} - \Qbf_{\theta}(s_t,a_t)\rVert^2_2 \right] \\
    \ybf_{t+1}(\omega) &\defeq \Rbf_t + \gamma \arg_Q \max_{a,\omega'} \omega^T \Qbf(s_{t+1},a,\omega') \\
    \max_{a,\omega'} \omega^T& \Qbf(s,a,\omega') = \omega^T\Qbf(s,a)\vert_{a = [\underset{a^i,\omega'}{\argmax} \omega^T \Qbf^i(s^i,a^i,\omega')]_{i=1}^n}
\end{align}
where $\arg_Q \underset{a,\omega'}{\max} \omega^T \Qbf(s_{t+1},a,\omega')$ extracts the vector $\Qbf(s_{t+1},a^*,\omega^*)$ such that $(a^*,\omega^*)$ achieves the $\argmax$.
Note that the global $\argmax$ is still efficiently computable via value decomposition, since
\begin{align}
&\max_{a} \omega^T Q(s,a) = 
% \max_{a} \sum_{i=1}^{n} \omega^T Q^i(s,a^i) 
= \sum_{i=1}^n \max_{a^i} \omega^T Q^i(s,a^i) \\
&\Rightarrow \argmax_{a} \omega^T Q = \left[ \argmax_{a^i} \omega^T Q^i(s,a^i) \right]_{i=1}^n
\end{align}

\section{Proofs}
\label{app:proofs}

\AgentCreationMarkov*
\begin{proof}
We prove this by construction.
First define a dummy individual agent state denoted $s_{\nexists}$ that is interpreted semantically as ``nonexistence''.
Define the global state space $\Scal$ to be $\Scal \defeq \prod_{i=1}^{N_{\max}} (\Scal^i \cup \lbrace s_{\nexists} \rbrace )$.
Hence, any arbitrary state $s_t$, with $n_t$ number of agents, belongs to $\Scal$.
More specifically, it belongs to the subspace
$\left( \prod_{i=1}^{n_t} \Scal^i \right) \times \left( \prod_{i=n_t+1}^{N_{\text{max}}} \lbrace s_{\nexists} \rbrace \right) \subset \Scal$.

Next, define the joint action space $\Acal$ to be $\Acal \defeq \prod_{i=1}^{N_{\max}} \Acal^i$.
At a time step during an episode when there are $n_t$ elements, the effectively available action space is just $\prod_{i=1}^{n_t} \Acal^i$, meaning that the joint action is written as $a = (a^1,\dotsc,a^{n_t},\underbrace{\text{no-op},\dotsc,\text{no-op}}_{N_{\max}-n_t})$,
where $\text{no-op}$ is ignored by the reward and transition function.
The reward for the joint state-action space is defined by 
\begin{align*}
&R\bigl((s^1,\dotsc,s^{n_t},\underbrace{s_{\nexists},\dotsc,s_{\nexists}}_{N_{\max}-n_t}), (a^1,\dotsc,a^{n_t},\underbrace{\text{no-op},\dotsc,\text{no-op}}_{N_{\max}-n_t})\bigr) \\
&= R\bigl( (s^1,\dotsc,s^{n_t}),(a^1,\dotsc,a^{n_t})\bigr) \, ,
\end{align*}
where we overload the use of notation $R$ since $R$ is already well-defined for variable number of agents.
Similarly, the transition function $P$ that was already well-defined for variable number of agents induces a transition function for the new state-action spaces $\Scal$ and $\Acal$.
Hence the stationary Markov game is defined by the tuple $\left( \Scal, \Acal, R, P, \gamma, N_{\max} \right)$.
This completes the construction.
\end{proof}

\Symmetry*
\begin{proof}
This is proved by induction on the number of graph attention layers.

\textit{Base case}.
In the original PDE, let $X = (x^1,\dotsc,x^n)$ denote the coordinates of elements $i=1,\dotsc,n$, let $v = (v^1,\dotsc,v^n)$ denote the input node features, and let $e = \lbrace e^{ij} \rbrace_{(i,j) \in E}$ denote the input edge features.
Let element $i$ with coordinate $x^i$ have node feature $v^i$ and edge feature set $\lbrace e^{ij} \rbrace_{j \in \Ncal(i)}$.
Suppose that the first layer of VDGN produces updated node features $\vhat = (\vhat^1, \dotsc, \vhat^n)$.
Any rotation can be represented by matrix multiplication with an orthogonal matrix $Q \in \Rbb^{k \times k}$, where $k$ is the number of spatial dimensions.
Upon rotation of the PDE solution and velocity field, we check that the element at coordinate $Qx^i$ produces the same node value as the value $\vhat^i$ that element $i$ would have produced without the rotation.
First, note that the displacement vector for any edge $e$ and  the velocity vector transform as $\hat{d}^e = Qd^e$ and $\qhat = Qq$, respectively.
The node feature at $Qx^i$ is $v^i$, since the PDE solution has been rotated.
The edge feature set at $Qx^i$ is $\lbrace e^{ij} \rbrace_{j \in \Ncal(i)}$ because:
a) by definition of rotation of the PDE, initial conditions are also rotated and hence the one-hot relative depth observation $\mathbf{1}[\text{depth}(i) - \text{depth}(j)]$ is unchanged;
b) the inner product part of the edge feature is unchanged because $\langle \dhat^e, \qhat \rangle = \langle Q d^e, Qq \rangle = \langle d^e, q \rangle$ because orthogonality of $Q$ implies $Q^TQ = 1$.
Hence the input node and edge features at $Qx^i$ in the rotated case are equal to those for element $i$ in the original case, so the updated node feature at $Qx^i$ is still $\vhat^i$.
For translation, which is represented by addition with some vector $b \in \Rbb^k$, the same reasoning applies by noting that displacements and velocities are invariance to addition, so the element at $x^i + b$ in the translated case has the same node and edge observations as the element $x^i$ in the original case, so the updated node feature at $x^i + b$ is still $\vhat^i$.

\textit{Inductive step}.
Suppose the claim holds for layer $l-1$.
This means that the node representation $\vhat^{l-1}$ (used as input to layer $l$) at element $Qx^i$ in the transformed PDE is equal to the input node representation at element $x^i$ in the original PDE.
The input edge representation is still $\lbrace e^{ij} \rbrace_{j \in \Ncal(i)}$ since edge representations are not updated, and the reasoning in the base case for edges applies.
Hence, layer $l$'s inputs in the transformed case are the same as its inputs in the original case, so the outputs $\vhat^l$ are also the same.
This completes the proof for equivariance.

\textit{Time invariance}.
This follows immediately from the fact that absolute time is not used as an observable, and the fact that VDGN learns a Markov policy.
\end{proof}

\section{Additional implementation details}
\label{app:implementation}

\textbf{Ancillary algorithmic details.}
Every $t_{\text{train}}$ environment steps, we sample batch size $M$ transitions from the replay buffer $\Bcal$ and conduct one training step.
We used the Adam optimizer with learning rate $\eta$ in Tensorflow \citep{abadi2016tensorflow}.
Each agent has its own independent $\epsilon$-greedy exploration with $\epsilon$ linearly decaying from $\epsilon_{\text{start}}$ to $\epsilon_{\text{end}}$ within $n_{\epsilon}$ episodes.
Along with prioritized replay, we used a lower bound of $p_{\text{unif}}$ on the probability of sampling each transition from the replay buffer.
Hyperparameters for all algorithmic settings are listed in \Cref{tab:hyperparam_alg}.
We ran a simple population-based elimination search with 128 starting samples on the final exploration rate $\epsilon_{text{end}}$, learning rate $\eta$, and target update rate $\tau$ on one $16 \times 16$ mesh, then used those hyperparameters for all experiments.
We used the same prioritized replay hyperparameters as the original paper. 
Graph attention hyperparameters $L$ and $R$ should be based on the spatial domain of influence (e.g., number of elements along a direction) to be included to inform a refinement decision.

\begin{table}[t]
    \centering
    \caption{Algorithm hyperparameters}
    \begin{tabular}{lcr}
    \toprule
        Name & Symbol & Value \\
        \midrule
        Batch size & $M$ & 16 \\
        Replay buffer size & $\lvert \Bcal \rvert$ & 10000 \\
        Exploration decay episodes & $n_{\epsilon}$ & 150000 \\
        Exploration end & $\epsilon_{\text{end}}$ & 0.01 \\
        Exploration start & $\epsilon_{\text{start}}$ & 0.5 \\
        Learning rate & $\eta$ & $5 \times 10^{-4}$ \\
        Prioritized replay exponent & $\alpha$ & 0.5 \\
        Prioritized replay & \multirow{2}{*}{$\beta$} & \multirow{2}{*}{0.4} \\
        \quad importance exponent & & \\
        Env steps per train & $t_{\text{train}}$ & 4 \\
        Target update rate & $\lambda$ & 0.0112 \\
        Uniform sampling probability & $p_{\text{unif}}$ & 0.001 \\
        Attention layer size & $d$ & 64 \\
        Number of attention heads & $H$ & 2 \\
        Number of attention layers & $L$ & 2 \\
        \multirow{2}{*}{Number of recurrent passes} & \multirow{2}{*}{$R$} & depth 1: 3 \\ 
        & & depth 2: 2 \\
        \bottomrule
    \end{tabular}
    \label{tab:hyperparam_alg}
\end{table}

\begin{table}[t]
    \centering
    \caption{Environment settings}
    \begin{tabular}{lcr}
    \toprule
        Name & Symbol & Value \\
        \midrule
        DoF threshold & $d_{\text{thres}}$ & 5620 \\
        Initial error threshold & $c_{\text{thres}}$ & $5 \times 10^{-4}$ \\
        Solver time increment & $d\tau$ & $2 \times 10^{-3}$ \\
        \multirow{2}{*}{Initial mesh partition} & \multirow{2}{*}{$(n_x, n_y)$} &  depth 1: $(16, 16)$ \\ 
        & & depth 2: $(8, 8)$ \\
        Training mesh dimension & $[s_x, s_y]$ & $[0, 2.0]^2$ \\
        \multirow{2}{*}{Solver time per mesh update} & \multirow{2}{*}{$\tau_{\text{step}}$} & depth 1: 0.25 \\  
        & & depth 2: 0.2 \\
        \multirow{2}{*}{Training max simulation time} & \multirow{2}{*}{$\tau_f$} & depth 1: 0.75 \\ 
        & & depth 2: 0.8 \\
        \bottomrule
    \end{tabular}
    \label{tab:env_settings}
\end{table}

\begin{figure*}[t]
\centering
\begin{subfigure}[t]{0.33\linewidth}
    \centering
    \includegraphics[width=0.49\linewidth]{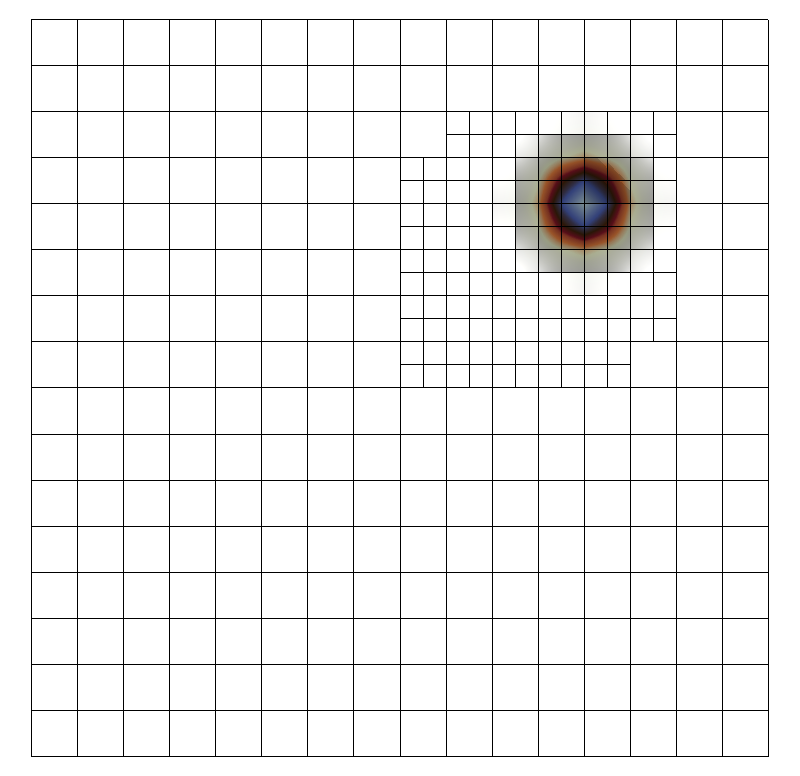}
    \includegraphics[width=0.49\linewidth]{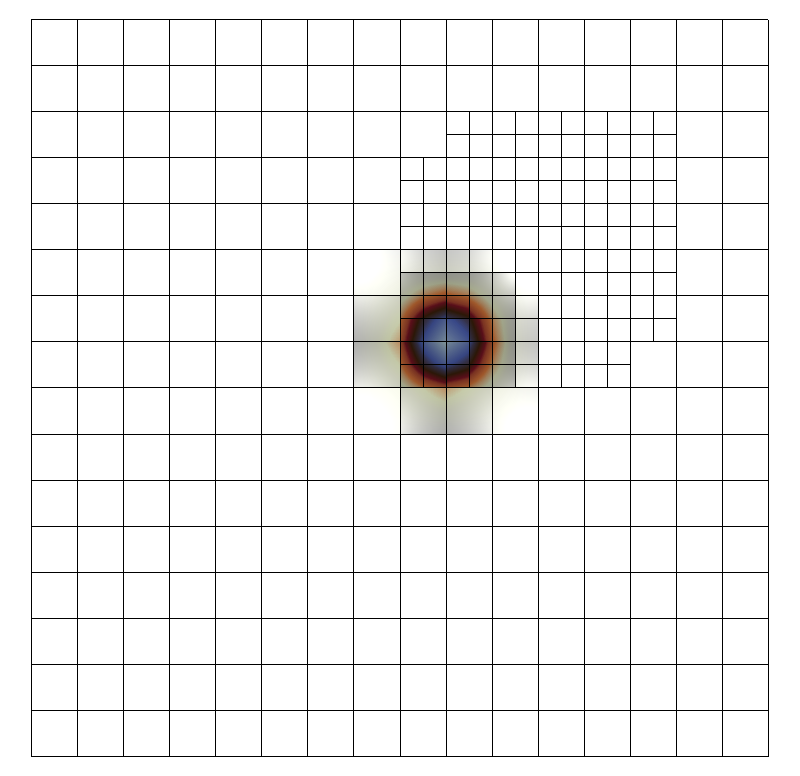}
    \captionsetup{labelformat=empty}
    \caption{$t=1$}
\end{subfigure}
\hfill
\begin{subfigure}[t]{0.33\linewidth}
    \centering
    \includegraphics[width=0.49\linewidth]{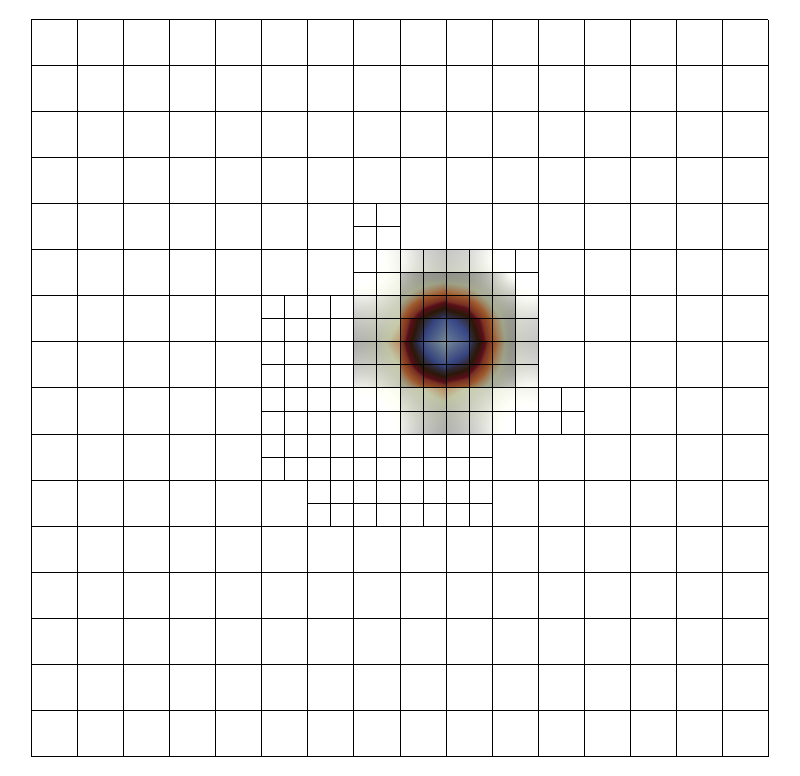}
    \includegraphics[width=0.49\linewidth]{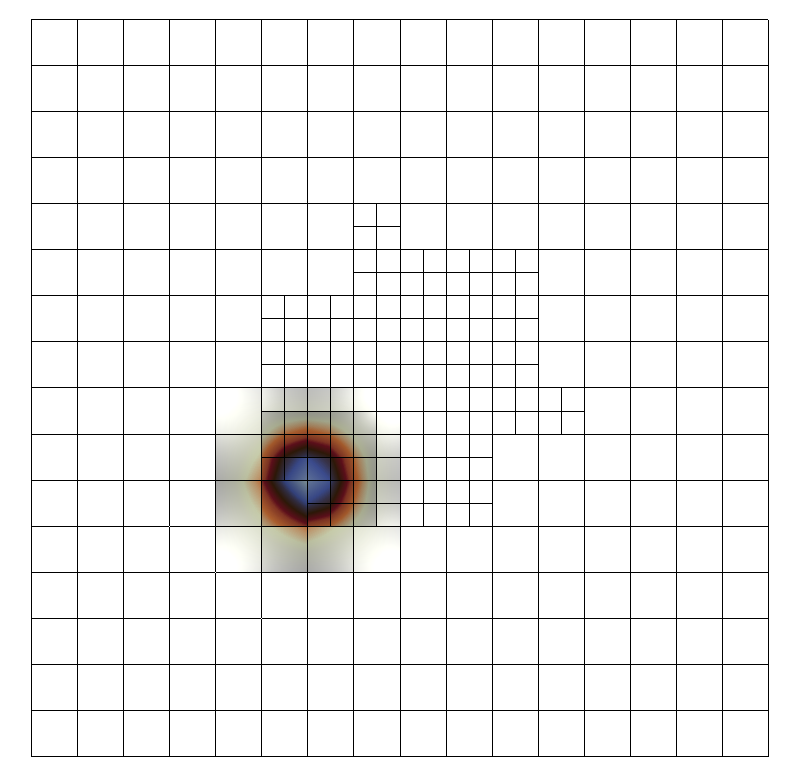}
    \captionsetup{labelformat=empty}
    \caption{$t=2$}
\end{subfigure}
\hfill
\begin{subfigure}[t]{0.33\linewidth}
    \centering
    \includegraphics[width=0.49\linewidth]{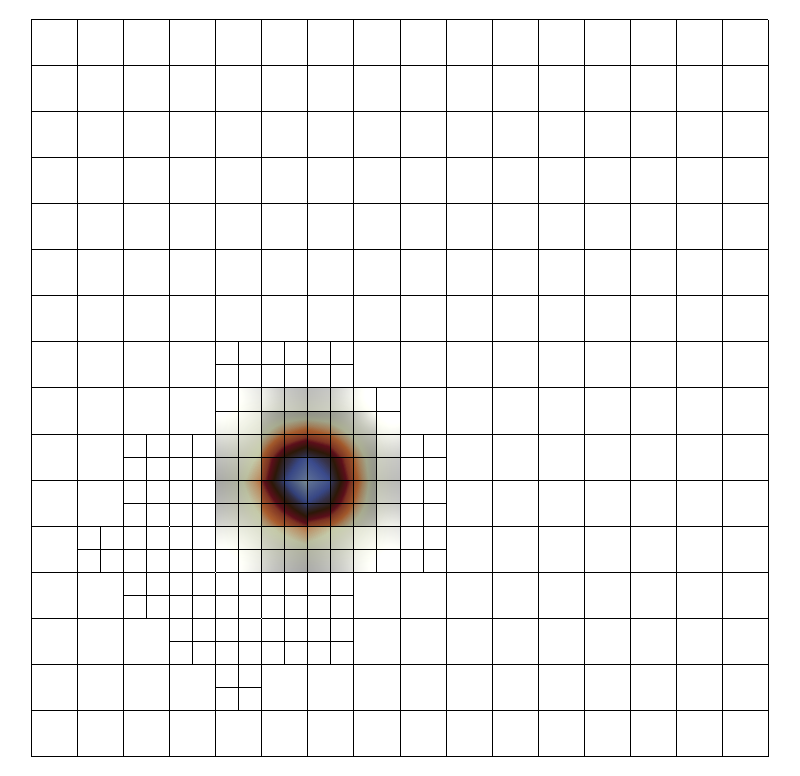}
    \includegraphics[width=0.49\linewidth]{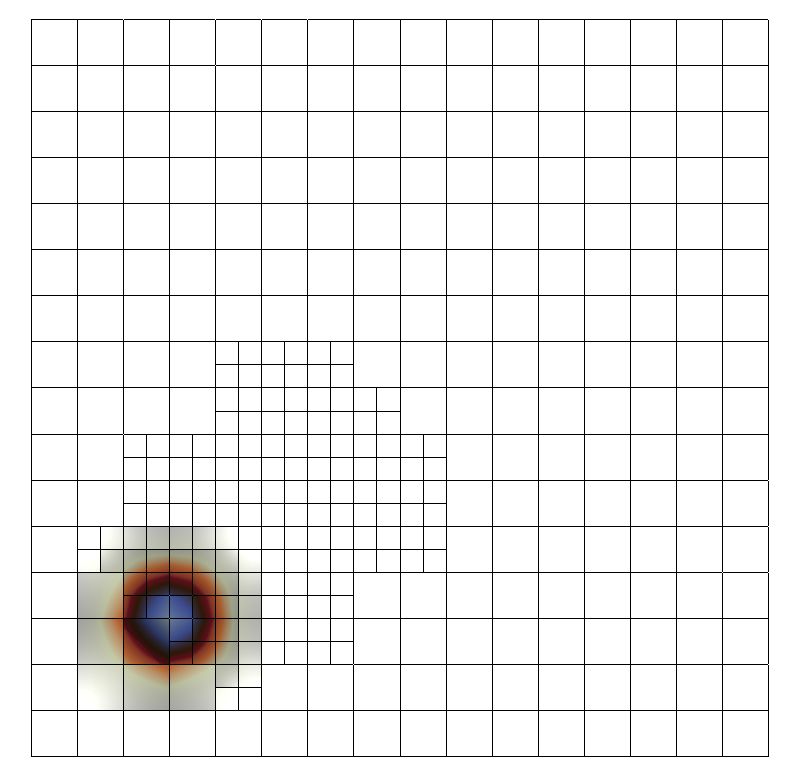}
    \captionsetup{labelformat=empty}
    \caption{$t=3$}
\end{subfigure}
\caption{Global refinements by VDGN are equivariant to rotation of the propagating feature. Compare to \Cref{fig:iso_periodic_velunif_nx16_ny16_depth1_tstep0p25_vdn_graphnet_nodoftime_3_ep210800}}
\label{fig:iso_periodic_velunif_nx16_ny16_depth1_tstep0p25_vdn_graphnet_nodoftime_3_ep210800_degree225}
\end{figure*}

\begin{figure*}[t]
\centering
\begin{subfigure}[t]{0.33\linewidth}
    \centering
    \includegraphics[width=0.49\linewidth]{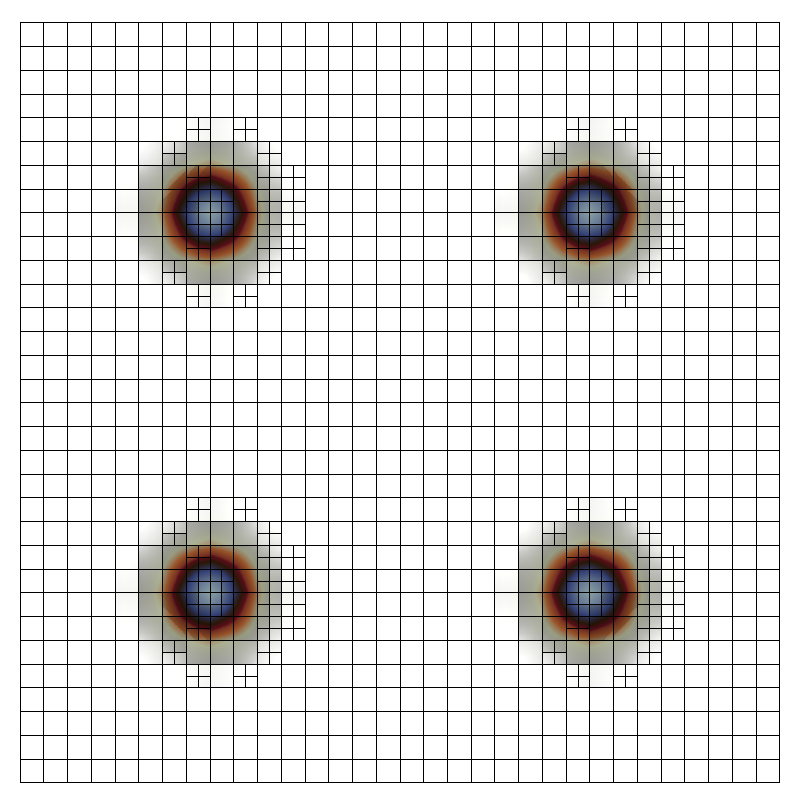}
    \includegraphics[width=0.49\linewidth]{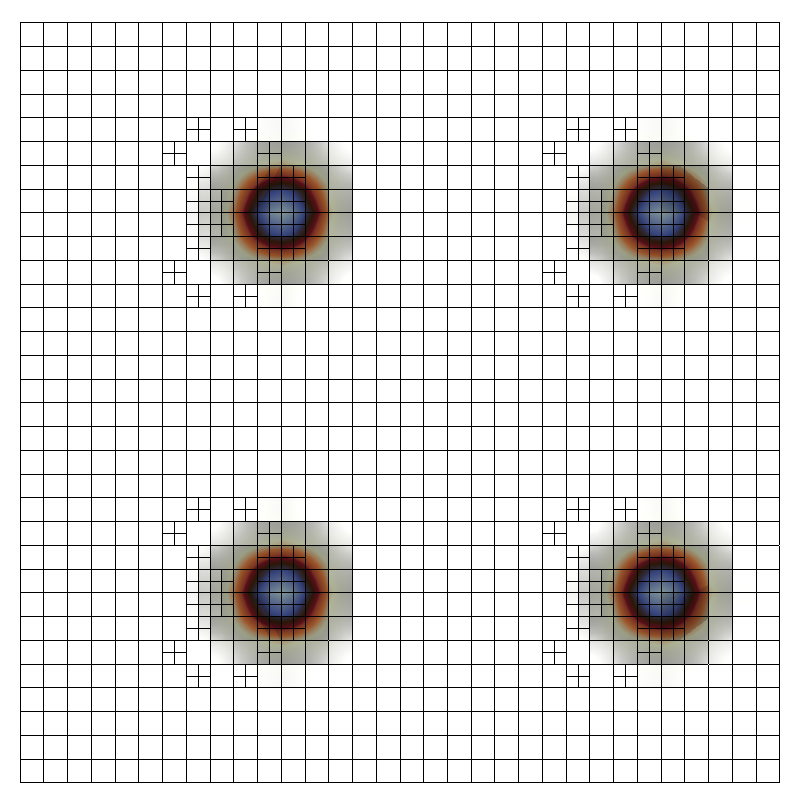}
    \captionsetup{labelformat=empty}
    \caption{$t=1$}
\end{subfigure}
\hfill
\begin{subfigure}[t]{0.33\linewidth}
    \centering
    \includegraphics[width=0.49\linewidth]{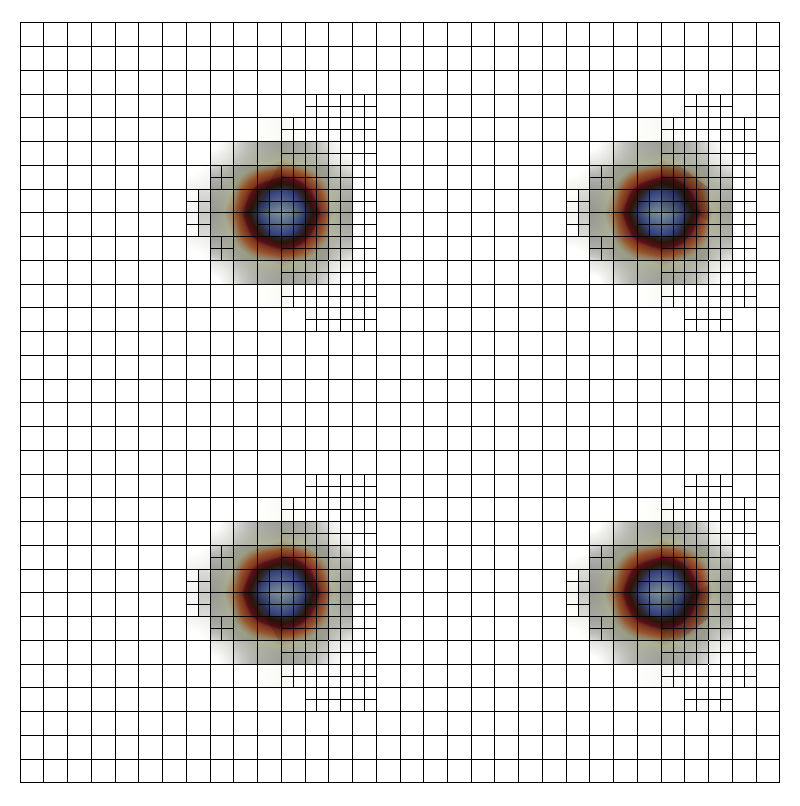}
    \includegraphics[width=0.49\linewidth]{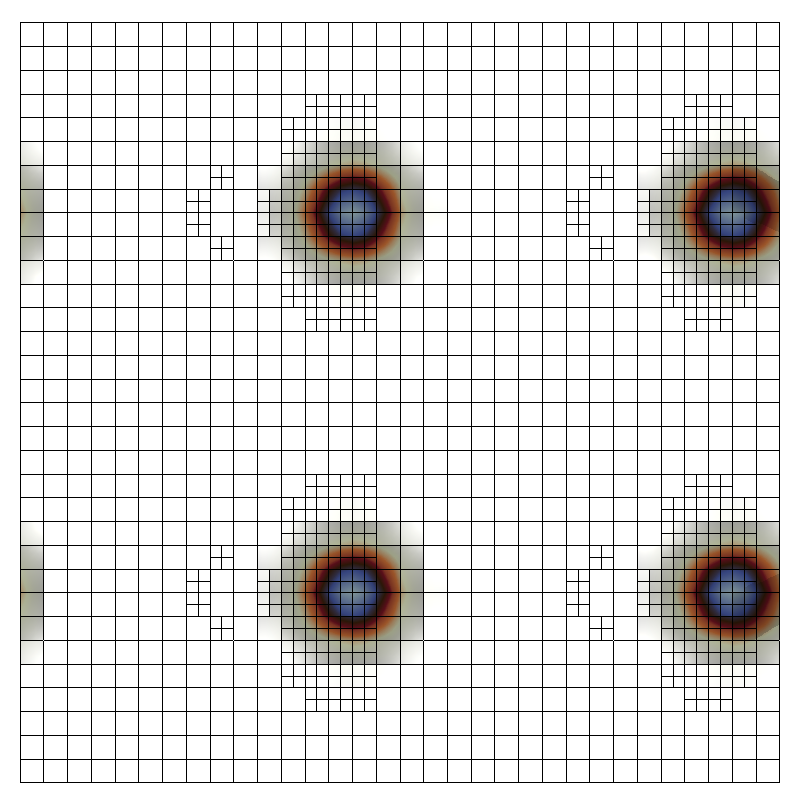}
    \captionsetup{labelformat=empty}
    \caption{$t=2$}
\end{subfigure}
\hfill
\begin{subfigure}[t]{0.33\linewidth}
    \centering
    \includegraphics[width=0.49\linewidth]{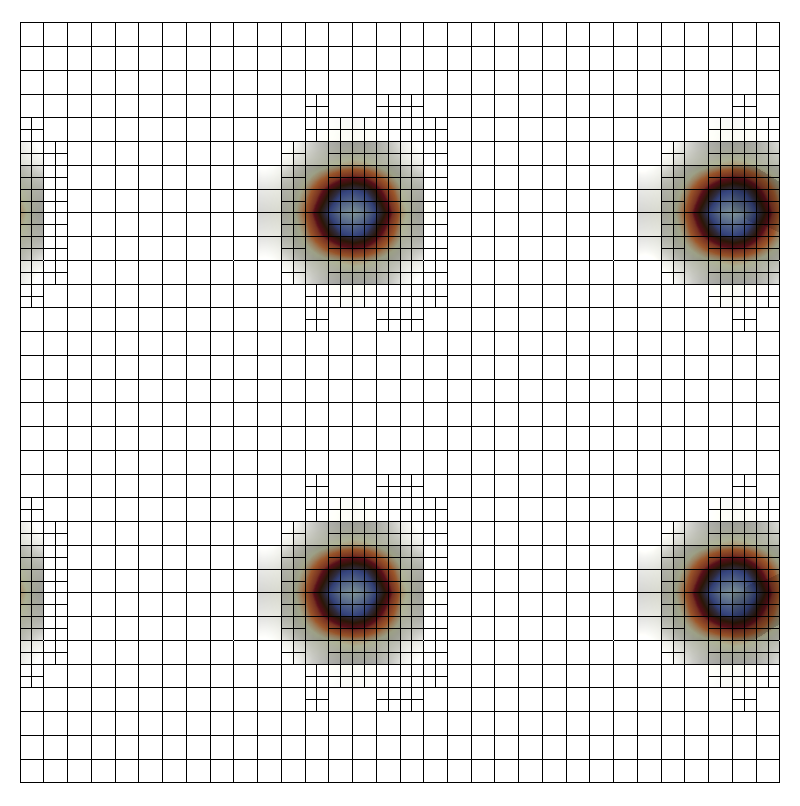}
    \includegraphics[width=0.49\linewidth]{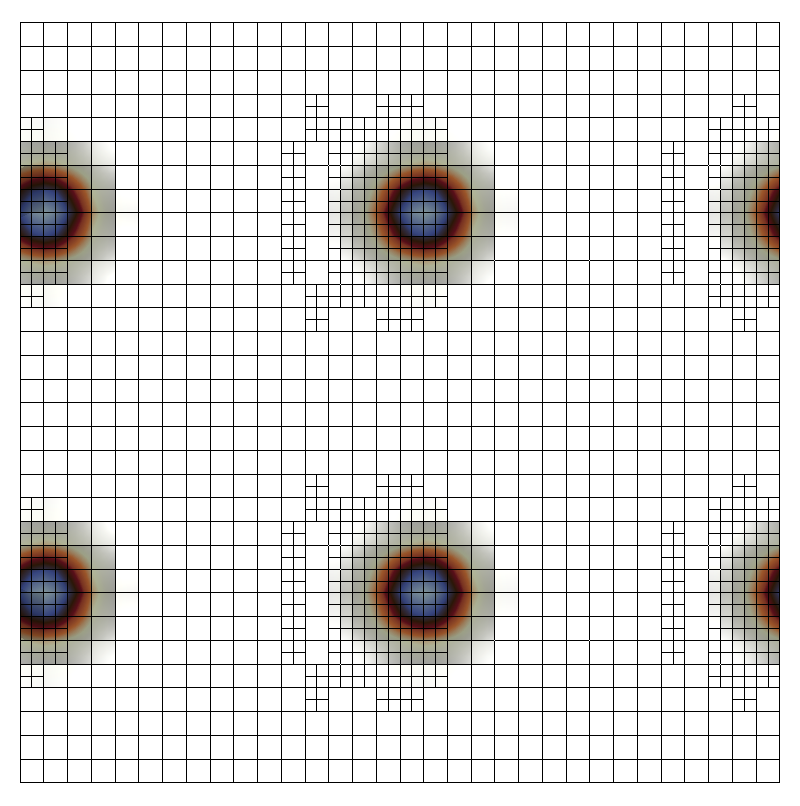}
    \captionsetup{labelformat=empty}
    \caption{$t=3$}
\end{subfigure}
\caption{Translational equivariance.}
\label{fig:iso_periodic_velunif_nx16_ny16_depth1_tstep0p25_vdn_graphnet_nodoftime_3_fourbumps}
\end{figure*}

\textbf{FEM settings.}
Hyperparameters for all environment settings are listed in \Cref{tab:env_settings}.
All simulations begin by applying level-1 refinement to elements whose initial error, which is available in general due to the known initial condition, is larger than $c_{\text{thres}}$.
The training mesh length and width are $[s_x, s_y]$.
We specified velocity by magnitude $u \in \Rbb$ and angle $\theta \in [0, 1.0]$, so that velocity components are $v_x = u \cos(2 \pi \theta)$ and $v_y = u \sin(2 \pi \theta)$.
We trained on 2D Gaussians
\begin{align}
    f(x,y,t) = 1 + \exp\left( -w ( (x - (x_0 + v_x t))^2 + (y - (y_0 + v_y t))^2 ) \right) \, ,
\end{align}
with initial conditions sampled from uniform distributions: $u \sim U[0, 1.5]$, $\theta \sim U[0, 1.0]$, $x_0, y_0 \sim U[0.5, 1.5]$, and $w = 100$.
At test time, anisotropic Gaussians are specified by
\begin{align}
    f(x,y,t) &= 1 + \exp\left( - ( w_x (x - dx)^2 + w_y (y - dy)^2 \right. \\
    &\quad \left. + w_{xy} (x - dx) (y - dy) ) \right) \\
    dx &= x_0 + v_x t \\
    dy &= y_0 + v_y t
\end{align}
with $\theta \sim U[0,1.0]$, $u \sim U[0, 1.5]$, $w_x, w_y, w_{xy} \sim U[20, 100]$, $x_0, y_0 \sim U[0.5,1.5]$.
Ring functions are
\begin{align}
    f(x,y,t) &= 1 + \exp\left( -w ( \sqrt{\Delta r^2} - r)^2 \right) \\
    \Delta r^2 &= (x - (x_0 + v_x t))^2 + (y - (y_0 + v_y t))^2
\end{align}
with $r \sim U[0.1, 0.3]$, $\theta \sim U[0, 1.0]$, $u \sim U[0, 1.5]$, $x_0, y_0 \sim U[0.5, 1.5]$.
Opposite traveling Gaussians have the same form as above, with $u \sim U[0, 1.5]$, $w=100$, the leftward moving Gaussian with $\theta = 0.5$, $x_0 \sim U[0.5, 1.5]$ and $y_0 \sim U[0.3, 0.7]$, while the rightward moving Gaussian has $\theta = 0.0$, $x_0 \sim U[0.5, 1.5]$ and $y_0 \sim U[1.3, 1.7]$.
The star mesh test case uses a mesh from MFEM \citep[Example 8]{anderson2019mfem}, with $\theta \sim U[0.05, 0.25]$, $u \sim U[0, 6.0]$, $w=5$, $x_0 = -1$ and $y_0 = -4$.

\begin{figure*}[t]
\centering
\begin{subfigure}[t]{0.24\linewidth}
    \centering
    \includegraphics[width=0.49\linewidth]{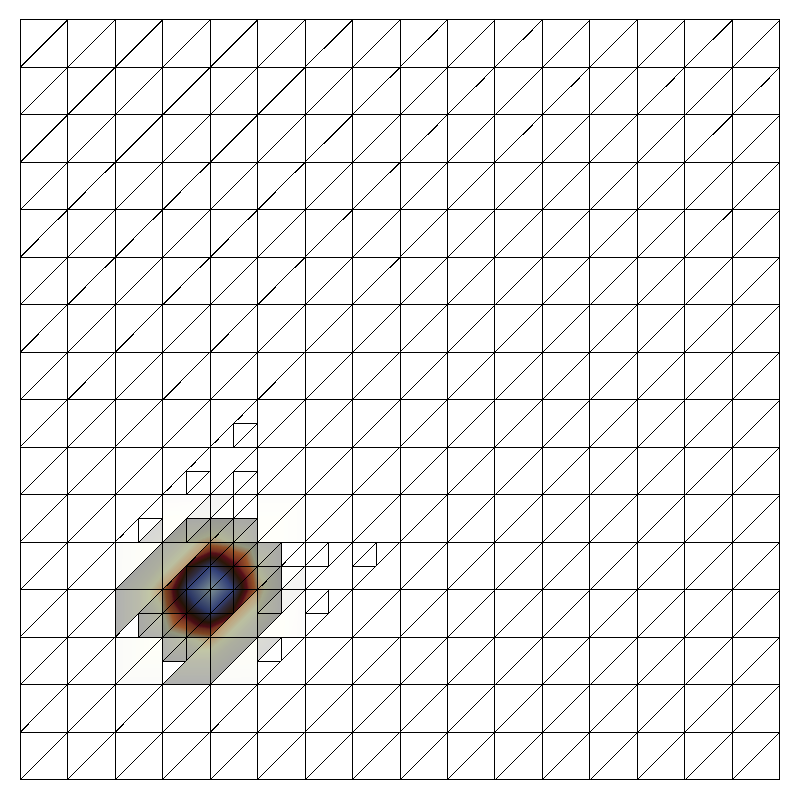}
    \includegraphics[width=0.49\linewidth]{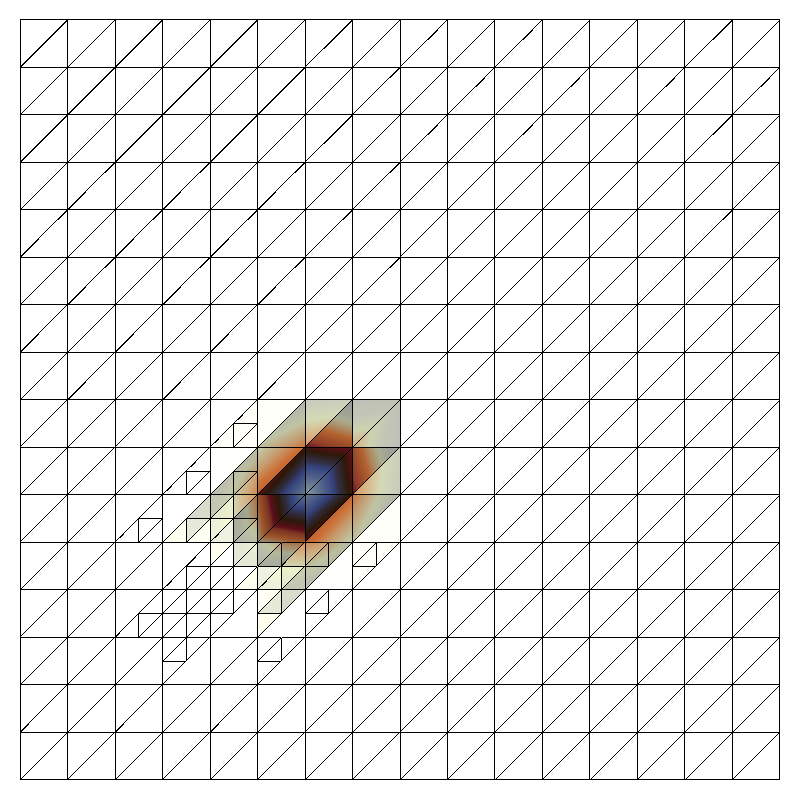}
    \captionsetup{labelformat=empty}
    \caption{$t=1$}
\end{subfigure}
\hfill
\begin{subfigure}[t]{0.24\linewidth}
    \centering
    \includegraphics[width=0.49\linewidth]{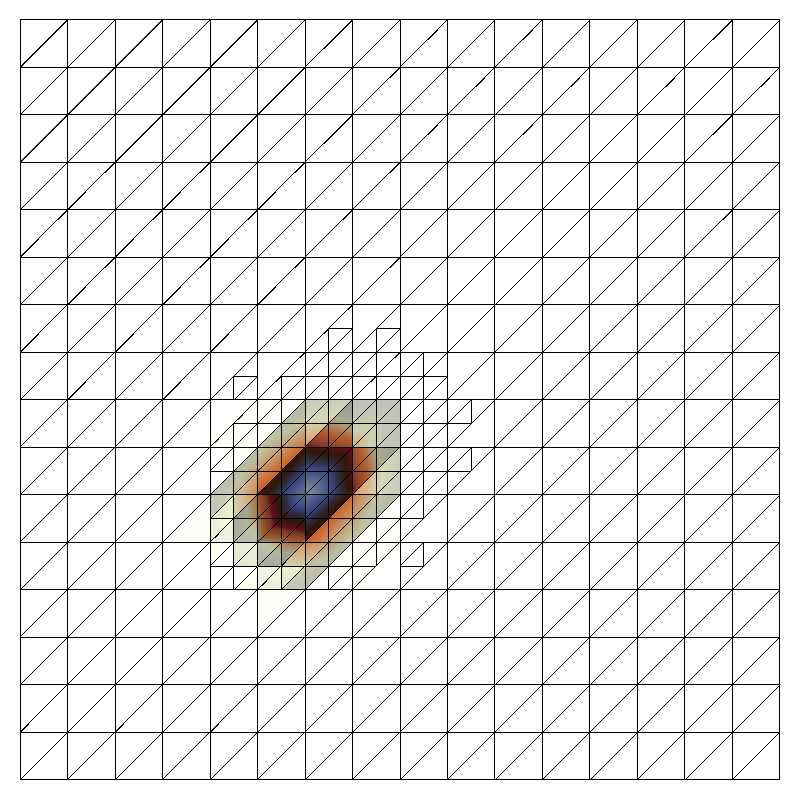}
    \includegraphics[width=0.49\linewidth]{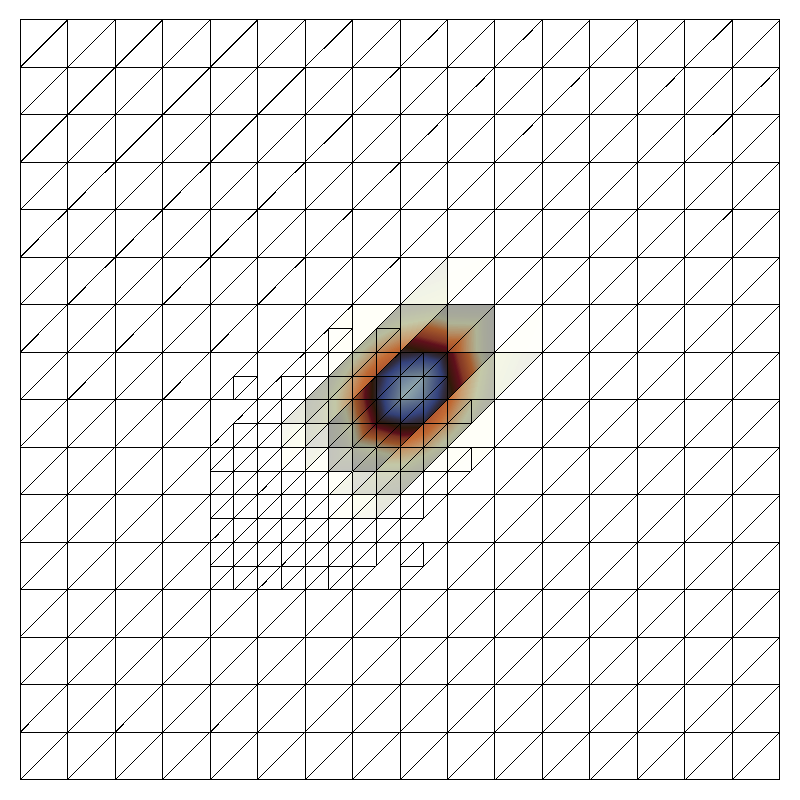}
    \captionsetup{labelformat=empty}
    \caption{$t=2$}
\end{subfigure}
\hfill
\begin{subfigure}[t]{0.24\linewidth}
    \centering
    \includegraphics[width=0.49\linewidth]{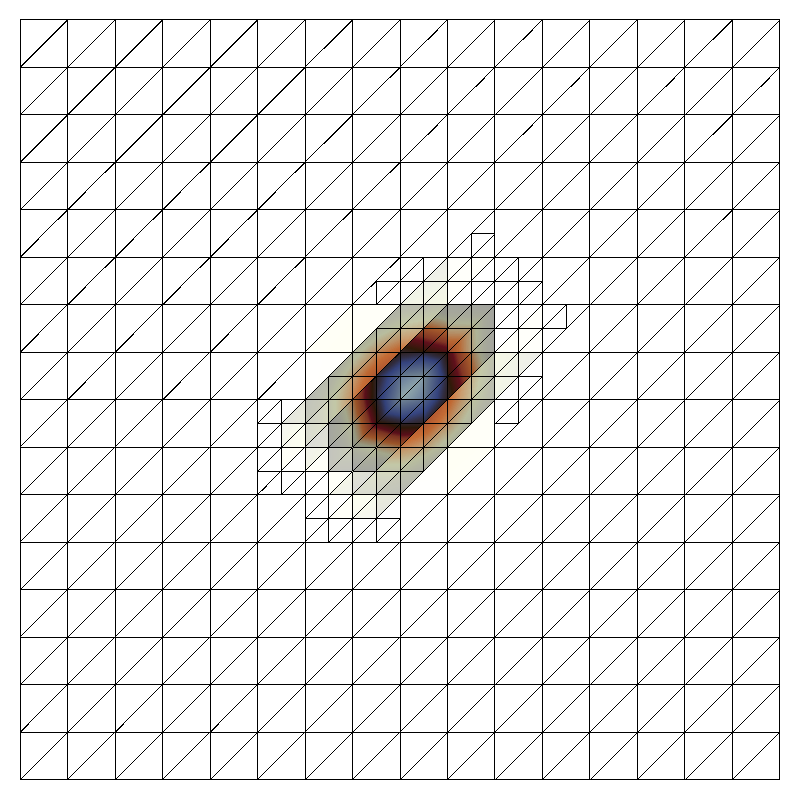}
    \includegraphics[width=0.49\linewidth]{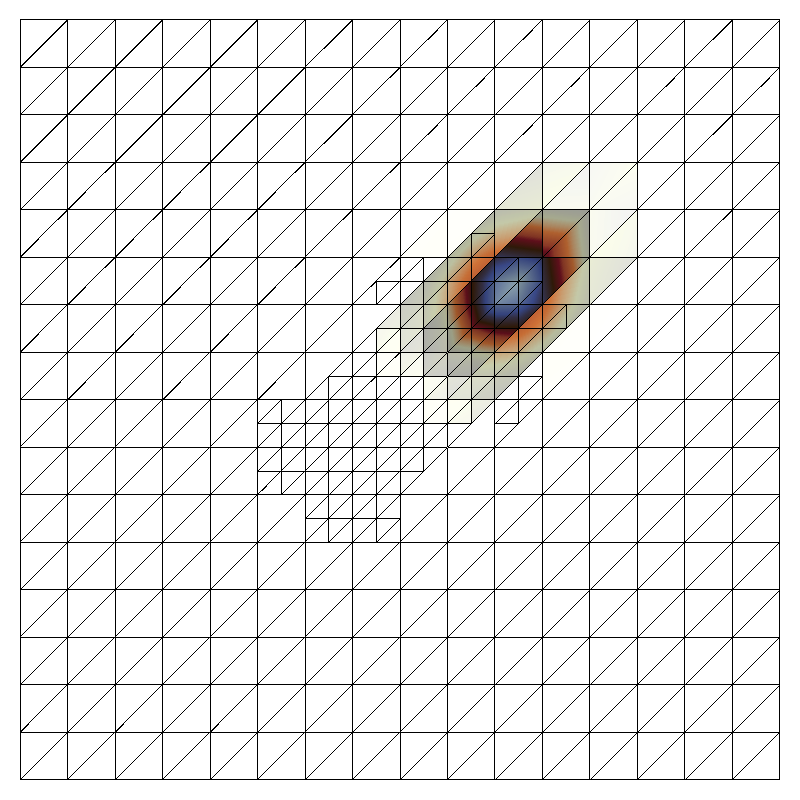}
    \captionsetup{labelformat=empty}
    \caption{$t=3$}
\end{subfigure}
\hfill
\begin{subfigure}[t]{0.24\linewidth}
    \centering
    \includegraphics[width=0.49\linewidth]{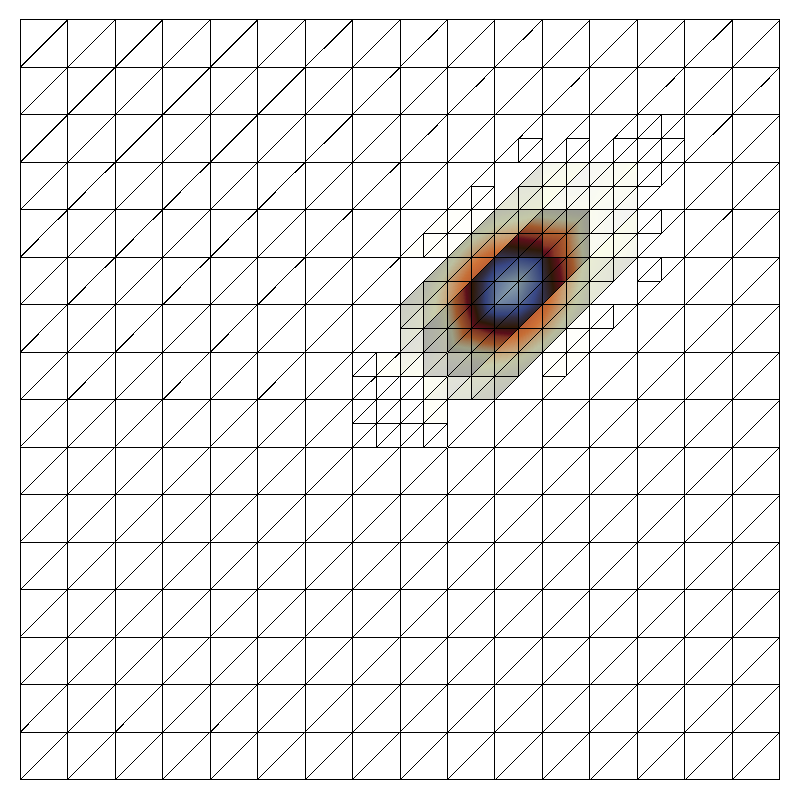}
    \includegraphics[width=0.49\linewidth]{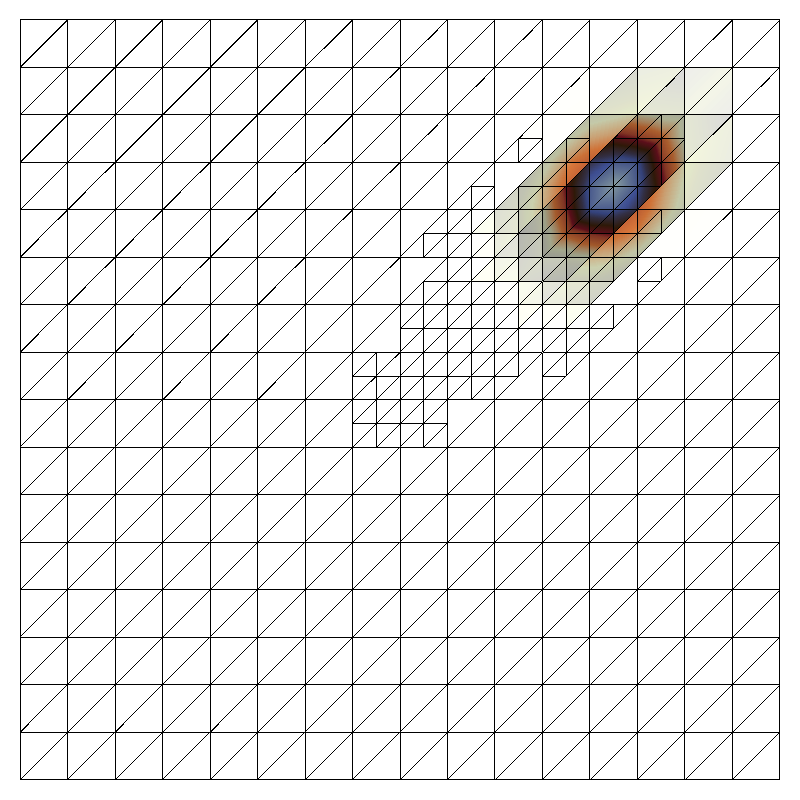}
    \captionsetup{labelformat=empty}
    \caption{$t=4$}
\end{subfigure}
\caption{Policy trained on quadrilateral elements can be run on triangular elements.}
\label{fig:iso_periodic_velunif_nx16_ny16_depth1_tstep0p25_vdn_graphnet_nodoftime_3_gen_triangle}
\end{figure*}

\begin{table*}[t]
    \centering
    \caption{Runtime in seconds of VDGN is comparable to local error threshold-based policies with best $\theta_r$ selected from \Cref{tab:efficiency}. Average over 100 test episodes.}
    \begin{tabular}{crrrrrrrrr}
    \toprule
    \multicolumn{1}{}{}&
    \multicolumn{3}{c}{In-distribution} &
    \multicolumn{6}{c}{Generalization}\\
    \cmidrule(r){2-4}
    \cmidrule(r){5-10}
    & \shortstack{Depth 1 \\ $375$ steps} & \shortstack{Depth 2 \\ $400$ steps} & \shortstack{Triangular \\ $1250$ steps} & \shortstack{Depth 1 \\ $2500$ steps} & \shortstack{Depth 2 \\ $2500$ steps} & \shortstack{Anisotropic \\ $2500$ steps} & \shortstack{Ring \\ $2500$ steps} & \shortstack{Opposite \\ $2500$ steps} & \shortstack{Star \\ $750$ steps} \\
    \midrule
    Threshold (action) & 0.009 & 0.005 & 0.005 & 0.015 & 0.012 & 0.014 & 0.014 & 0.021 & 0.005 \\
    Threshold (total) & 1.18 & 0.84 & 2.15 & 12.9 & 11.0 & 11.9 & 12.2 & 18.1 & 1.30 \\
    \midrule
    VDGN (action) & 0.024 & 0.011 & 0.005 & 0.026 & 0.011 & 0.008 & 0.008 & 0.008 & 0.008 \\
    VDGN (total) & 1.96 & 1.18 & 2.26 & 15.0 & 8.23 & 17.3 & 18.5 & 21.6 & 3.43 \\
    \bottomrule
    \end{tabular}
    \label{tab:runtime}
\end{table*}

\section{Additional results}
\label{app:results}

% \subsection{Symmetries}
\textbf{Symmetry.}
In \Cref{fig:iso_periodic_velunif_nx16_ny16_depth1_tstep0p25_vdn_graphnet_nodoftime_3_ep210800}, the initial conditions are $x_0 = y_0 = 0.5$ and $v_x = v_y = 1.5$.
We rotated the initial conditions by $\pi$, so that $x_0 = y_0 = 1.5$ and $v_x = v_y = -1.5$, and we can see from
\Cref{fig:iso_periodic_velunif_nx16_ny16_depth1_tstep0p25_vdn_graphnet_nodoftime_3_ep210800_degree225} that the global refinement choices are similarly rotated by $\pi$.
Rotational equivariance can be also seen in \Cref{fig:iso_periodic_velunif_nx16_ny16_depth1_tstep0p25_vdn_graphnet_nodoftime_3_ep210800_opposite}, where the refinement choices for the leftward and rightward moving waves are reflections of each other.
Translational equivariance along the vertical and horizontal axes can be seen in \Cref{fig:iso_periodic_velunif_nx16_ny16_depth1_tstep0p25_vdn_graphnet_nodoftime_3_fourbumps}.

% \subsection{Triangular elements}

% Policies trained on quadrilateral elements tested on triangle elements
% (see \Cref{fig:iso_periodic_velunif_nx8_ny8_depth1_tstep0p25_triangle_vdn_graphnet_resume_lr0p0005}).

% \subsection{Runtime}

\textbf{Runtime.} \Cref{tab:runtime} shows that the runtime of trained VDGN policies at test time is on the same order of magnitude as Threshold policies, for both average time per action (for all elements) and average time per episode.

\textbf{Scaling.}
\Cref{fig:iso_periodic_velunif_nx16_ny16_depth1_tstep0p25_vdn_graphnet_nodoftime_3_ep210800_tfinal5} shows a policy trained with only three steps per episode (375 solver steps), running on 20 steps at test time (2500 solver steps).
\Cref{fig:iso_periodic_velunif_nx16_ny16_depth1_tstep0p25_vdn_graphnet_nodoftime_3_ep210800_nx64_ny64,fig:iso_periodic_velunif_nx16_ny16_depth1_tstep0p25_vdn_graphnet_nodoftime_3_ep210800_nx64_ny64_continued}
shows a policy trained on a $16 \times 16$ mesh with only three steps per episode, tested on a $64 \times 64$ mesh.

\textbf{Using error estimator as observation.}
\Cref{tab:efficiency_errest} shows that VDGN retains better performance than Threshold policies when the true error in agents' observations is replaced by the estimates given by the ZZ error estimator \citep{zienkiewicz1992superconvergent}.

\section{Other formulations and their difficulties}
\label{app:other_formulations}

\subsection{Single-agent Markov decision process}
\label{app:single-agent}

As mentioned in \Cref{sec:related_work}, previous work that treat AMR as a sequential decision-making problem took the single-agent viewpoint.
This poses issues with producing multiple element refinements at each mesh update step.

In the global approach \citep{yang2021reinforcement}, a single agent takes the entire mesh state as input, chooses one element out of all available elements to refine (or potentially de-refine), and then the global solution is updated.
To avoid computing the global solution after every refinement/de-refinement of a single element, which would be prohibitively expensive, an extension of this approach would have to use either a fixed hyperparameter or a fixed rule to determine the number of rounds of single-element selection between each solution update.
Specifying this hyperparameter or rule is hard to do in advance and may erroneously exclude the true optimal policy from the policy search space. 

In the local approach \citep{foucart2022deep}, a single agent is ``centered'' on an element and makes a refine/no-op/de-refine decision for that element.
At training time, because each refine/de-refine action impacts the global solution, which directly impacts the reward, the environment transition involves a global solution update every time the agent chooses an action for an element.
Since this is prohibitively expensive to do for larger number of elements at test time, \citep{foucart2022deep} redefines the environment transition at test time so that the global solution is updated only after the agent chooses an action for all elements.
This presents the agent with different definitions of the environment transition function at train and test time.

\begin{table}[htb]
    \centering
    \caption{Mean and standard error of efficiency of VDGN using ZZ error estimator values as observation, versus threshold-based policies, over 100 runs with uniform random ICs.}
    \begin{tabular}{crr}
    \toprule
    $\theta_r$ & Depth 1 $375$ steps & Depth 2 400 steps \\
    \midrule
    $5\times10^{-3}$ & 0.916 (0.007) & 0.523 (0.021)  \\
    $5\times10^{-4}$ & 0.899 (0.002) & 0.529 (0.022) \\
    $5\times10^{-5}$ & 0.857 (0.003) & 0.509 (0.021) \\
    $5\times10^{-6}$ & 0.813 (0.004) & 0.474 (0.019) \\
    $5\times10^{-7}$ & 0.767 (0.004) & 0.431 (0.018) \\
    $5\times10^{-8}$ & 0.718 (0.005) & 0.390 (0.016) \\
    $5\times10^{-15}$ & 0.473 (0.006) & 0.260 (0.014) \\
    \midrule
    VDGN & \textbf{0.948} (0.002) & \textbf{0.543} (0.022) \\
    VDGN/best $\theta_r$ & 1.03 & 1.03 \\
    \bottomrule
    \end{tabular}
    \label{tab:efficiency_errest}
\end{table}

\subsection{Fully-decentralized training and execution}
\label{app:fully-decentralized}

In contrast to the approach of centralized training with decentralized execution used in this work, one may consider fully-decentralized training and execution.
A reason for full decentralization is the fact that the governing equations of most physical systems in finite element simulations are local in nature, so one may hypothesize that an element only needs to respond to local observations and learn from local rewards.
However, this poses two main challenges:
1) the class of policies accessible by fully-decentralized training with local rewards may not contain the global optimum joint policy;
2) for $h$-adaptive mesh refinement with maximum depth greater than one, one must confront the posthumous credit assignment problem: a refinement or de-refinement action may have long-term impact on the global error but the agent responsible for that action immediately disappears upon taking the action and hence does not receive future rewards.

For the second challenge in particular, we can easily construct a scenario in which an agent who only receives an immediate reward and disappears upon refinement (i.e., does not persist) cannot learn anticipatory refinement.
We can construct a scenario where an agent cannot make an optimal refinement action:
1) use a moving feature such that error is minimized by the sequence of actions: $a^i_t = 1$ at time $t$, and $a^{j}_{t+1} = 1, \forall j \in \text{Desc}(i)$ at time $t+1$;
2) construct the feature so that action $a^i_t=1$ has no instantaneous benefit for error reduction, meaning that the instantaneous reward $r^i_t$ for the state-action pair $(o^i_t, a^i_t=1)$ is the same as the reward for the pair $(o^i_t, a^i_t=0)$.
Since agent $i$'s trajectory terminates upon receiving reward $r^i_t$, this means agent $i$ cannot distinguish the value of refinement versus no-op.
This problem may be overcome by letting all agents persist after refinement, receive dummy observations and local reward but take no action, so that they can learn the past action's impact on future local reward.
Another way is to let the agent disappear upon refinement/de-refinement at transition $(s_t,a_t,r_t)$ but use the final reward $r_T$ (e.g., upon episode termination at time $t=T$) as the reward for the transition $(s_t,a_t,r_T)$.

%%%%%%% Longer time %%%%%%%%%%%%%%
\begin{figure*}[ht]
\centering
\begin{subfigure}[t]{0.33\linewidth}
    \centering
    \includegraphics[width=0.49\linewidth]{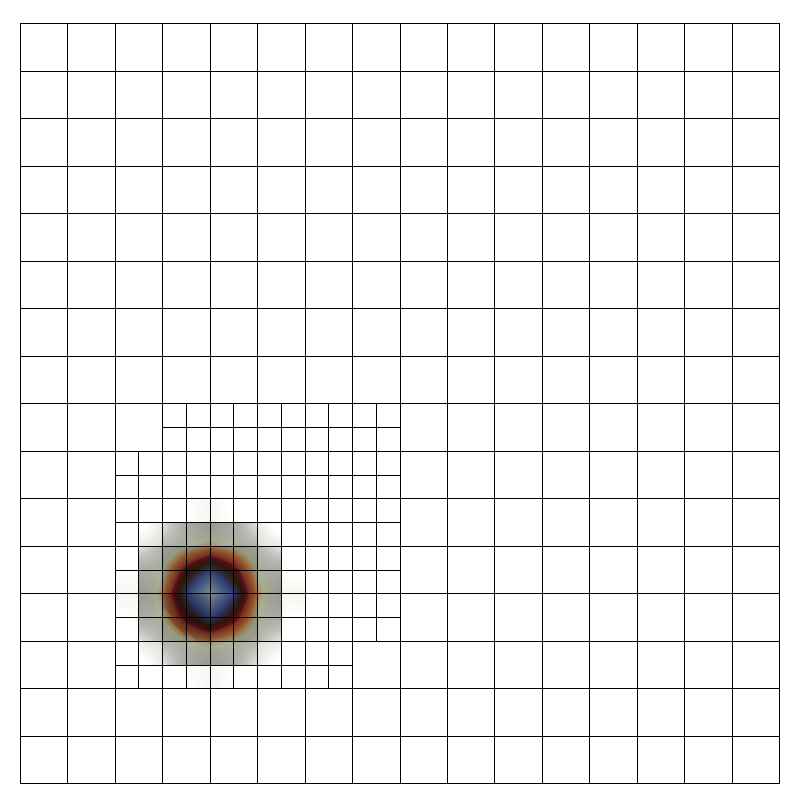}
    \includegraphics[width=0.49\linewidth]{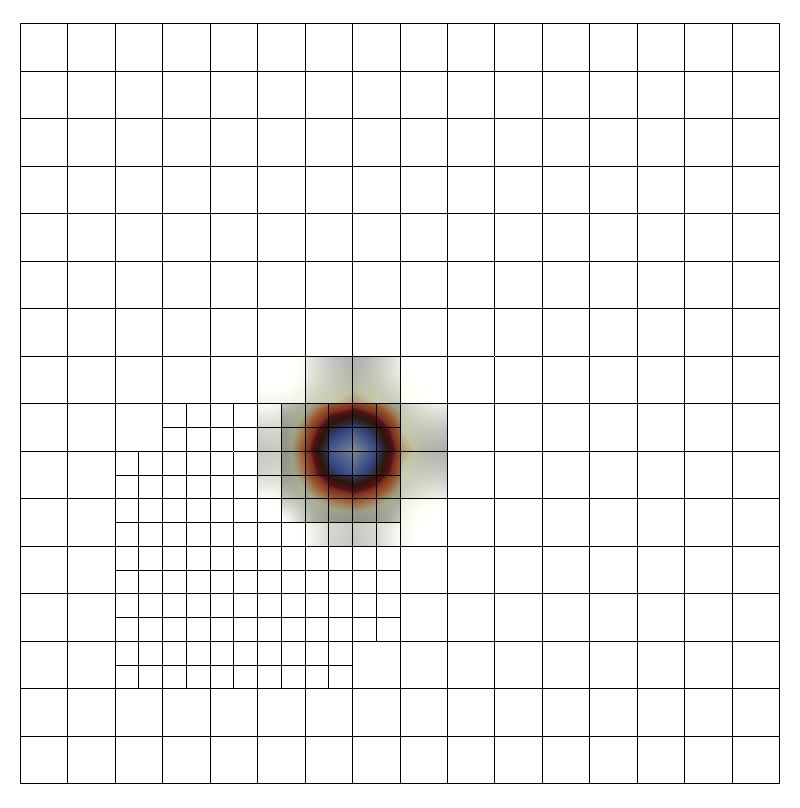}
    \captionsetup{labelformat=empty}
    \caption{$t=1$}
\end{subfigure}
\hfill
\begin{subfigure}[t]{0.33\linewidth}
    \centering
    \includegraphics[width=0.49\linewidth]{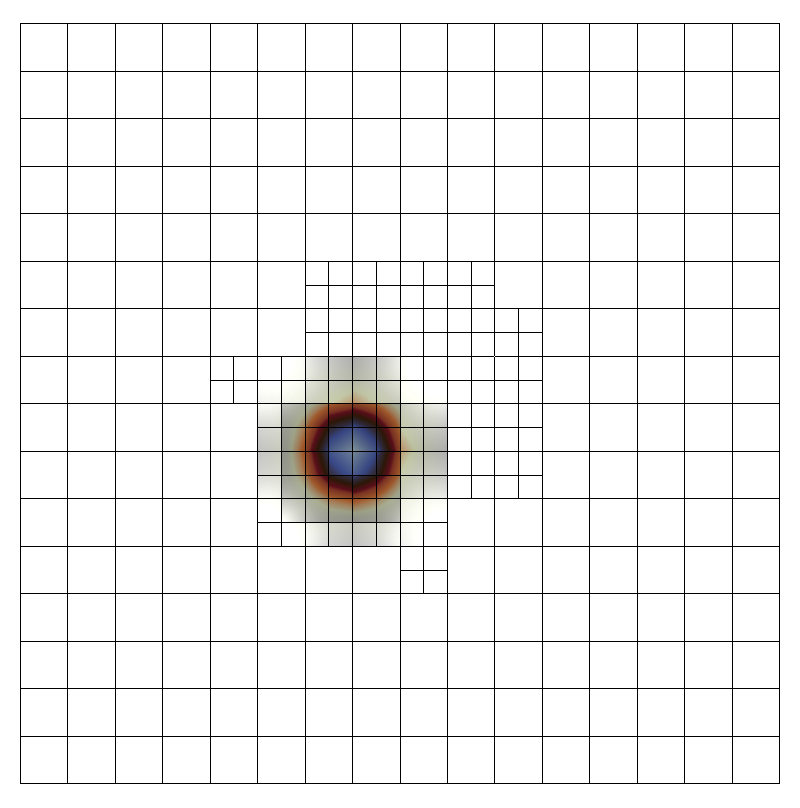}
    \includegraphics[width=0.49\linewidth]{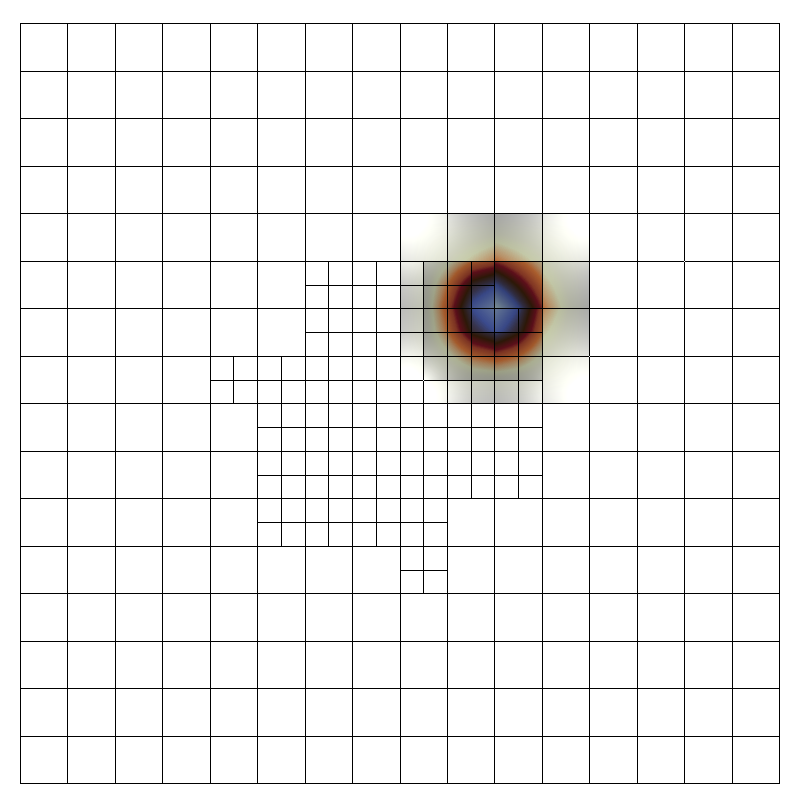}
    \captionsetup{labelformat=empty}
    \caption{$t=2$}
\end{subfigure}
\hfill
\begin{subfigure}[t]{0.33\linewidth}
    \centering
    \includegraphics[width=0.49\linewidth]{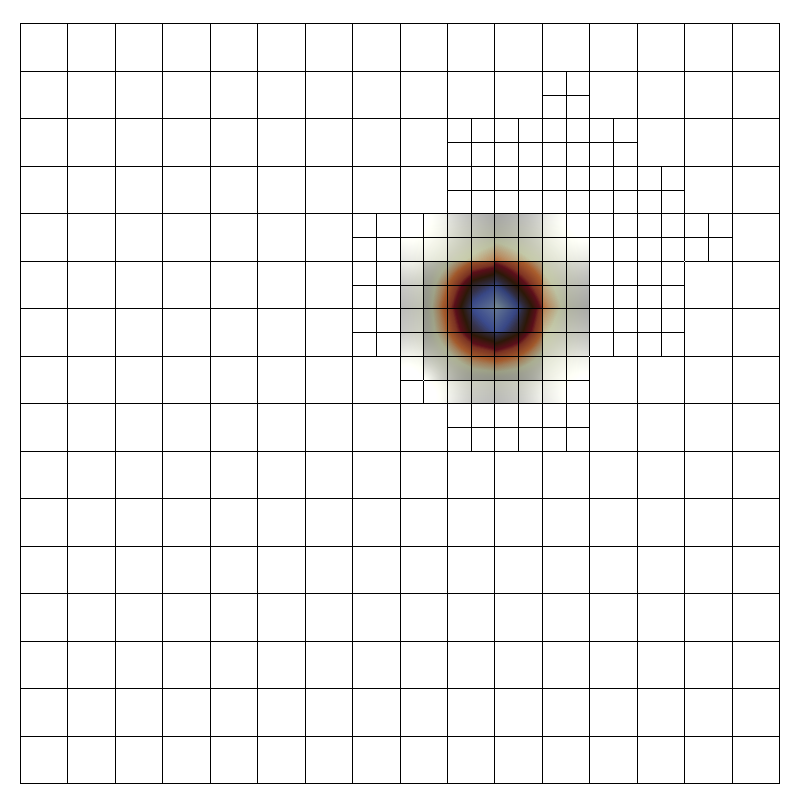}
    \includegraphics[width=0.49\linewidth]{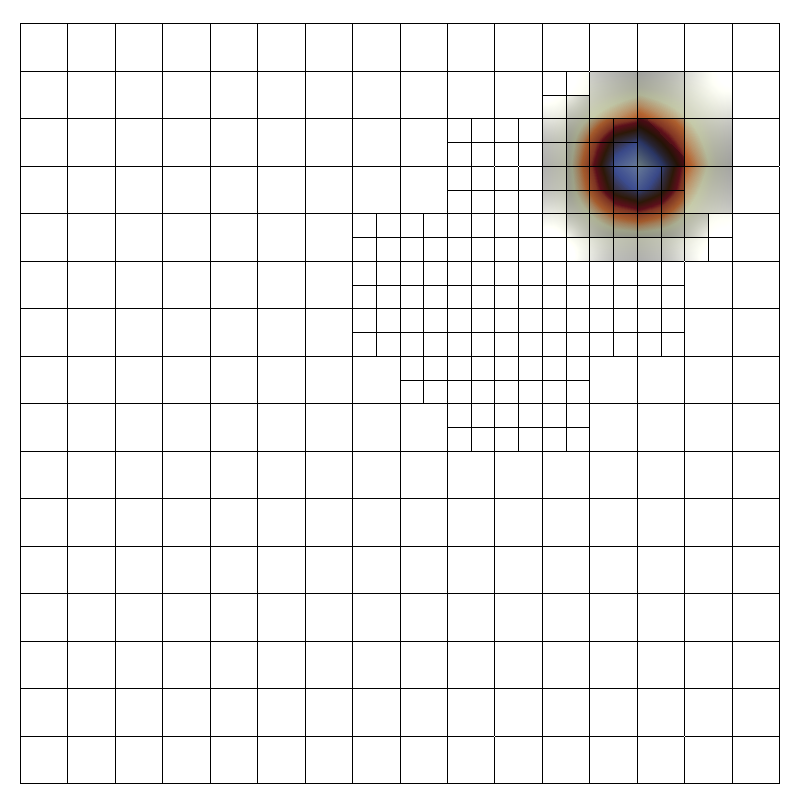}
    \captionsetup{labelformat=empty}
    \caption{$t=3$}
\end{subfigure}

\begin{subfigure}[t]{0.33\linewidth}
    \centering
    \includegraphics[width=0.49\linewidth]{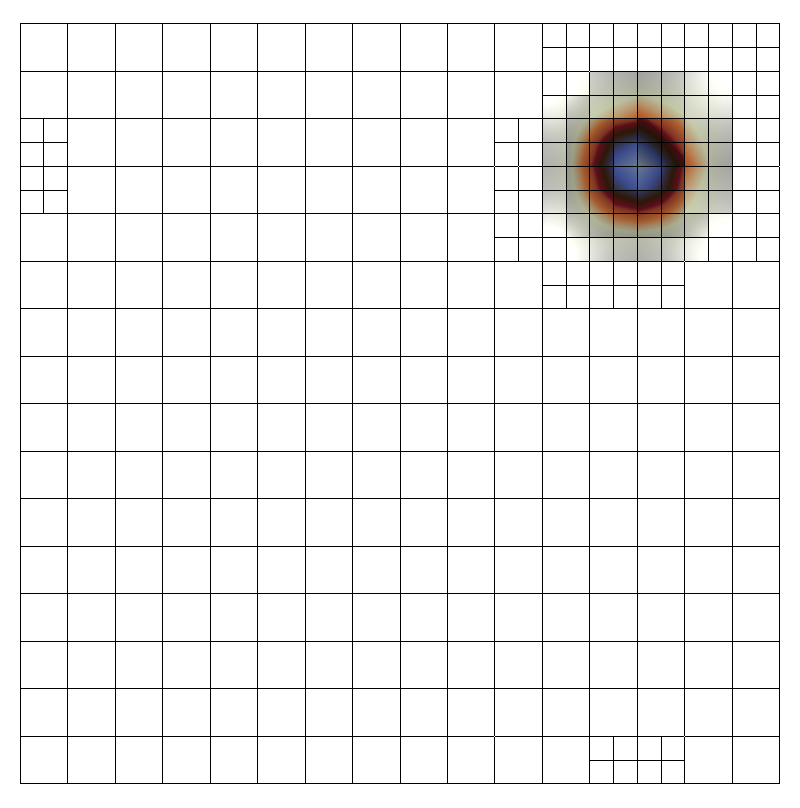}
    \includegraphics[width=0.49\linewidth]{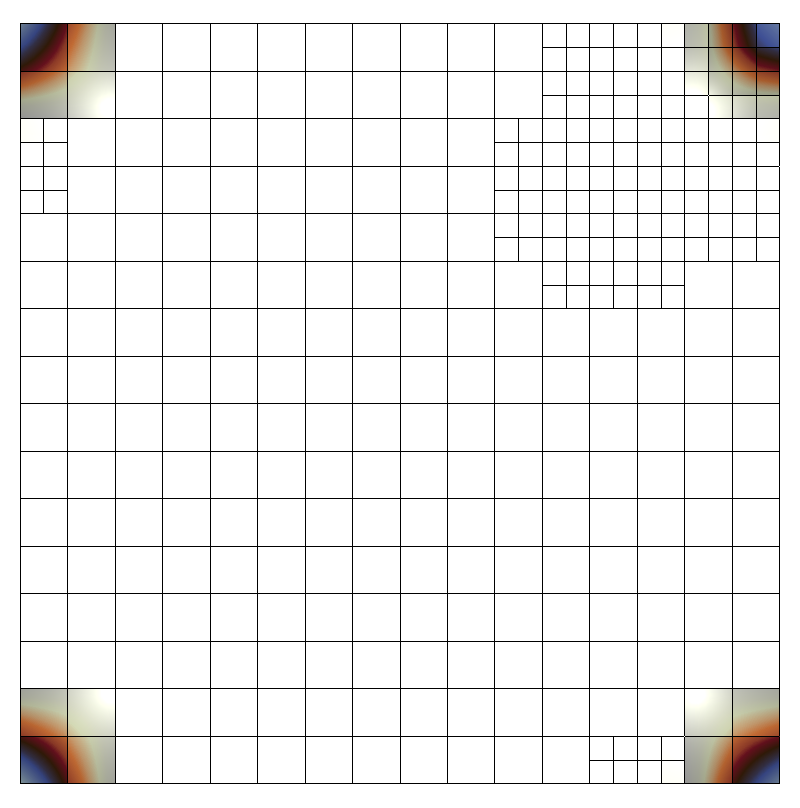}
    \captionsetup{labelformat=empty}
    \caption{$t=4$}
\end{subfigure}
\hfill
\begin{subfigure}[t]{0.33\linewidth}
    \centering
    \includegraphics[width=0.49\linewidth]{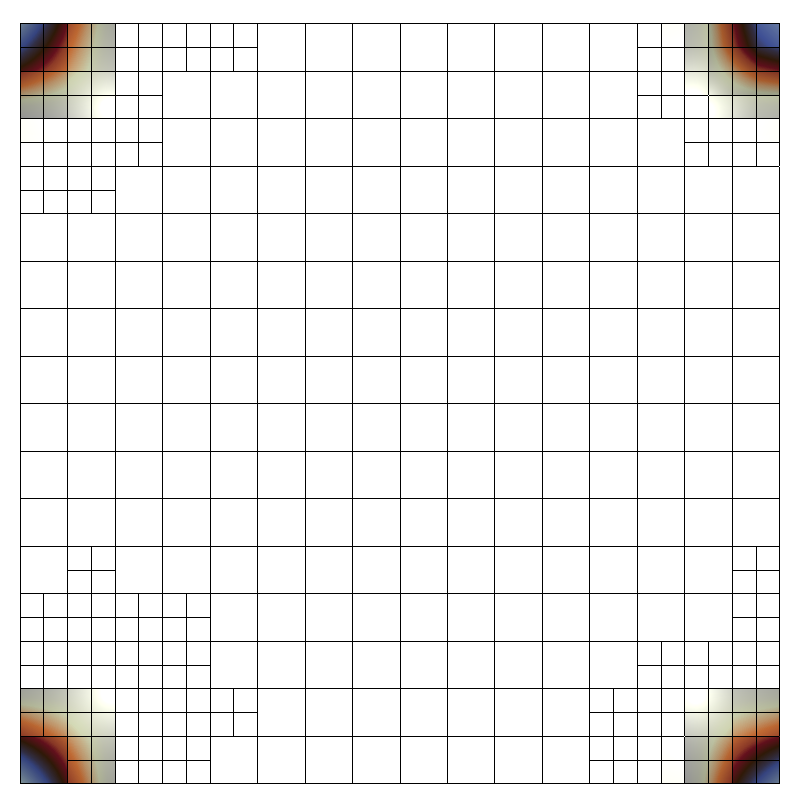}
    \includegraphics[width=0.49\linewidth]{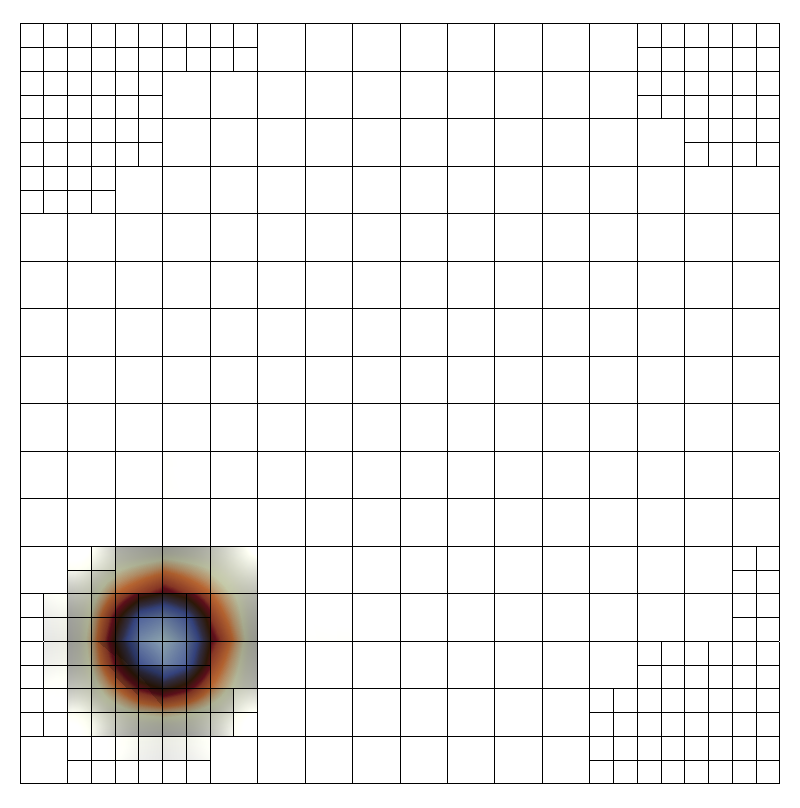}
    \captionsetup{labelformat=empty}
    \caption{$t=5$}
\end{subfigure}
\hfill
\begin{subfigure}[t]{0.33\linewidth}
    \centering
    \includegraphics[width=0.49\linewidth]{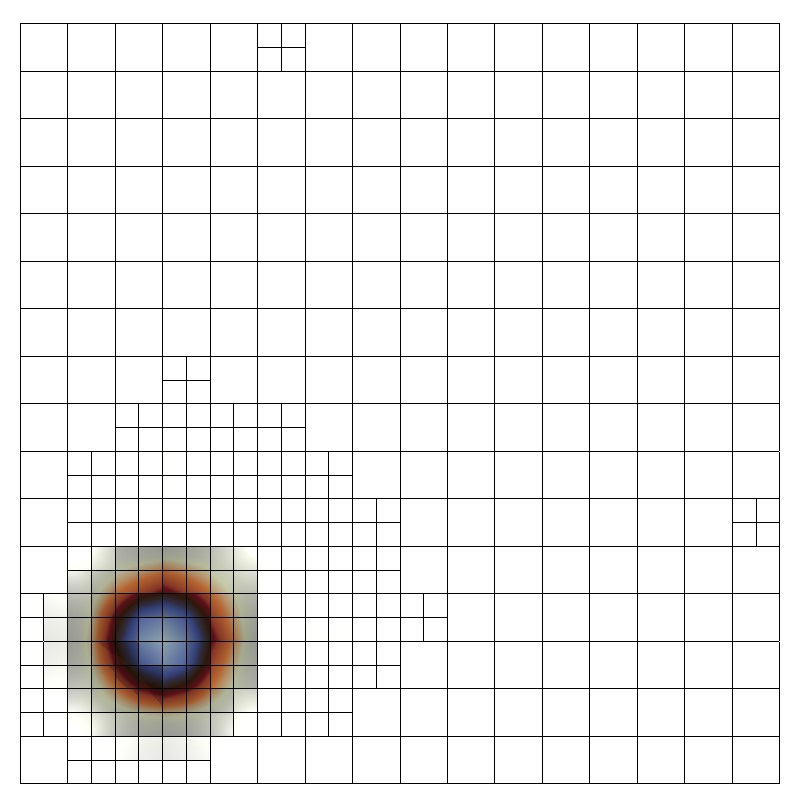}
    \includegraphics[width=0.49\linewidth]{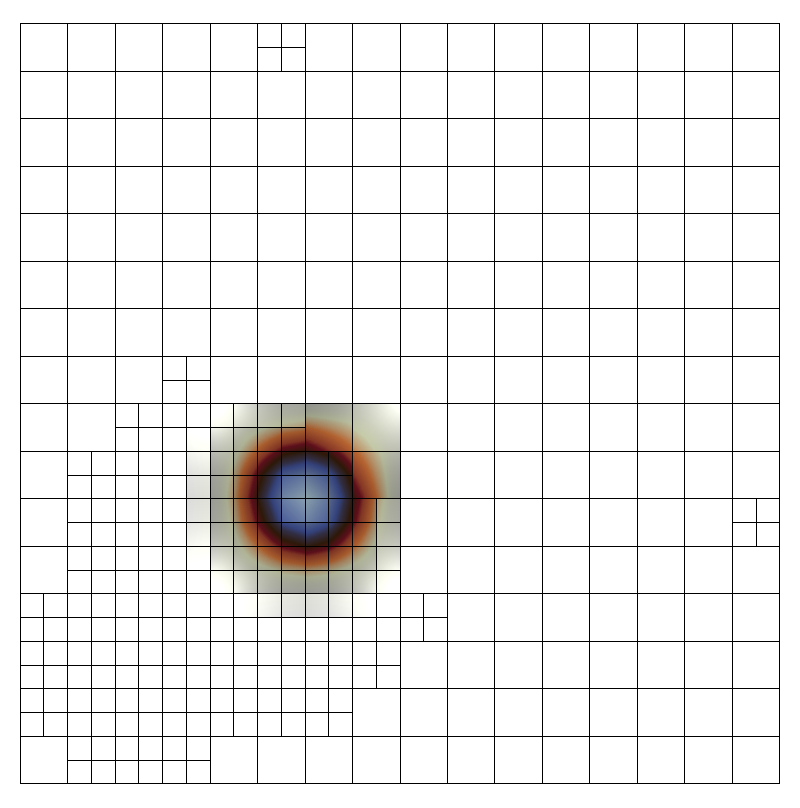}
    \captionsetup{labelformat=empty}
    \caption{$t=6$}
\end{subfigure}

\begin{subfigure}[t]{0.33\linewidth}
    \centering
    \includegraphics[width=0.49\linewidth]{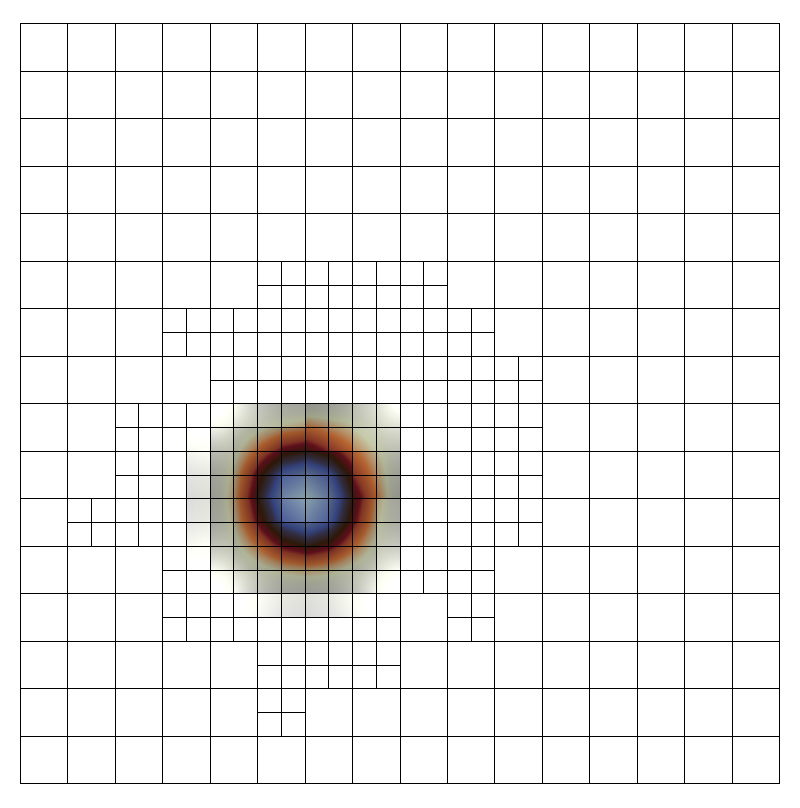}
    \includegraphics[width=0.49\linewidth]{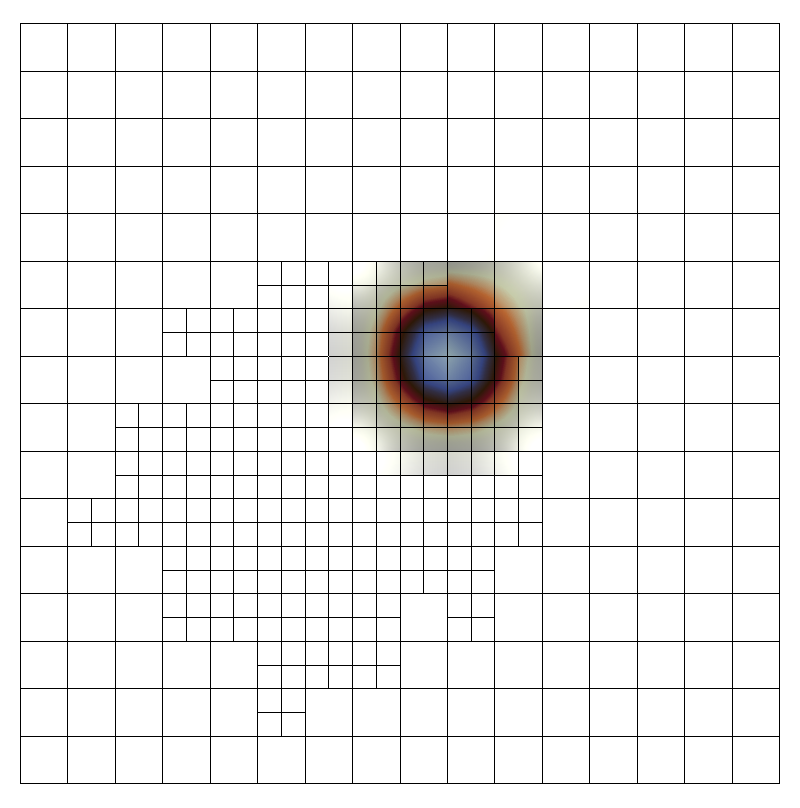}
    \captionsetup{labelformat=empty}
    \caption{$t=7$}
\end{subfigure}
\hfill
\begin{subfigure}[t]{0.33\linewidth}
    \centering
    \includegraphics[width=0.49\linewidth]{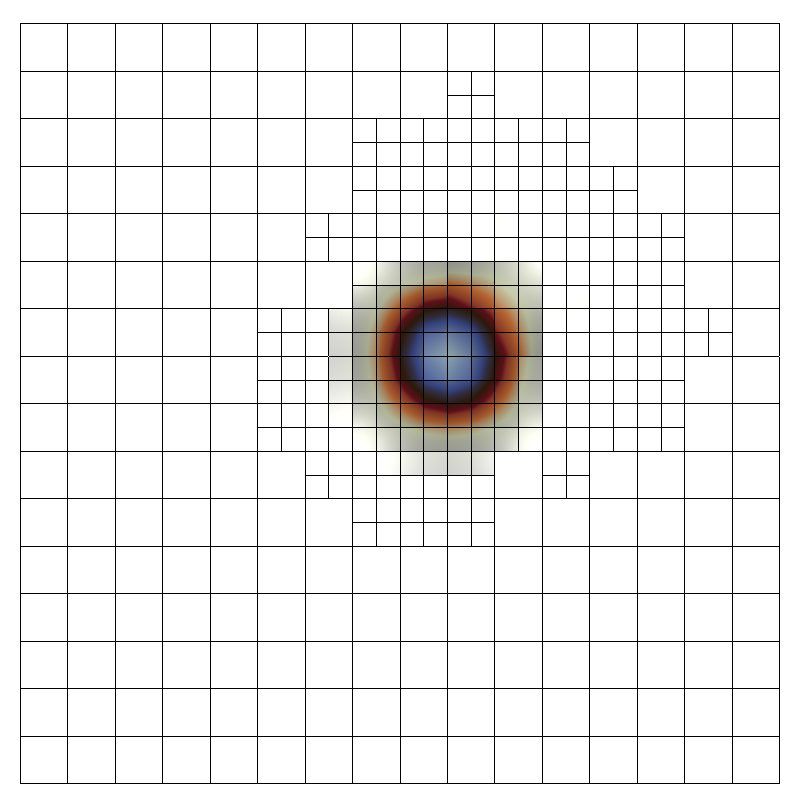}
    \includegraphics[width=0.49\linewidth]{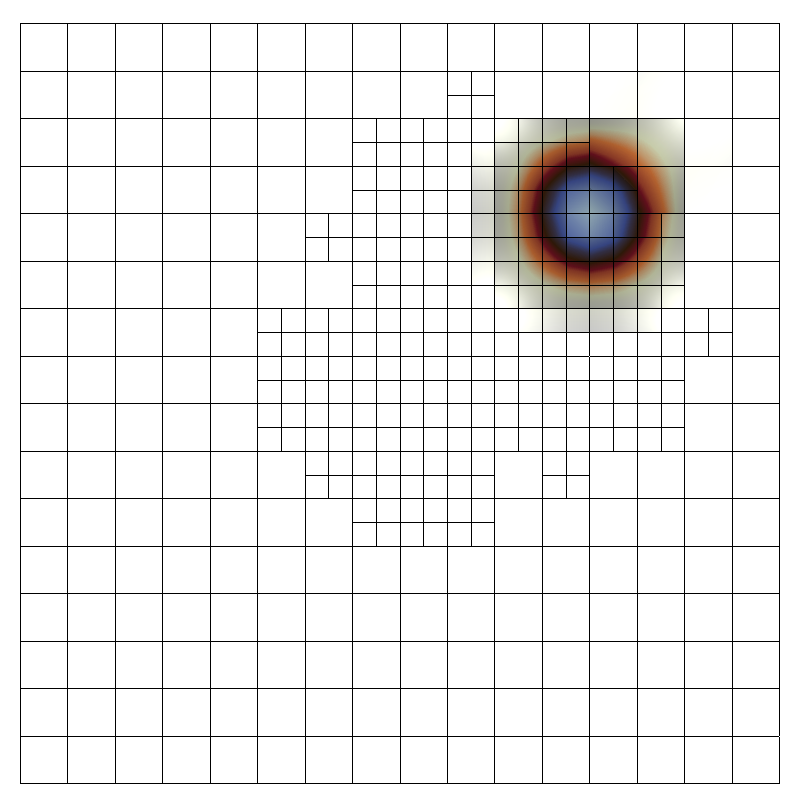}
    \captionsetup{labelformat=empty}
    \caption{$t=8$}
\end{subfigure}
\hfill
\begin{subfigure}[t]{0.33\linewidth}
    \centering
    \includegraphics[width=0.49\linewidth]{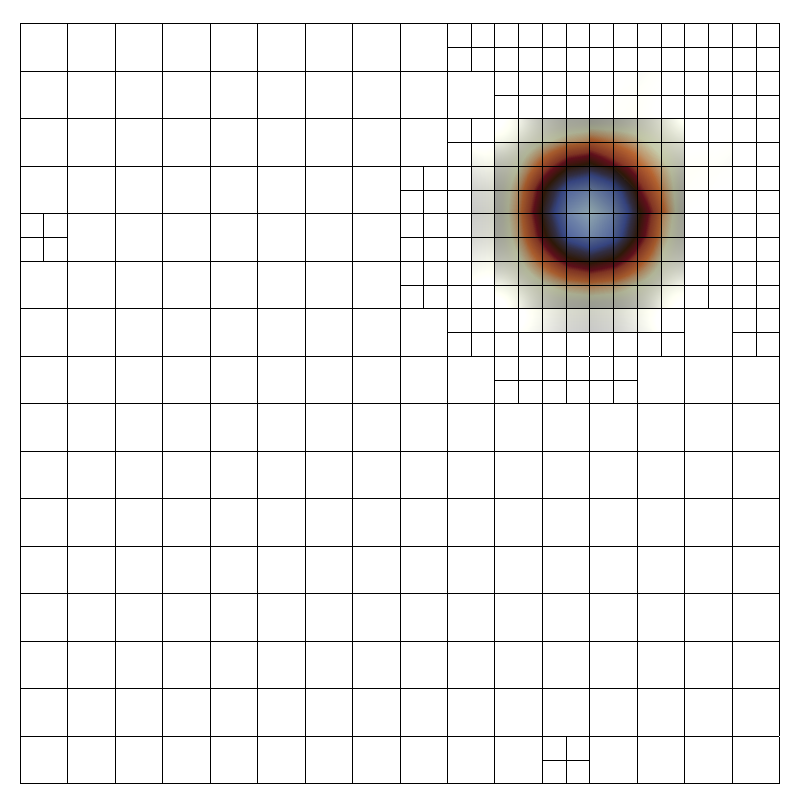}
    \includegraphics[width=0.49\linewidth]{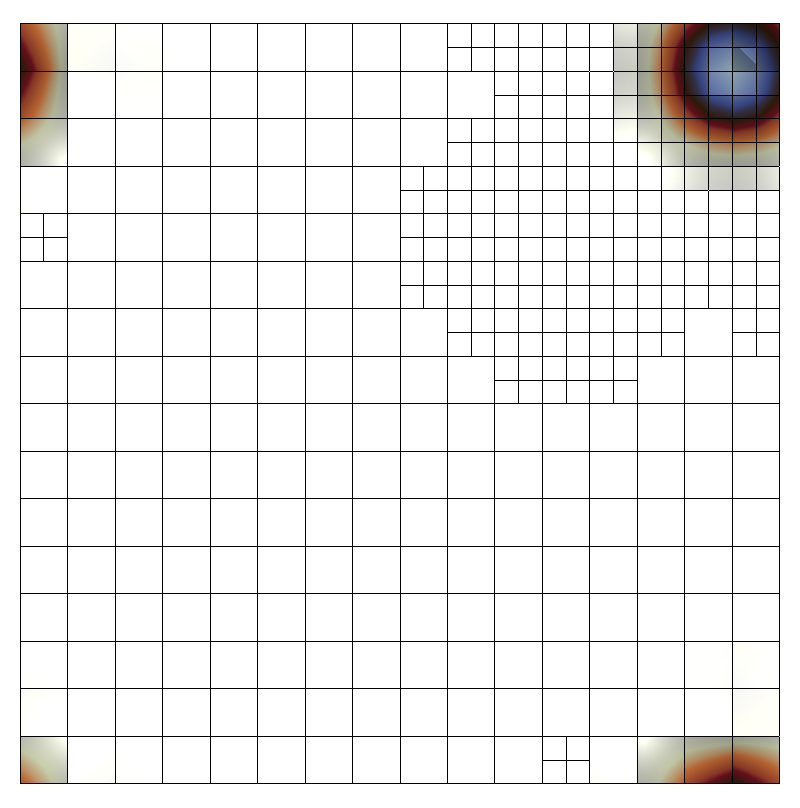}
    \captionsetup{labelformat=empty}
    \caption{$t=9$}
\end{subfigure}

\begin{subfigure}[t]{0.33\linewidth}
    \centering
    \includegraphics[width=0.49\linewidth]{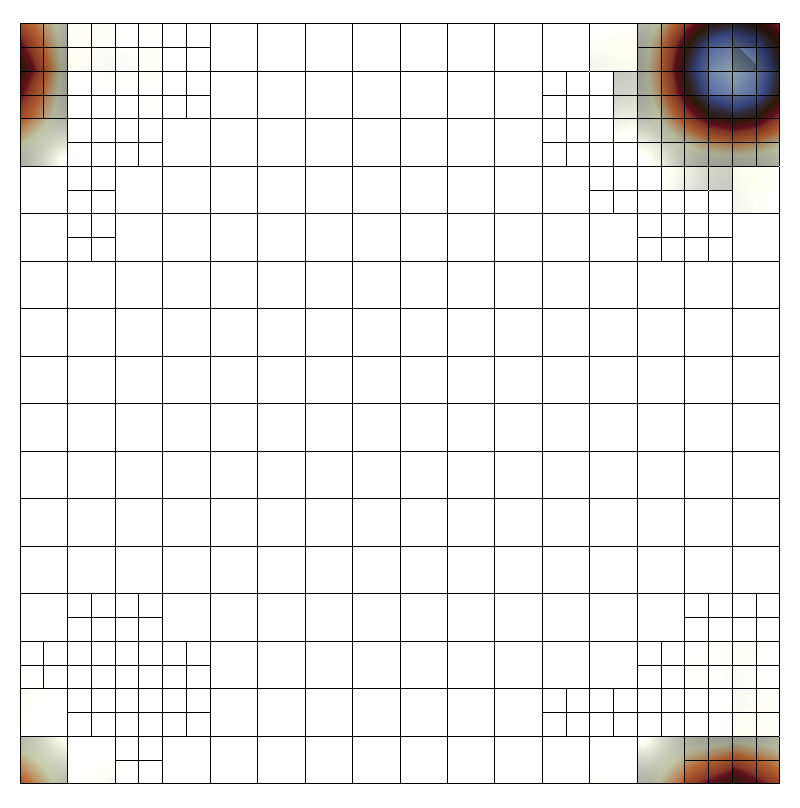}
    \includegraphics[width=0.49\linewidth]{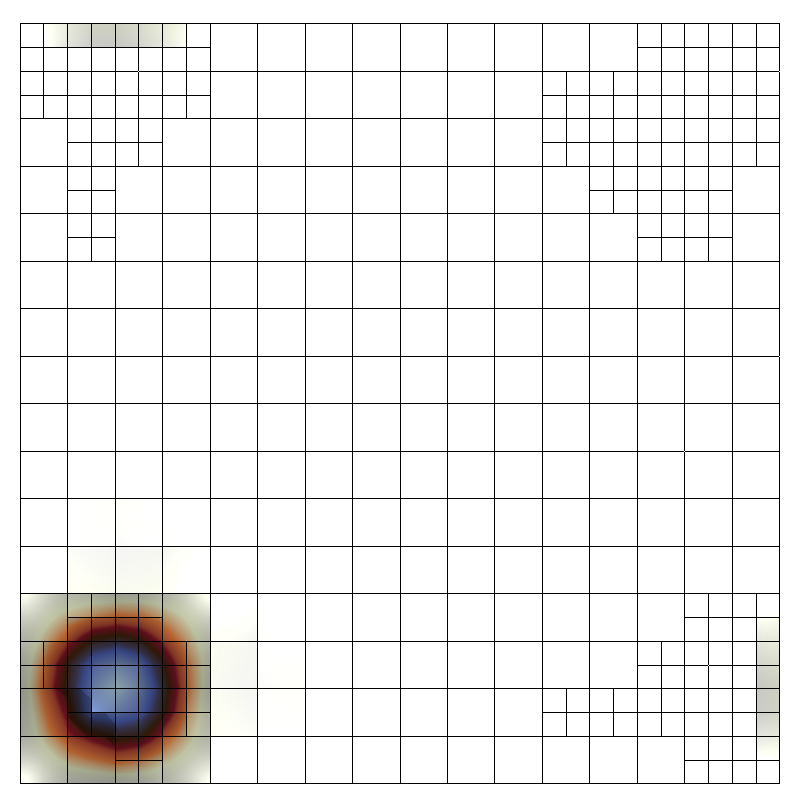}
    \captionsetup{labelformat=empty}
    \caption{$t=10$}
\end{subfigure}
\hfill
\begin{subfigure}[t]{0.33\linewidth}
    \centering
    \includegraphics[width=0.49\linewidth]{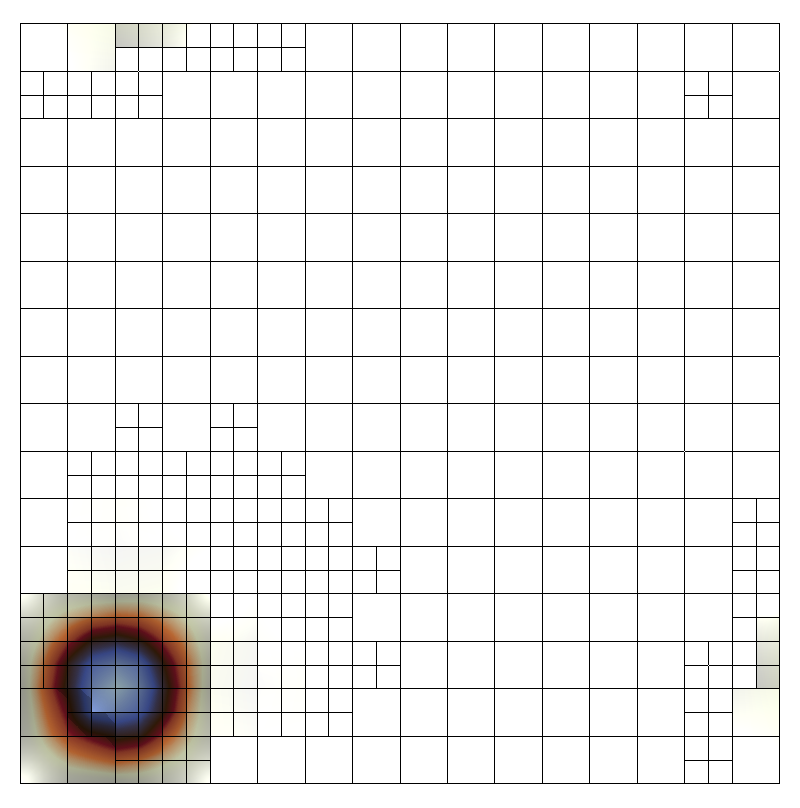}
    \includegraphics[width=0.49\linewidth]{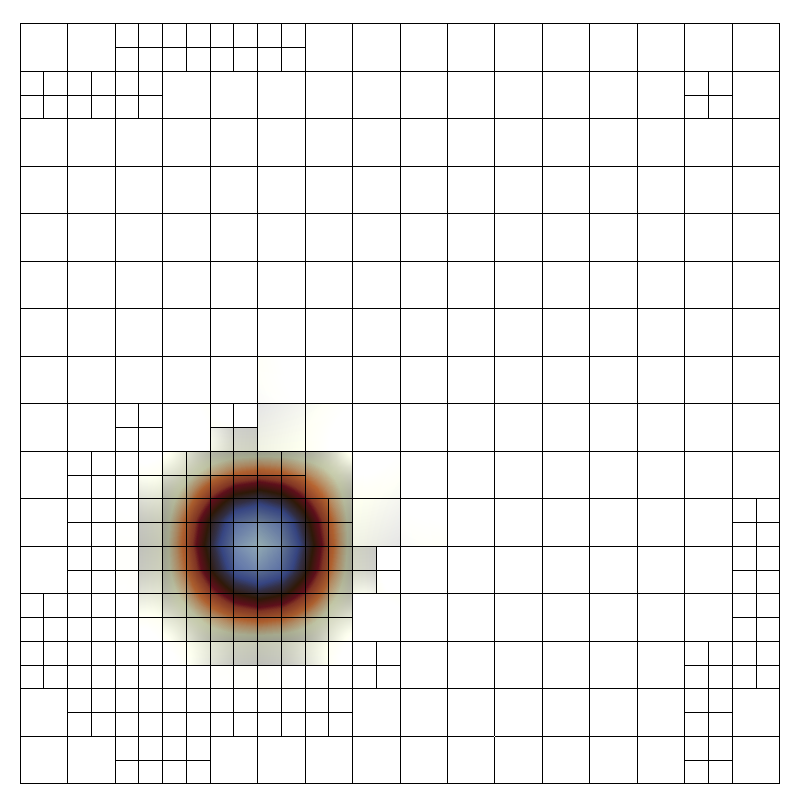}
    \captionsetup{labelformat=empty}
    \caption{$t=11$}
\end{subfigure}
\hfill
\begin{subfigure}[t]{0.33\linewidth}
    \centering
    \includegraphics[width=0.49\linewidth]{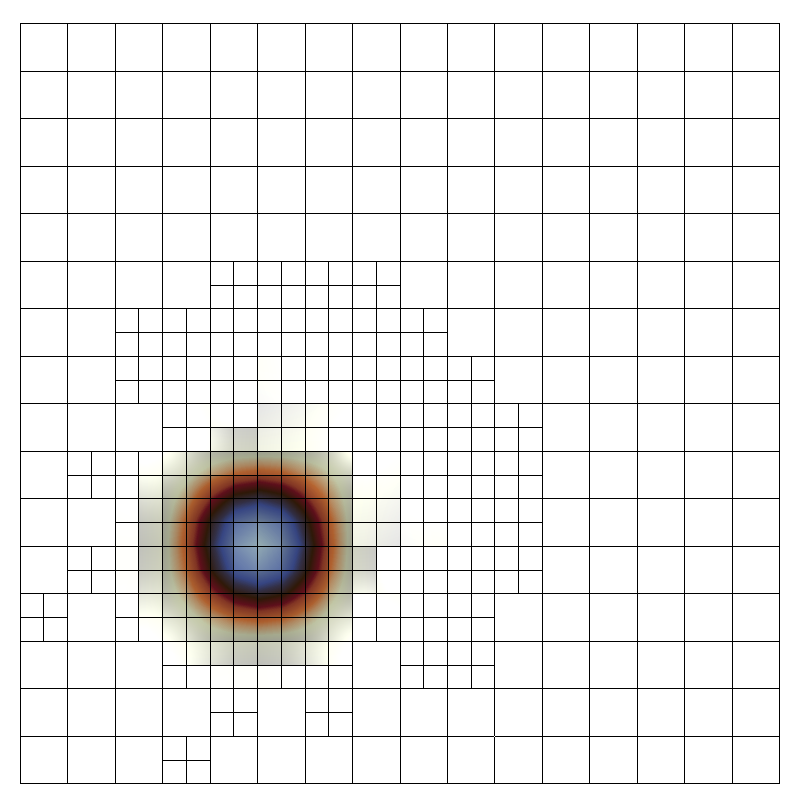}
    \includegraphics[width=0.49\linewidth]{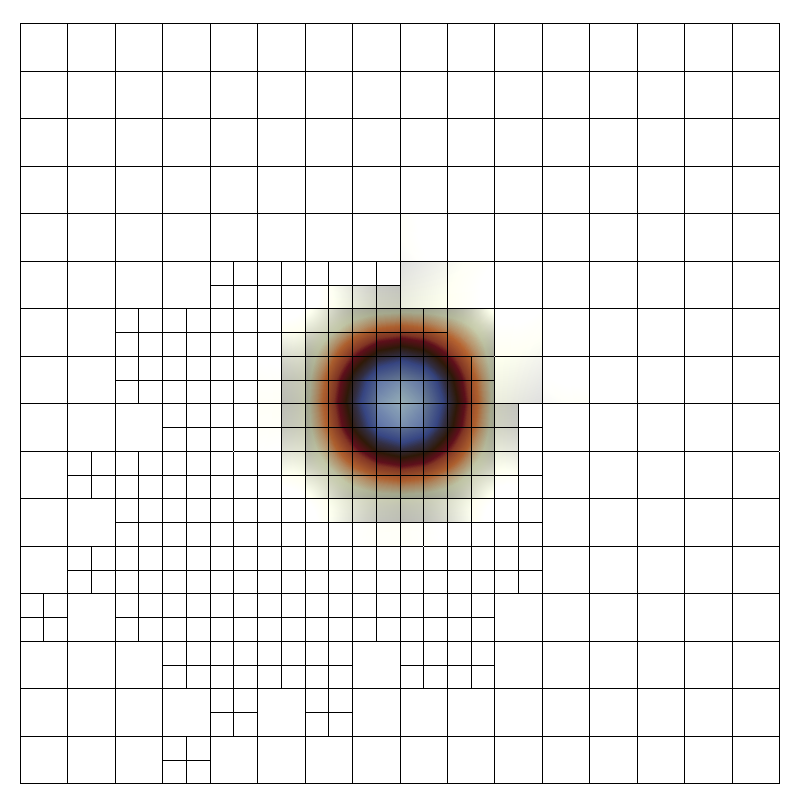}
    \captionsetup{labelformat=empty}
    \caption{$t=12$}
\end{subfigure}

\begin{subfigure}[t]{0.33\linewidth}
    \centering
    \includegraphics[width=0.49\linewidth]{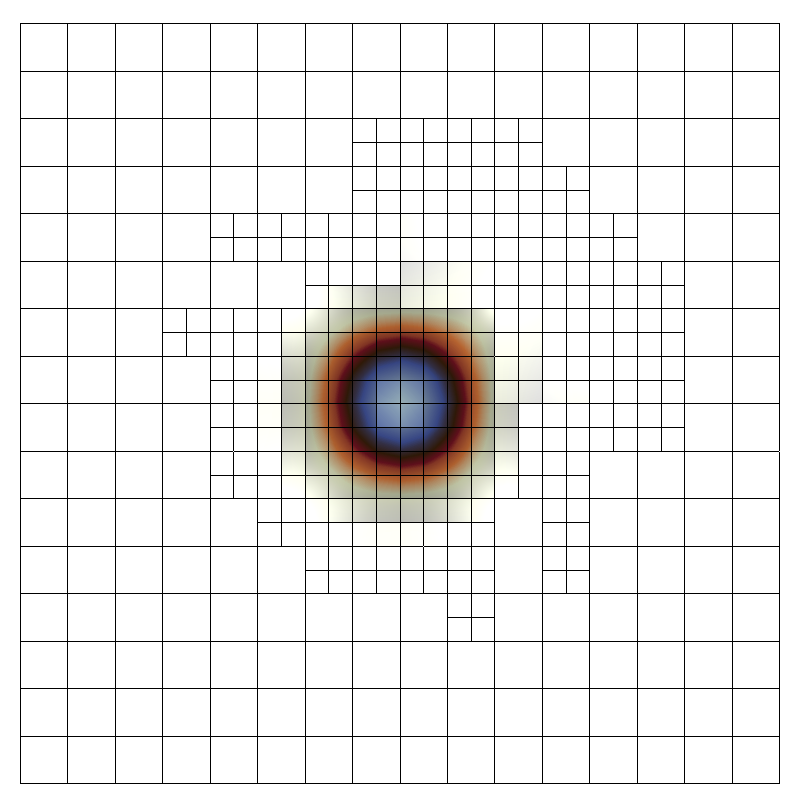}
    \includegraphics[width=0.49\linewidth]{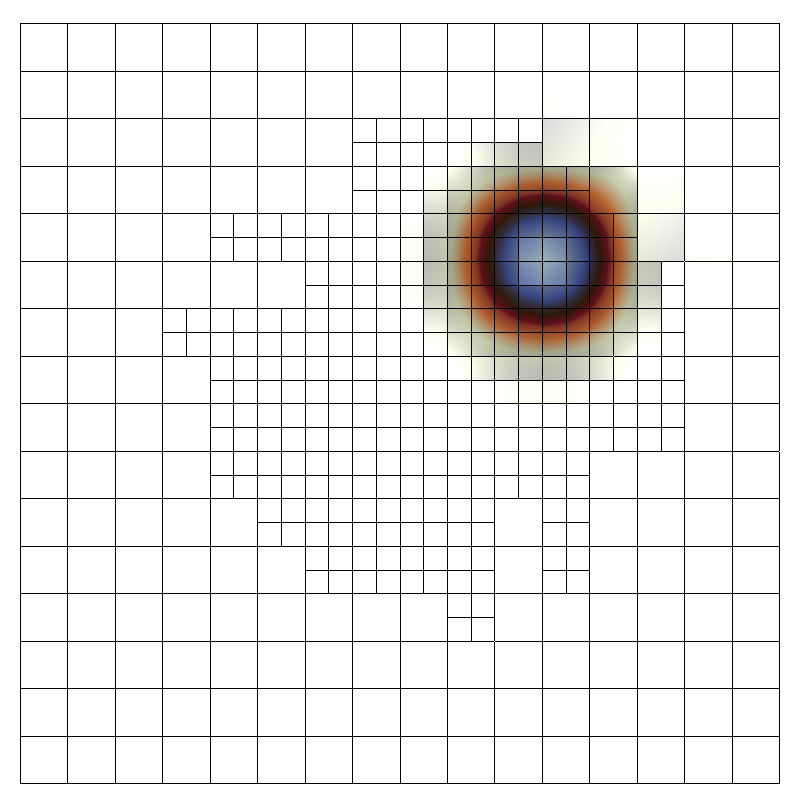}
    \captionsetup{labelformat=empty}
    \caption{$t=13$}
\end{subfigure}
\hfill
\begin{subfigure}[t]{0.33\linewidth}
    \centering
    \includegraphics[width=0.49\linewidth]{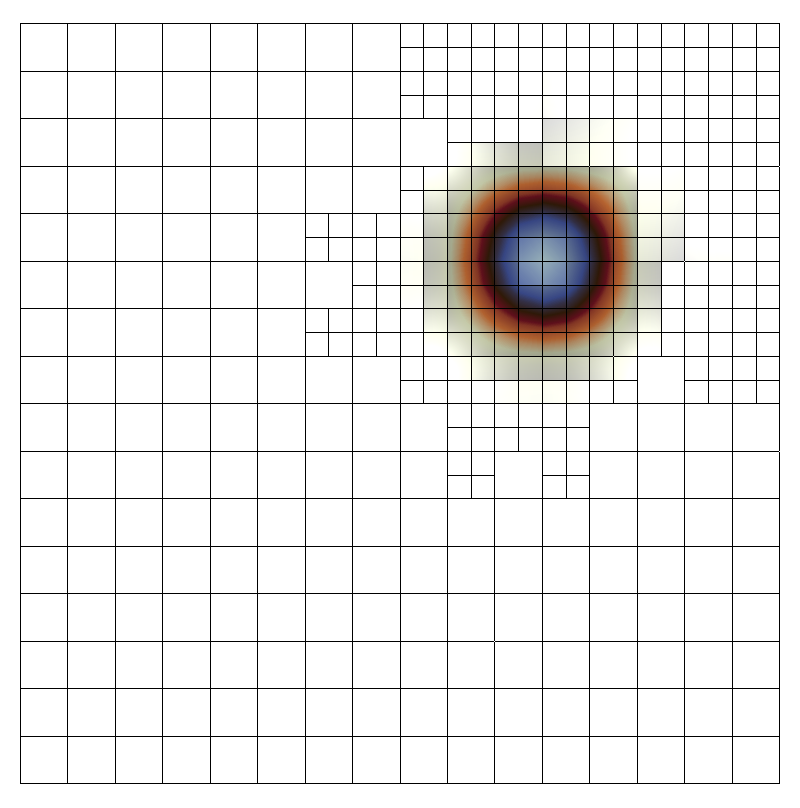}
    \includegraphics[width=0.49\linewidth]{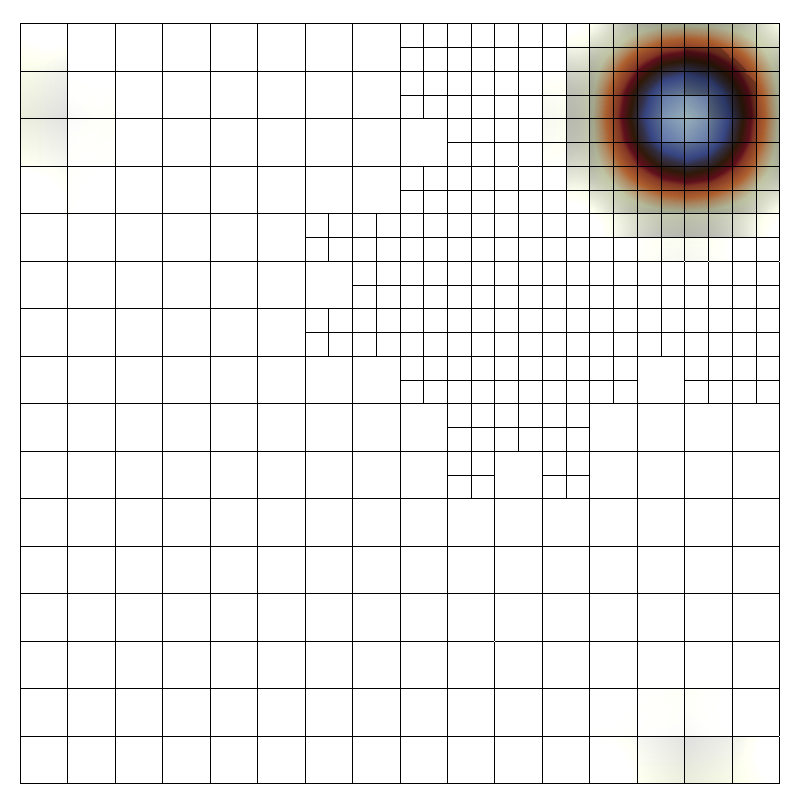}
    \captionsetup{labelformat=empty}
    \caption{$t=14$}
\end{subfigure}
\hfill
\begin{subfigure}[t]{0.33\linewidth}
    \centering
    \includegraphics[width=0.49\linewidth]{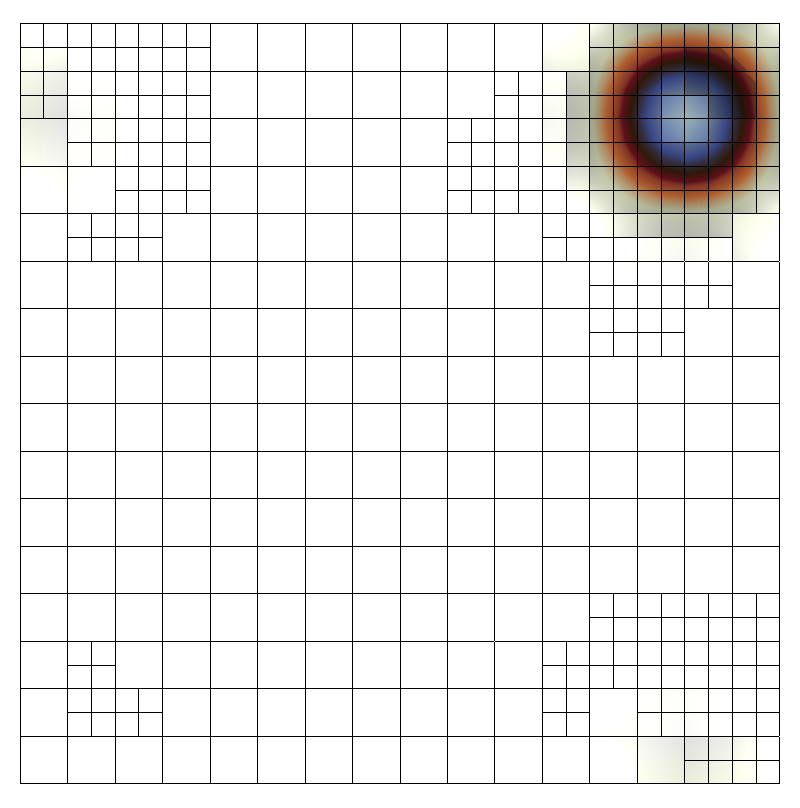}
    \includegraphics[width=0.49\linewidth]{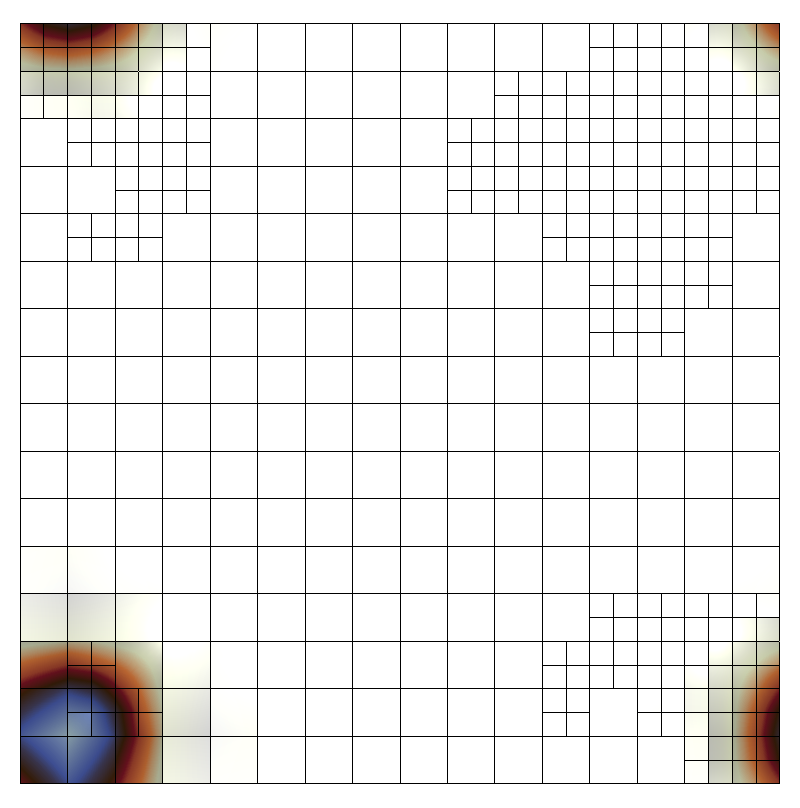}
    \captionsetup{labelformat=empty}
    \caption{$t=15$}
\end{subfigure}

\begin{subfigure}[t]{0.33\linewidth}
    \centering
    \includegraphics[width=0.49\linewidth]{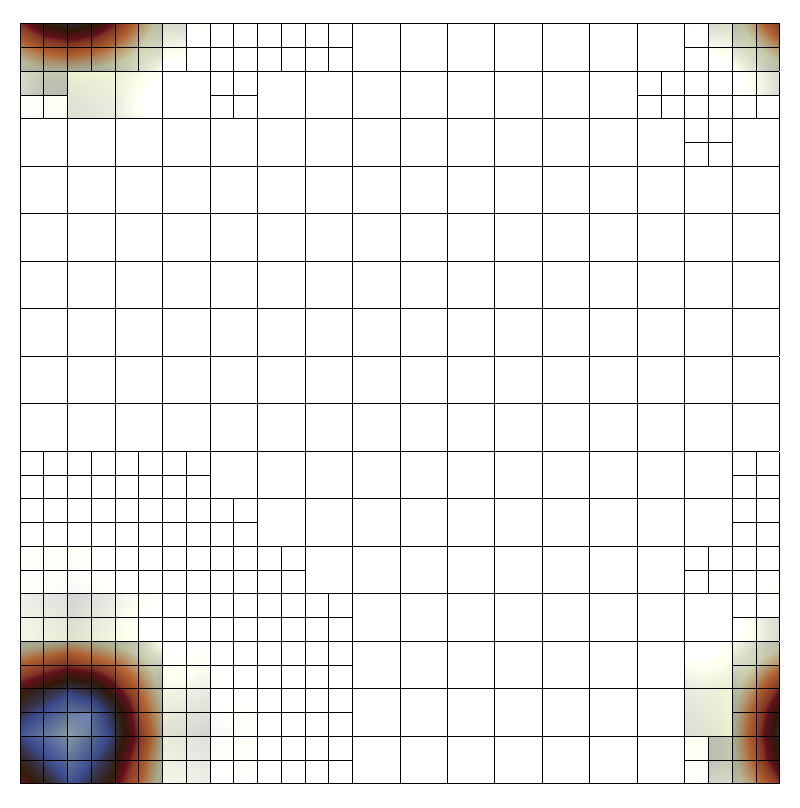}
    \includegraphics[width=0.49\linewidth]{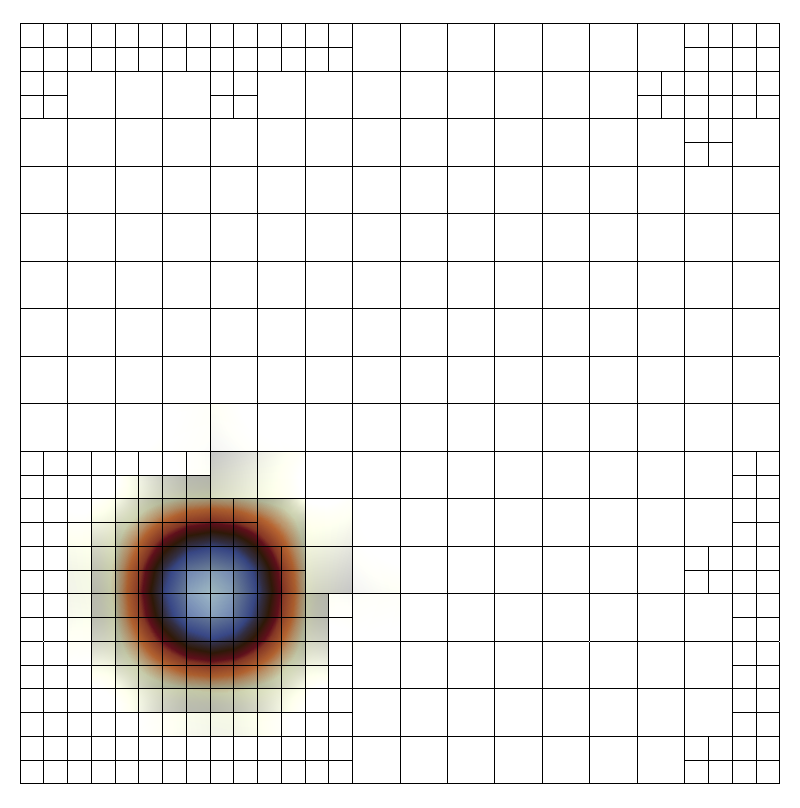}
    \captionsetup{labelformat=empty}
    \caption{$t=16$}
\end{subfigure}
\hfill
\begin{subfigure}[t]{0.33\linewidth}
    \centering
    \includegraphics[width=0.49\linewidth]{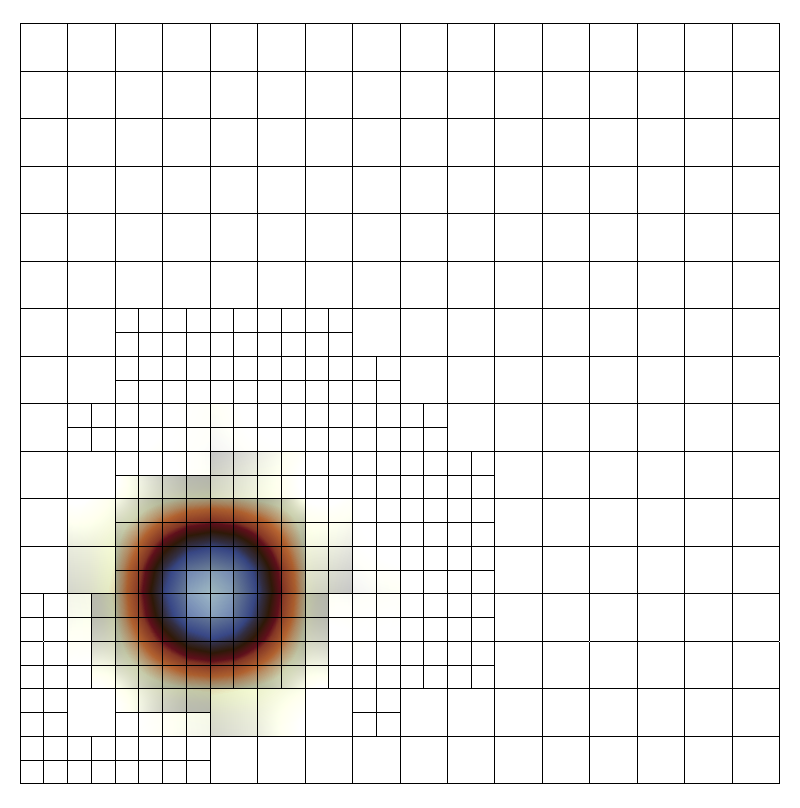}
    \includegraphics[width=0.49\linewidth]{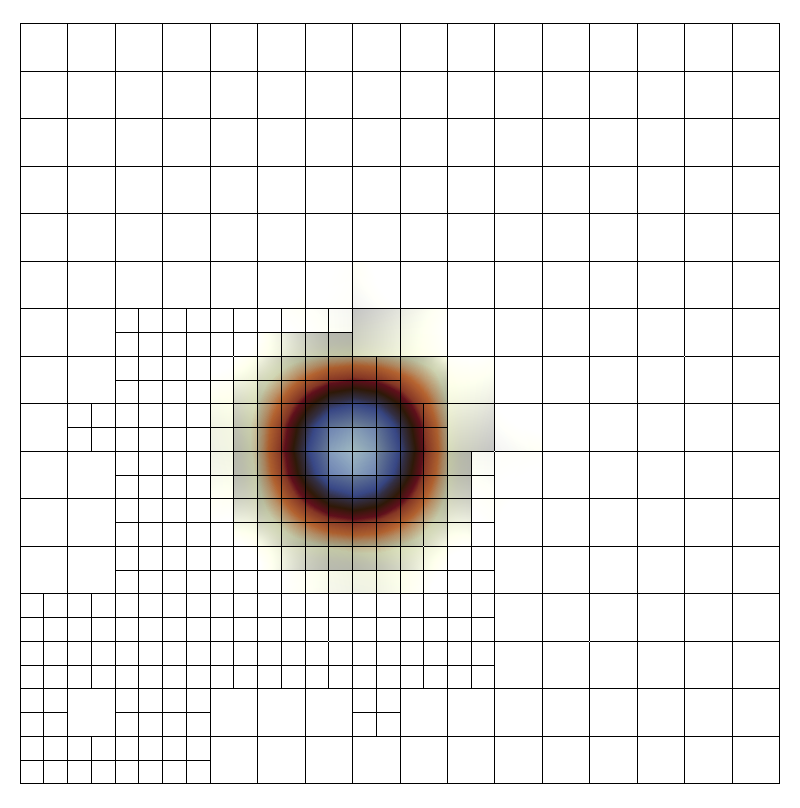}
    \captionsetup{labelformat=empty}
    \caption{$t=17$}
\end{subfigure}
\hfill
\begin{subfigure}[t]{0.33\linewidth}
    \centering
    \includegraphics[width=0.49\linewidth]{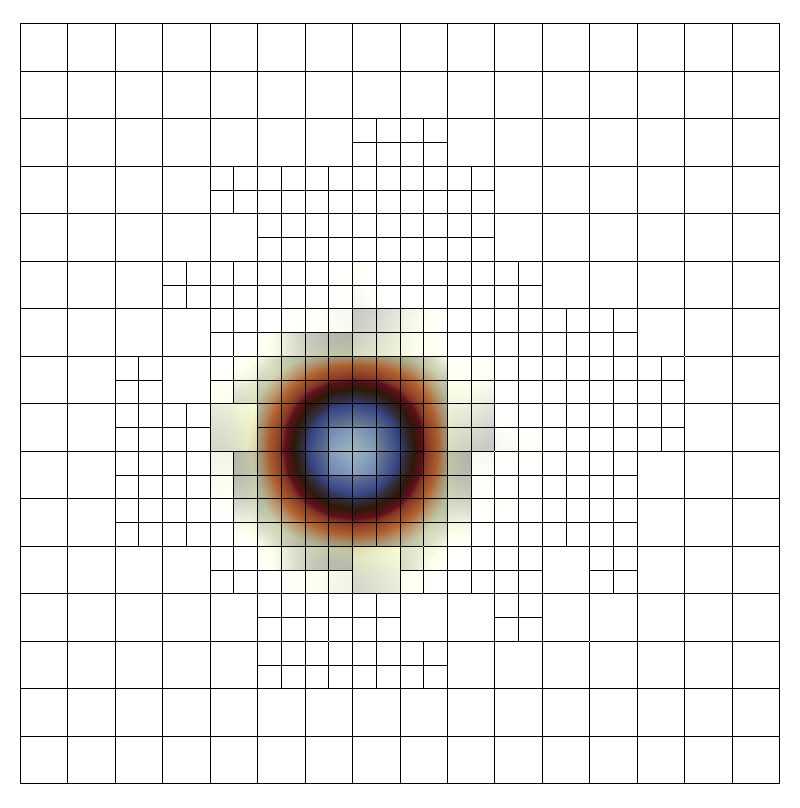}
    \includegraphics[width=0.49\linewidth]{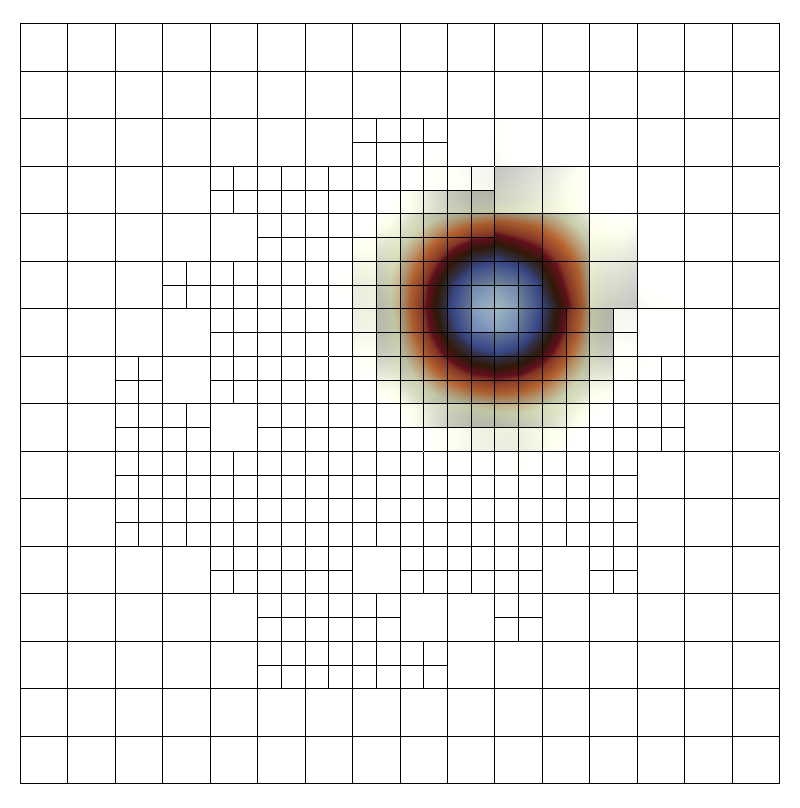}
    \captionsetup{labelformat=empty}
    \caption{$t=18$}
\end{subfigure}
\caption{VDGN generalizes to longer simulation durations, despite seeing only up to $t=3$ during training episodes.}
\label{fig:iso_periodic_velunif_nx16_ny16_depth1_tstep0p25_vdn_graphnet_nodoftime_3_ep210800_tfinal5}
\end{figure*}

\newpage
%%%%%%% Larger initial mesh time %%%%%%%%%%%%%%
\begin{figure*}[ht]
\centering
\begin{subfigure}[t]{1.0\linewidth}
    \centering
    \includegraphics[width=0.49\linewidth]{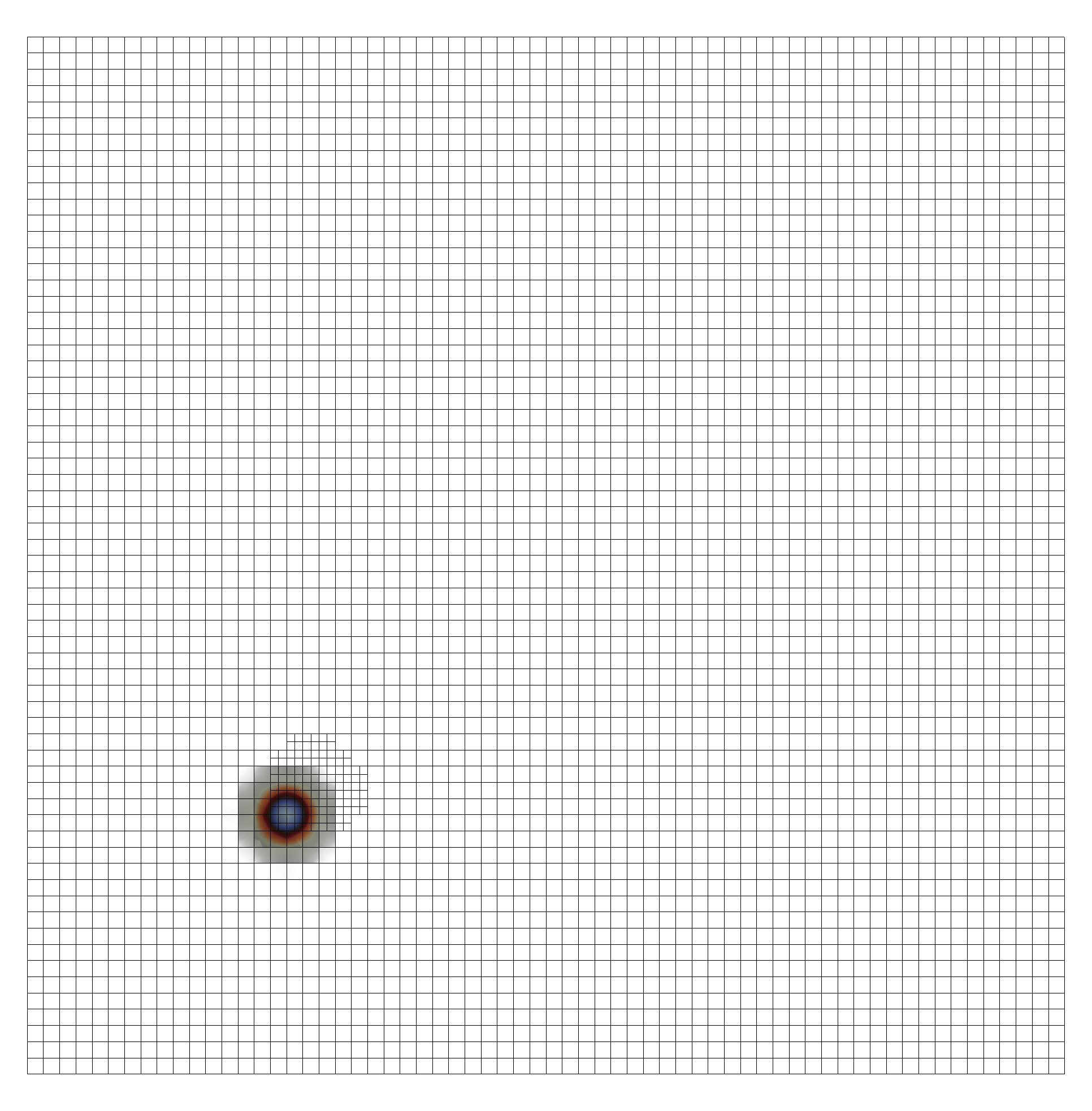}
    \includegraphics[width=0.49\linewidth]{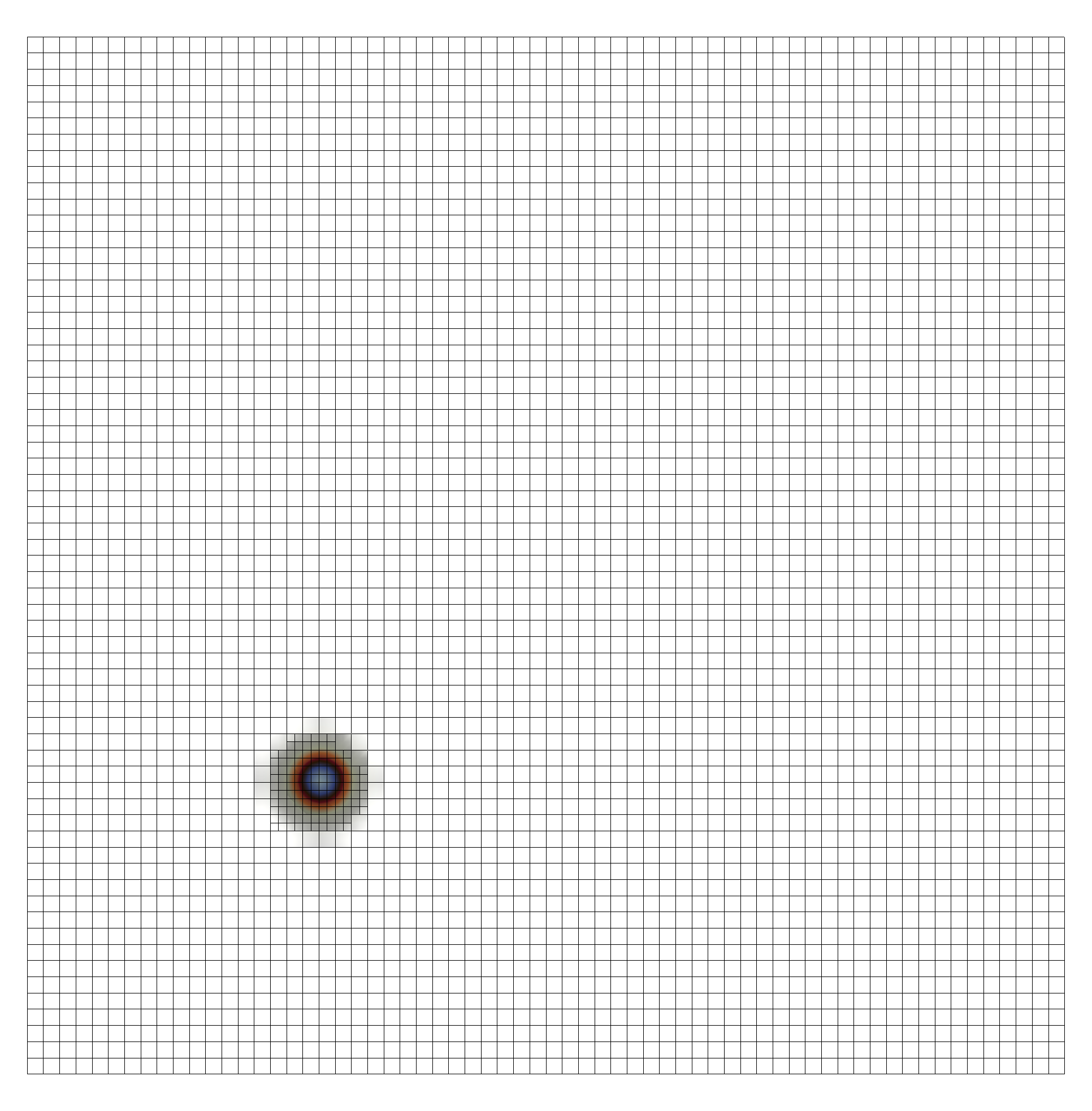}
    \captionsetup{labelformat=empty}
    \caption{$t=1$}
\end{subfigure}
\begin{subfigure}[t]{1.0\linewidth}
    \centering
    \includegraphics[width=0.49\linewidth]{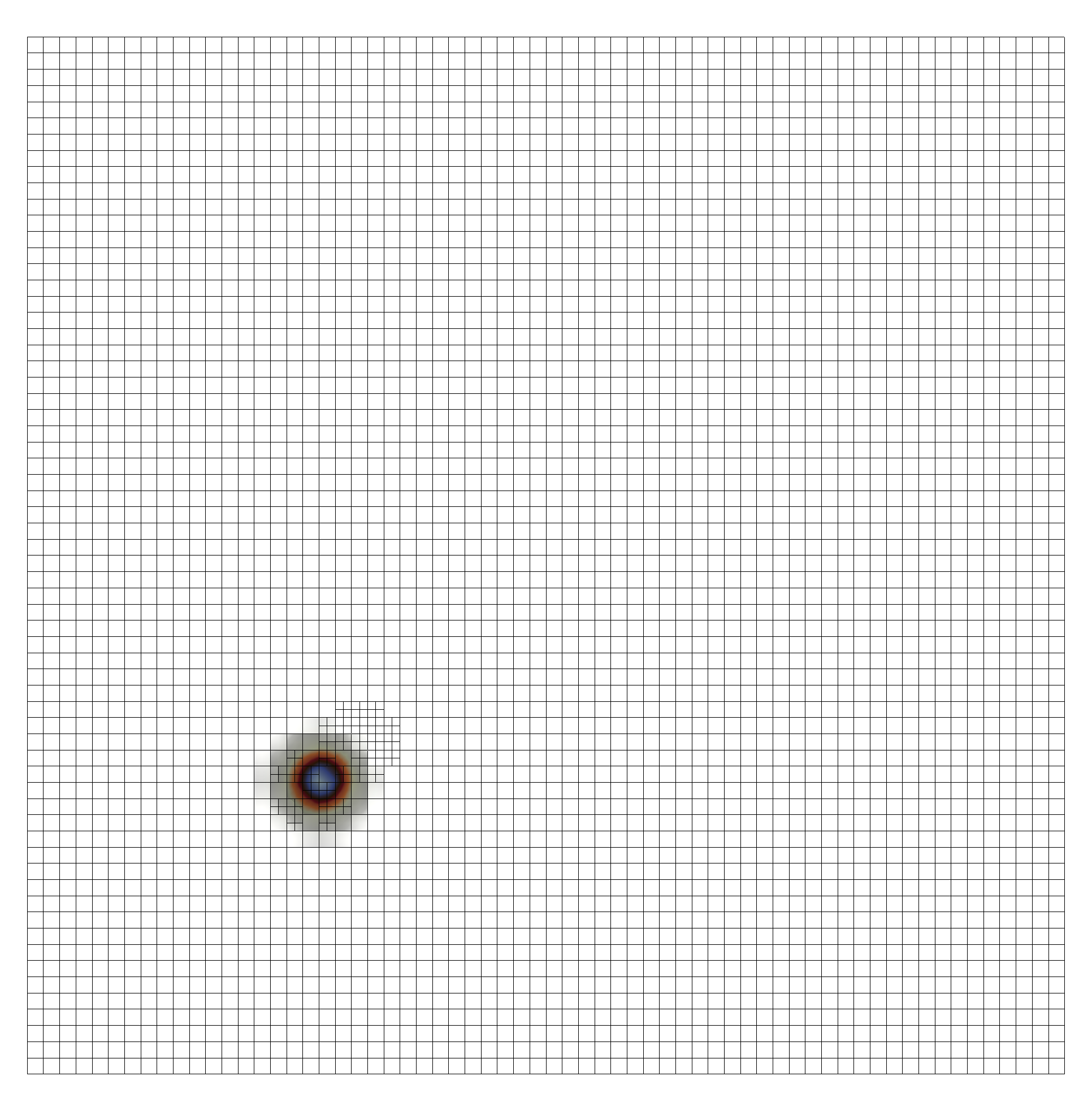}
    \includegraphics[width=0.49\linewidth]{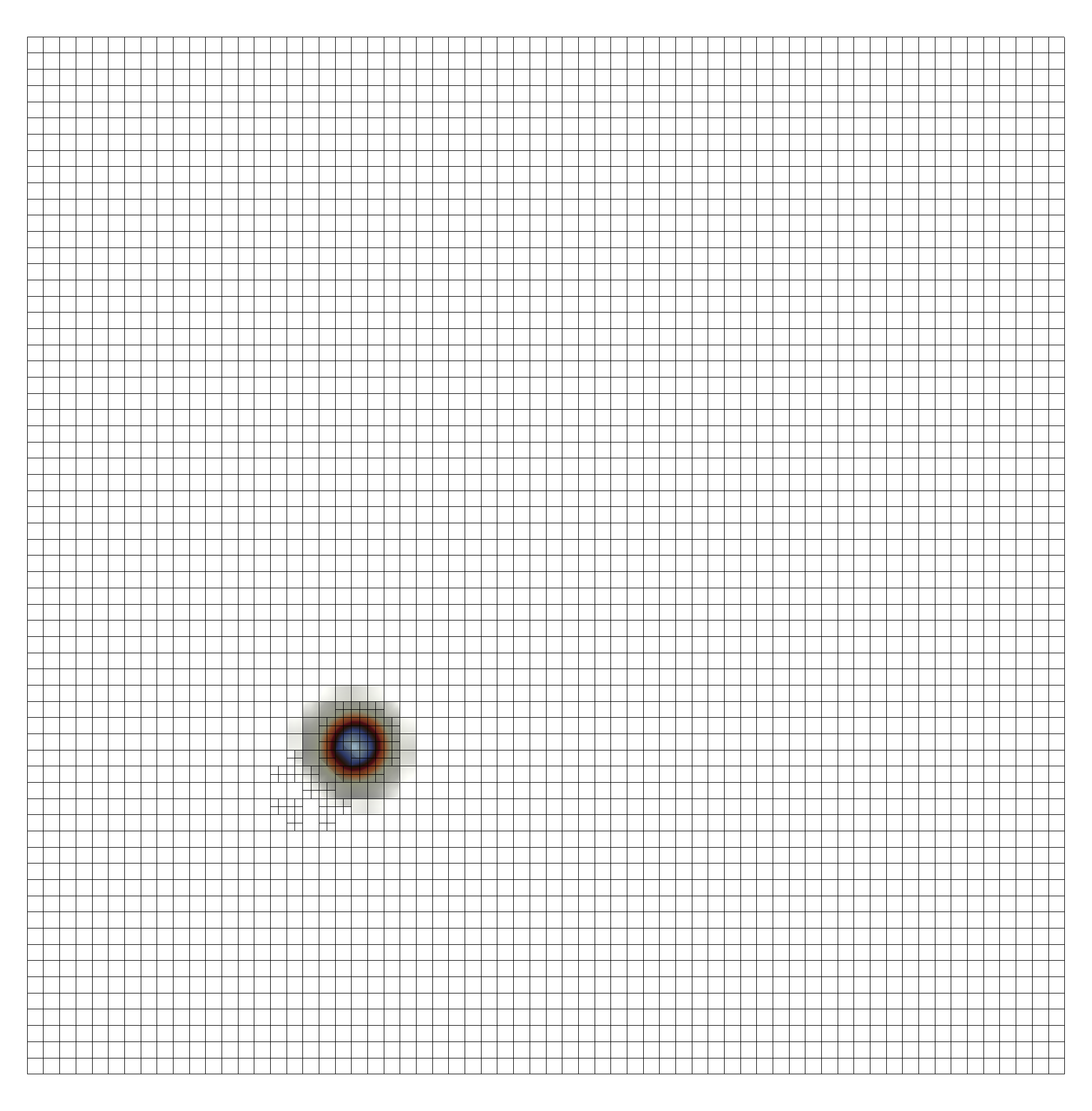}
    \captionsetup{labelformat=empty}
    \caption{$t=2$}
\end{subfigure}
\caption{VDGN generalizes to larger meshes not seen in training. See \Cref{fig:iso_periodic_velunif_nx16_ny16_depth1_tstep0p25_vdn_graphnet_nodoftime_3_ep210800_nx64_ny64_continued} for subsequent time steps.
}
\label{fig:iso_periodic_velunif_nx16_ny16_depth1_tstep0p25_vdn_graphnet_nodoftime_3_ep210800_nx64_ny64}
\end{figure*}

\begin{figure*}[ht]
\centering
\begin{subfigure}[t]{0.49\linewidth}
    \centering
    \includegraphics[width=0.49\linewidth]{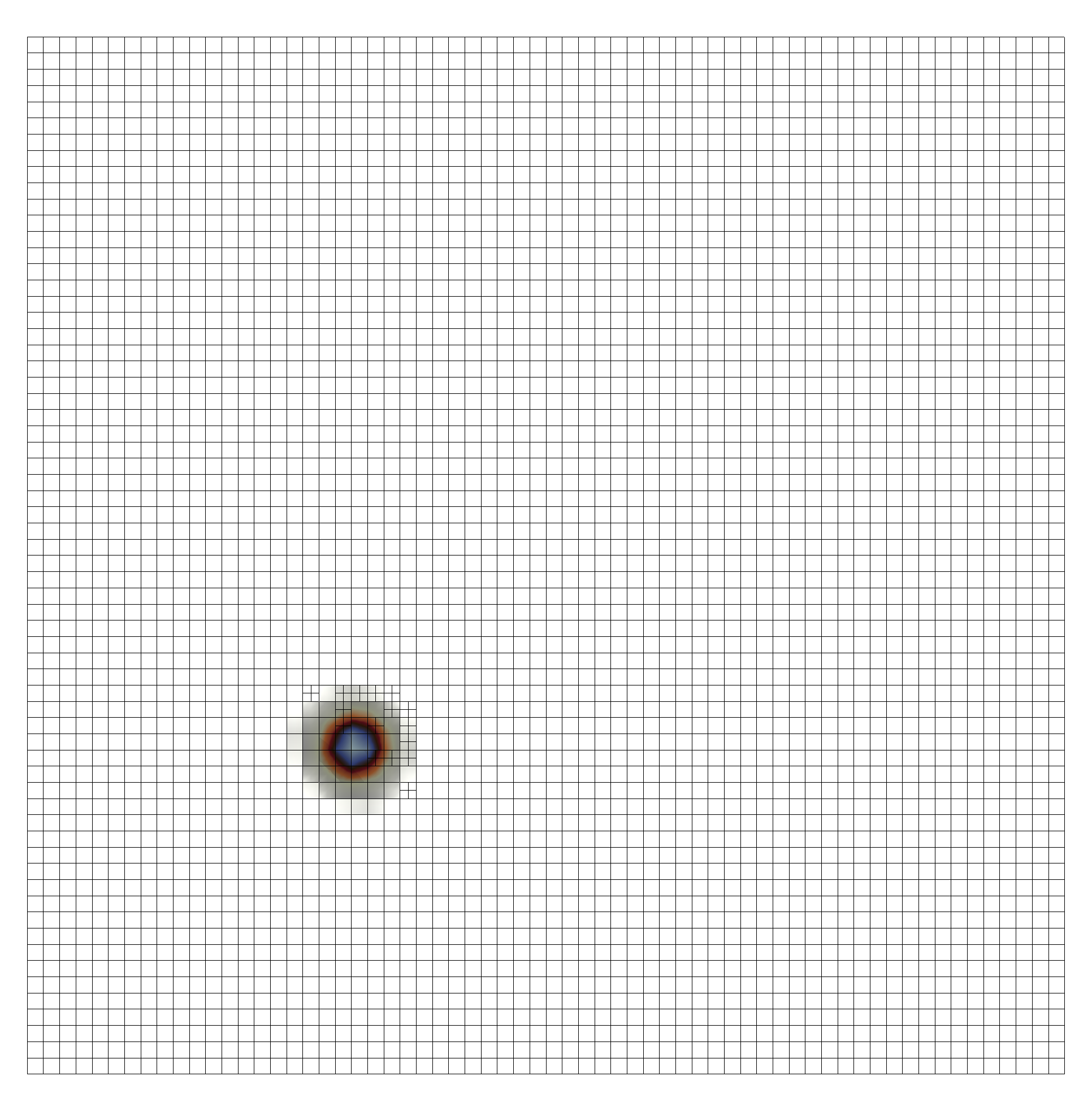}
    \includegraphics[width=0.49\linewidth]{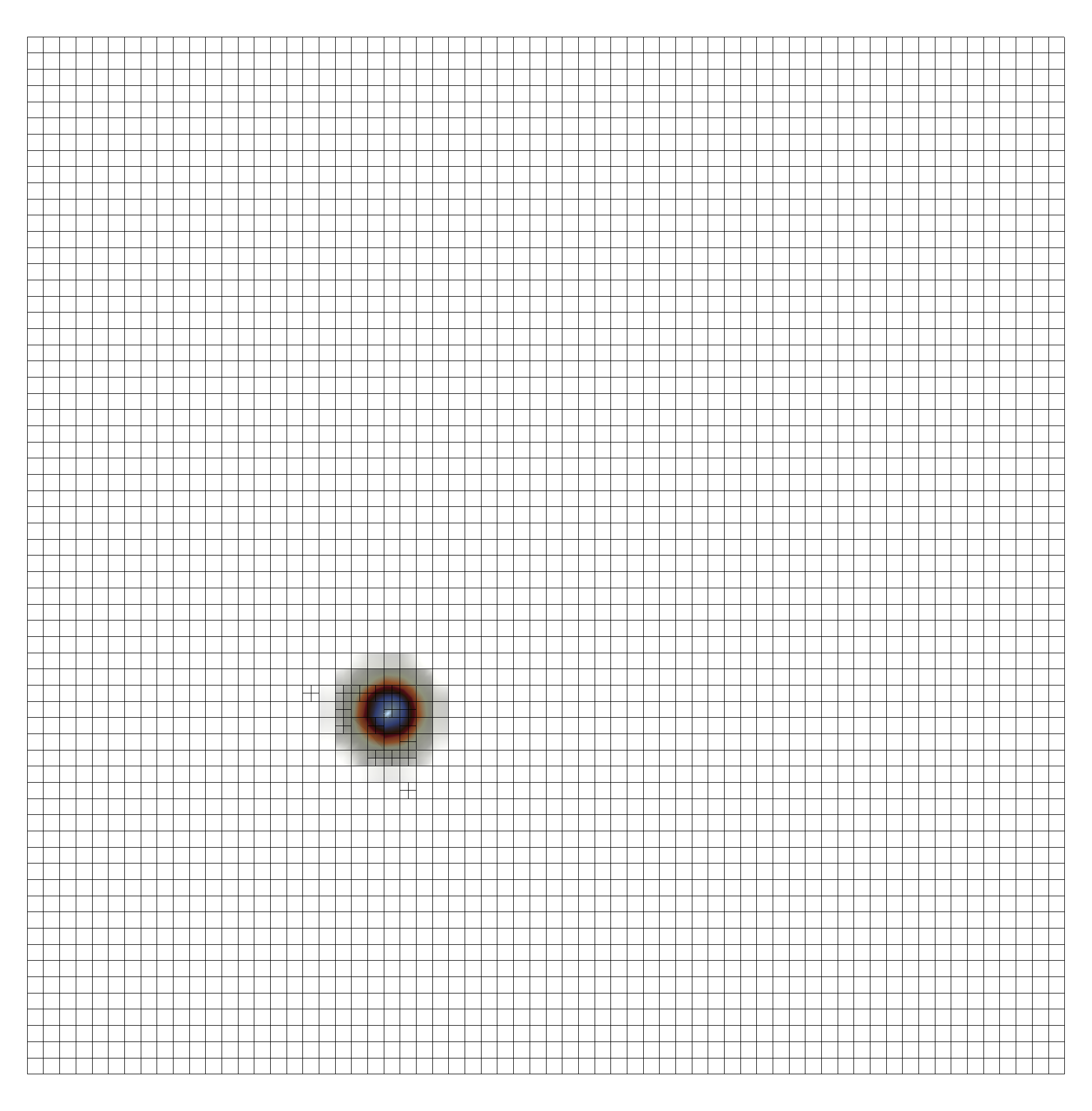}
    \captionsetup{labelformat=empty}
    \caption{$t=3$}
\end{subfigure}
\hfill
\begin{subfigure}[t]{0.49\linewidth}
    \centering
    \includegraphics[width=0.49\linewidth]{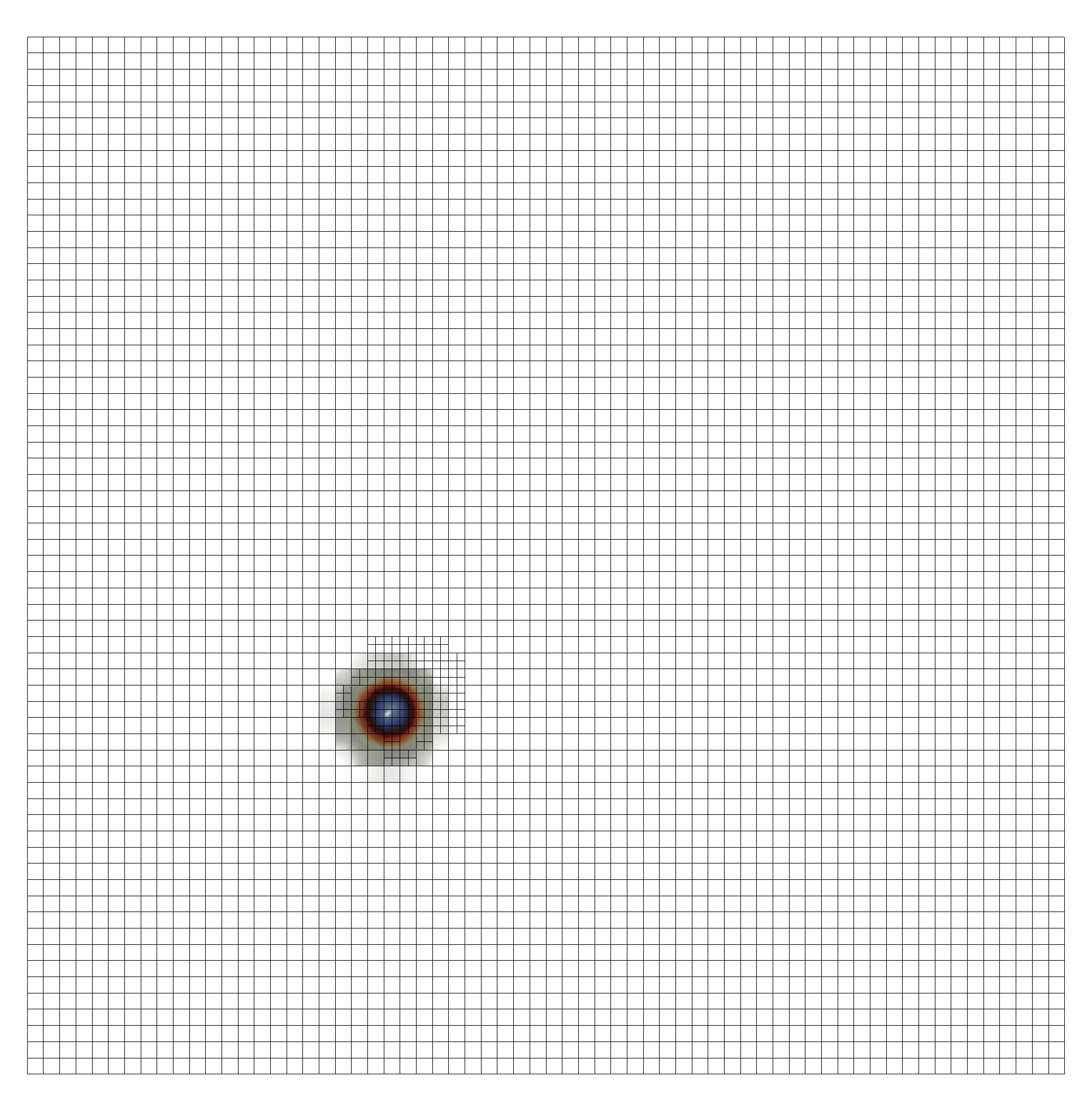}
    \includegraphics[width=0.49\linewidth]{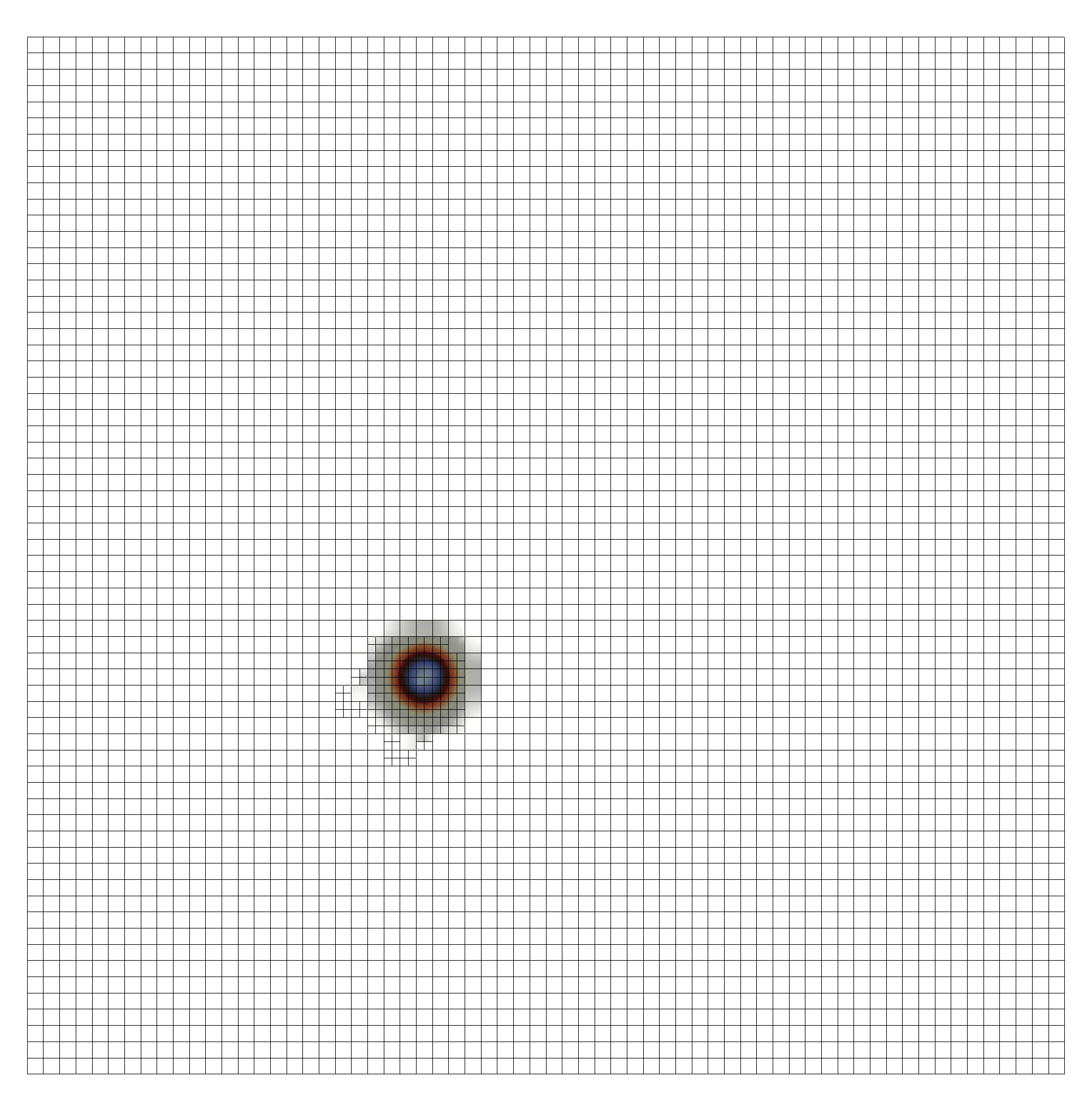}
    \captionsetup{labelformat=empty}
    \caption{$t=4$}
\end{subfigure}

\begin{subfigure}[t]{0.49\linewidth}
    \centering
    \includegraphics[width=0.49\linewidth]{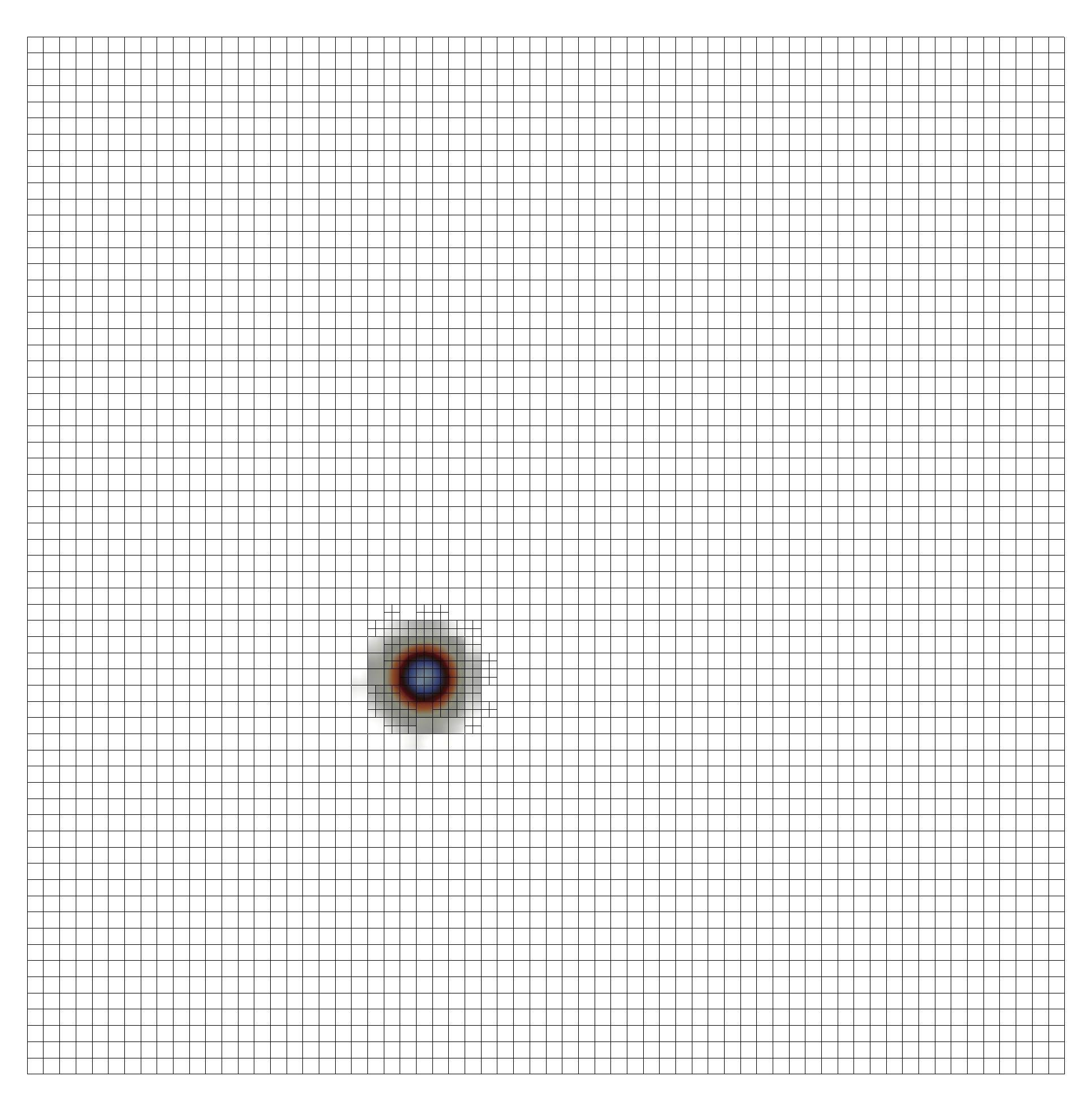}
    \includegraphics[width=0.49\linewidth]{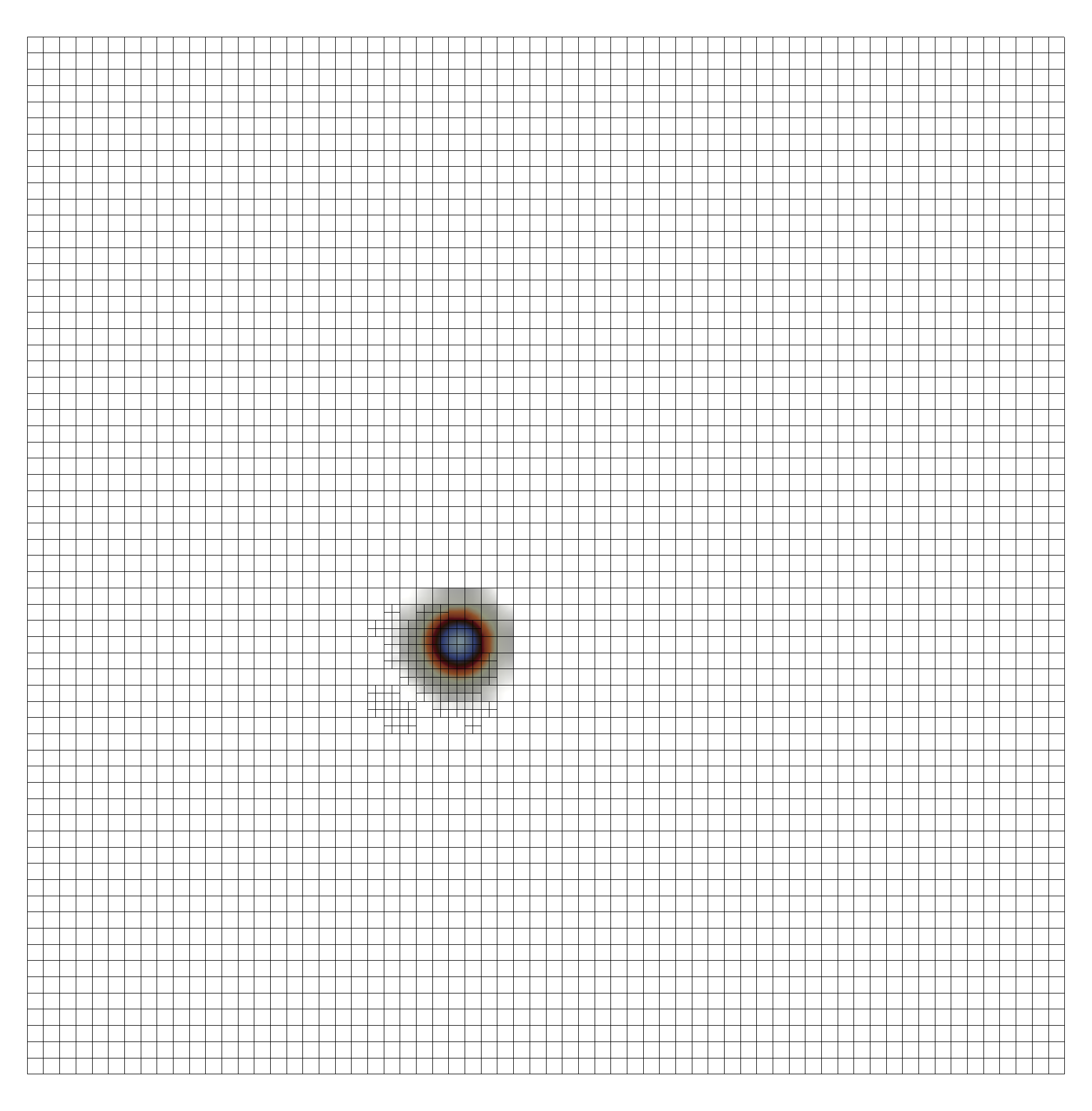}
    \captionsetup{labelformat=empty}
    \caption{$t=5$}
\end{subfigure}
\hfill
\begin{subfigure}[t]{0.49\linewidth}
    \centering
    \includegraphics[width=0.49\linewidth]{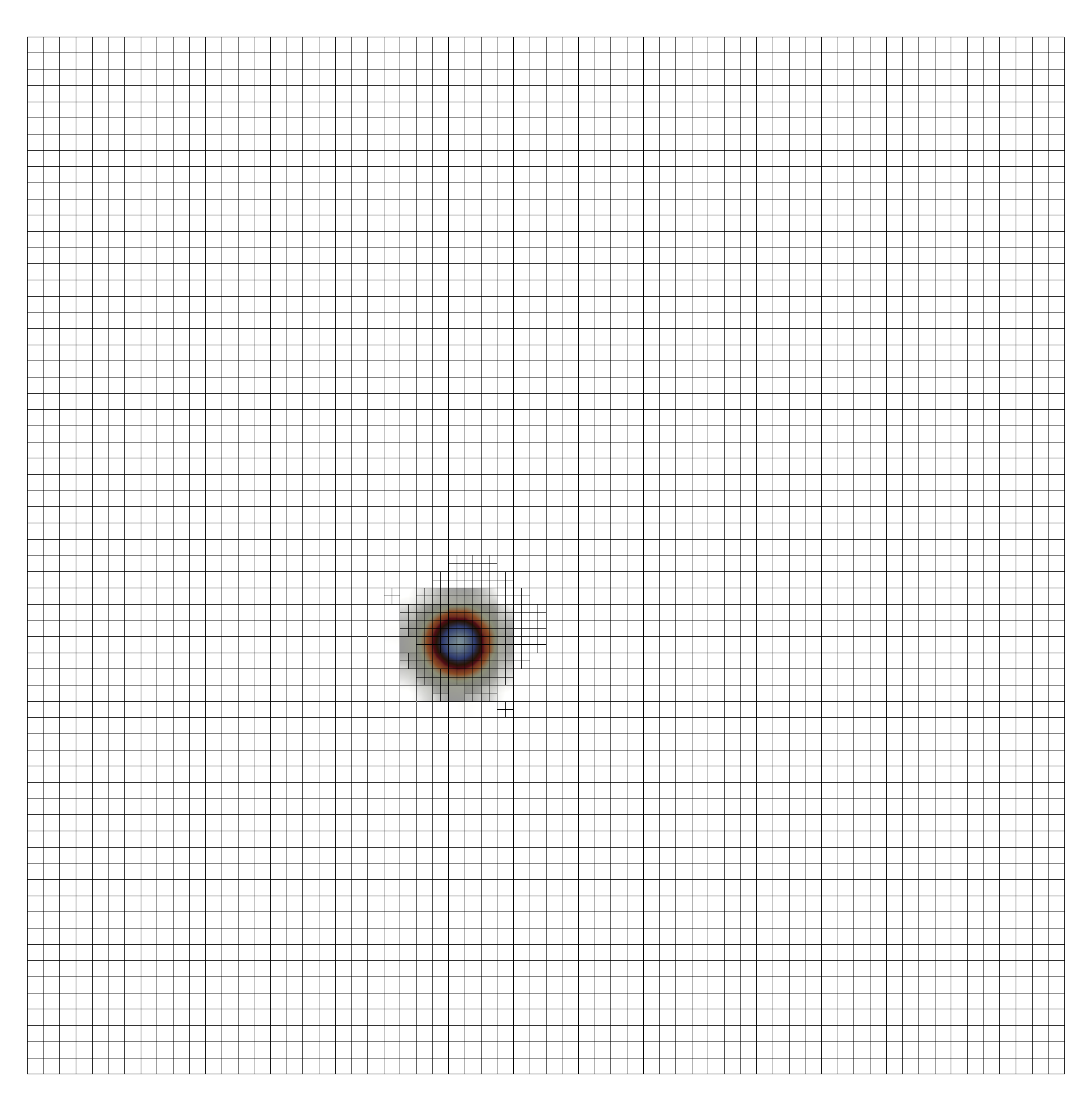}
    \includegraphics[width=0.49\linewidth]{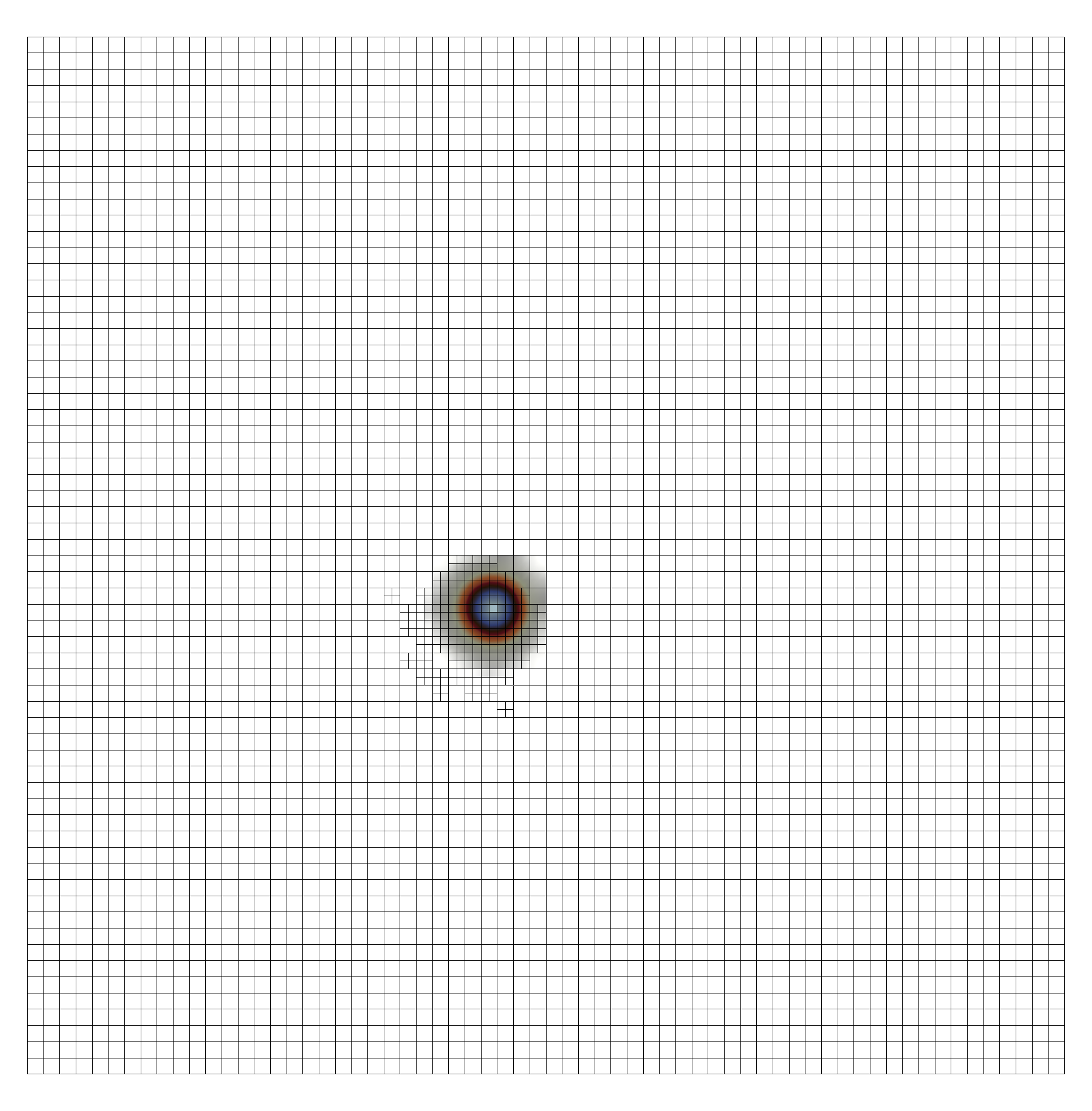}
    \captionsetup{labelformat=empty}
    \caption{$t=6$}
\end{subfigure}

\begin{subfigure}[t]{0.49\linewidth}
    \centering
    \includegraphics[width=0.49\linewidth]{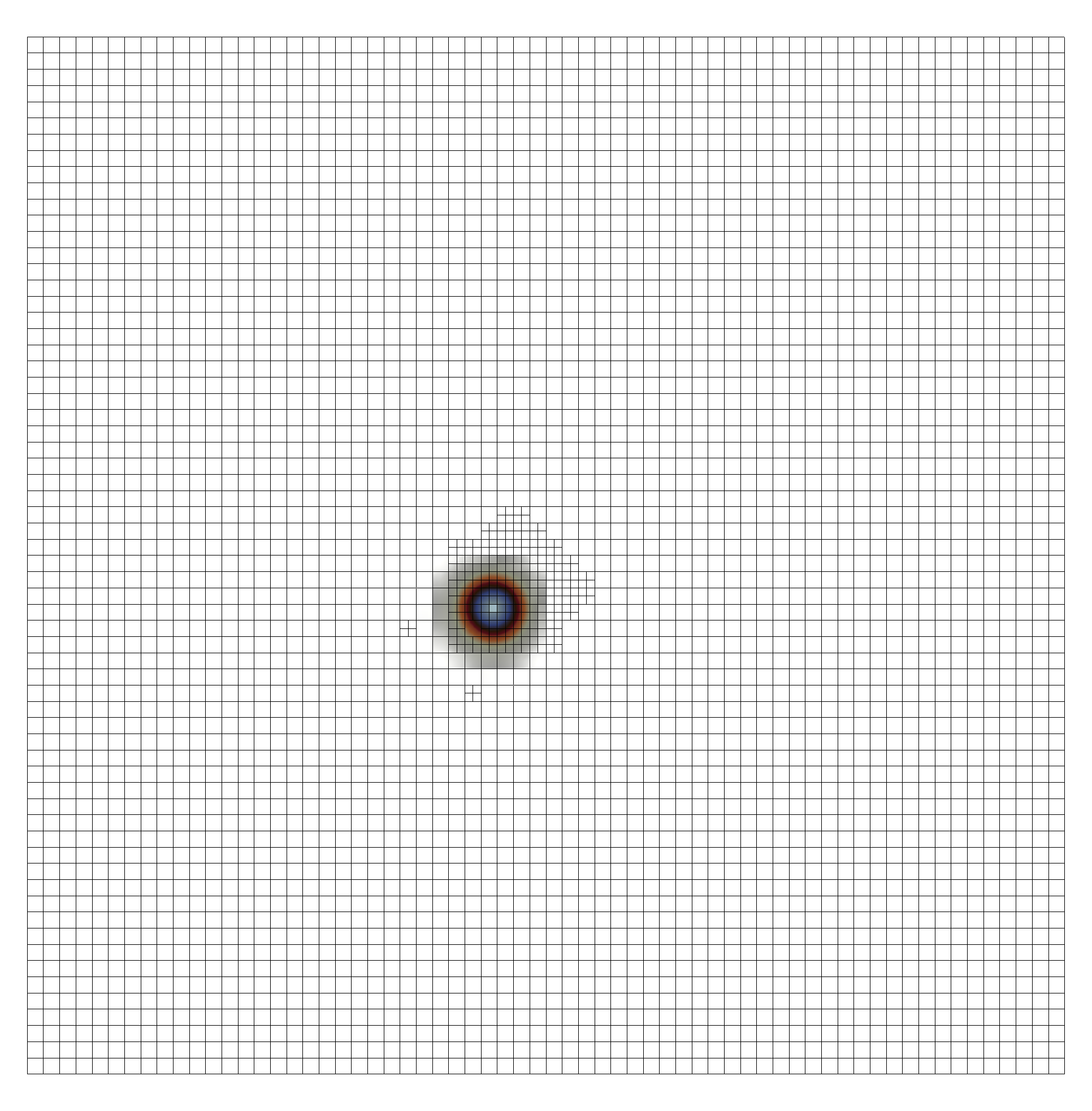}
    \includegraphics[width=0.49\linewidth]{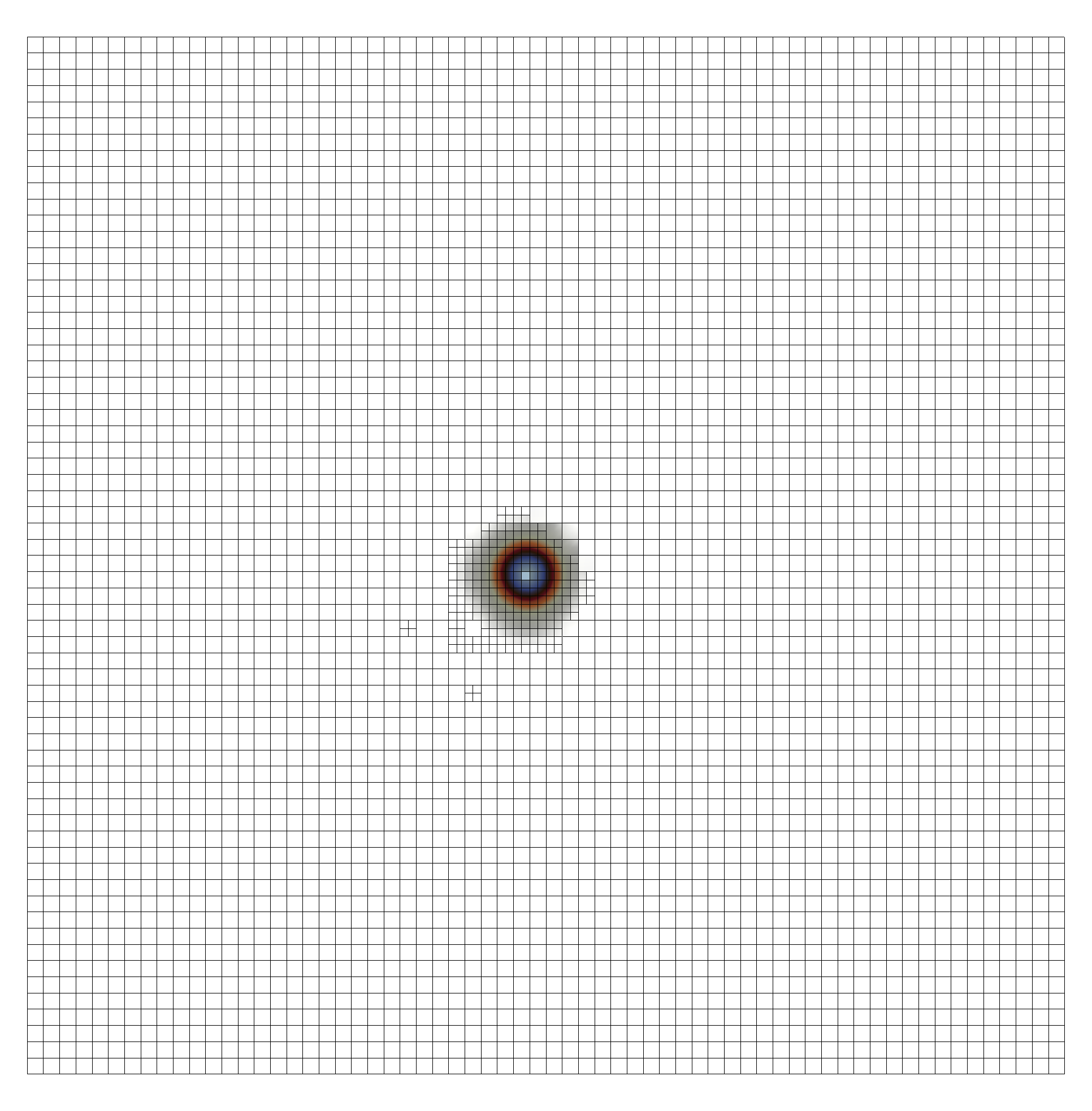}
    \captionsetup{labelformat=empty}
    \caption{$t=7$}
\end{subfigure}
\hfill
\begin{subfigure}[t]{0.49\linewidth}
    \centering
    \includegraphics[width=0.49\linewidth]{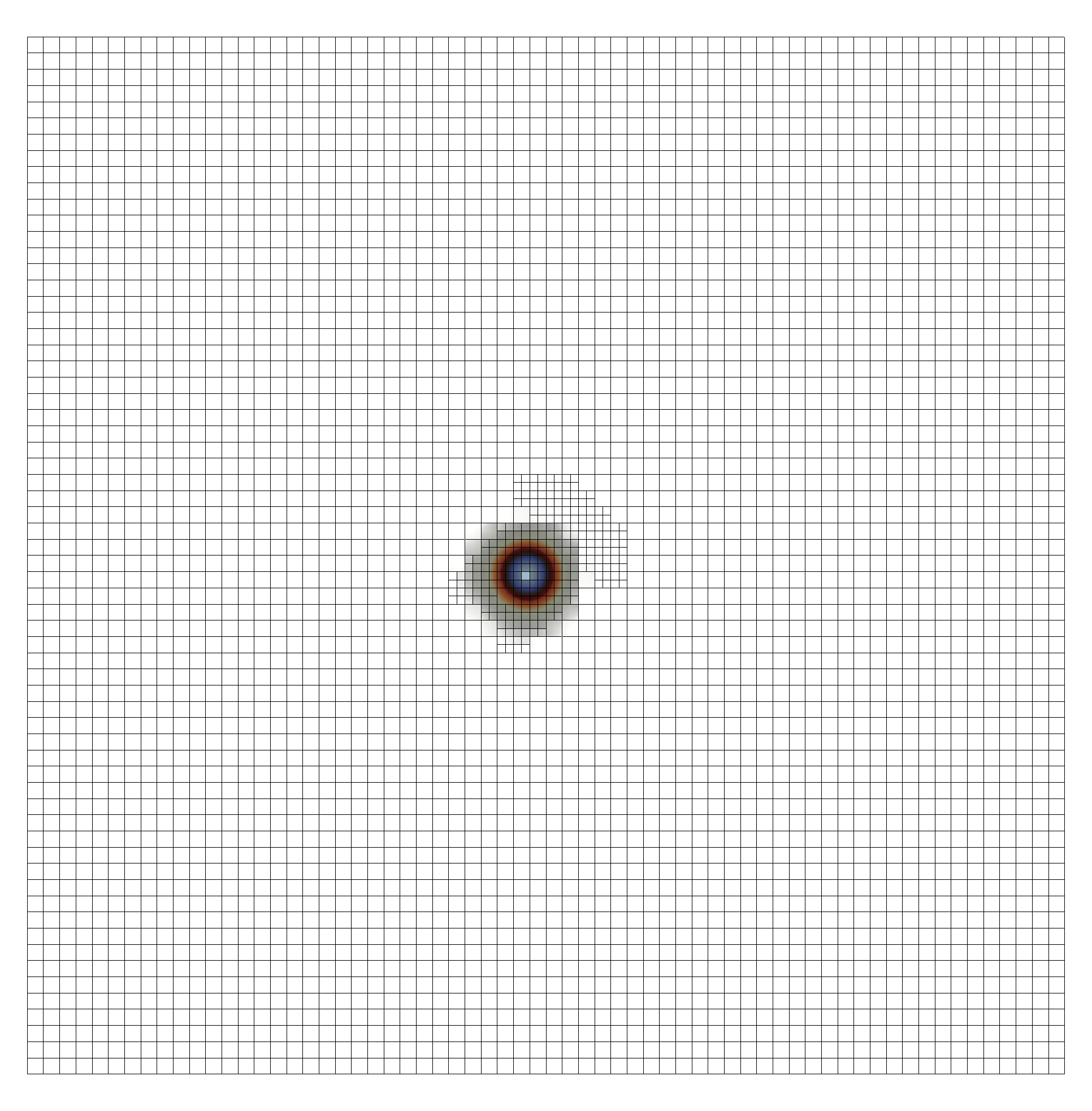}
    \includegraphics[width=0.49\linewidth]{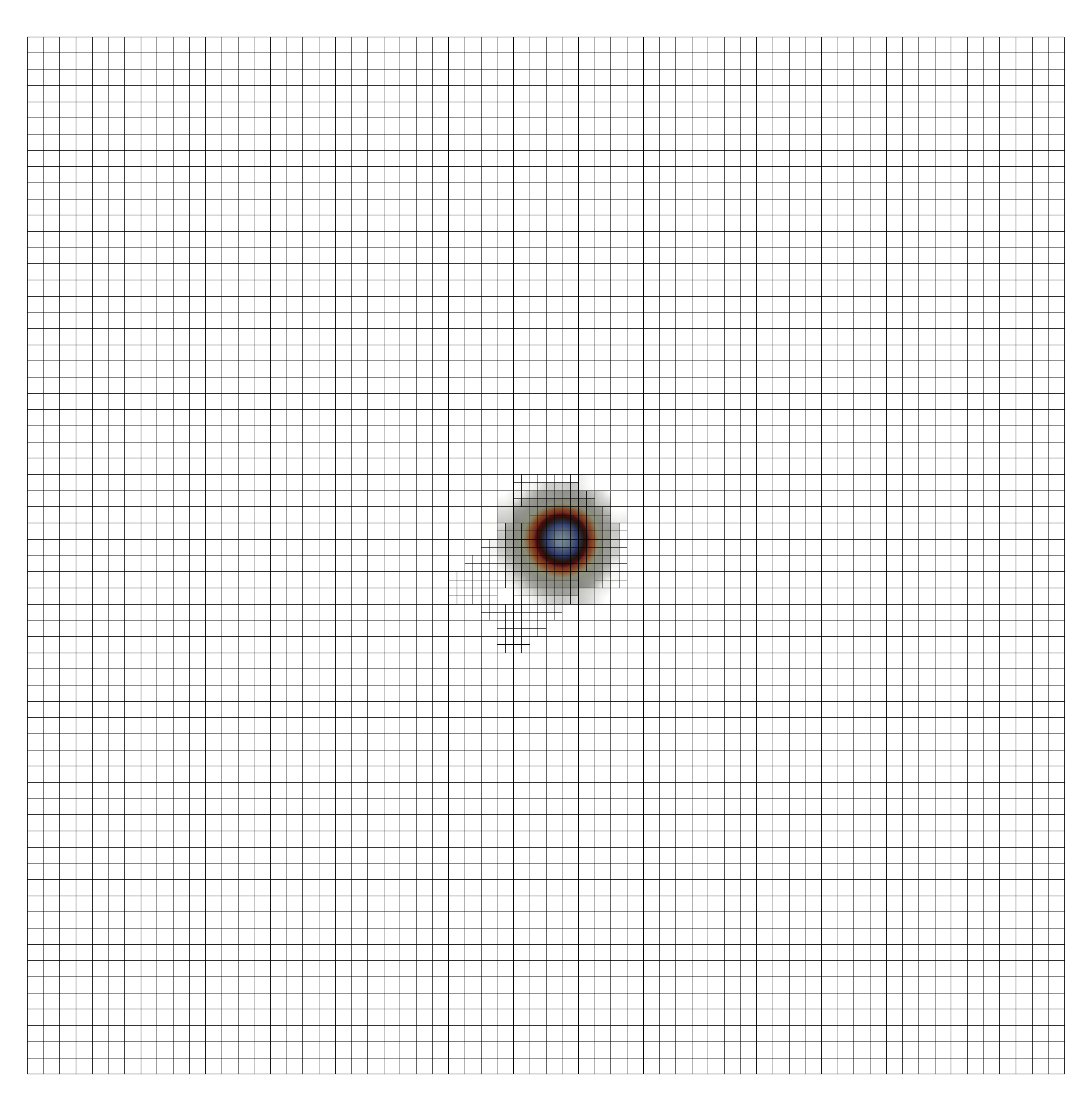}
    \captionsetup{labelformat=empty}
    \caption{$t=8$}
\end{subfigure}

\begin{subfigure}[t]{0.49\linewidth}
    \centering
    \includegraphics[width=0.49\linewidth]{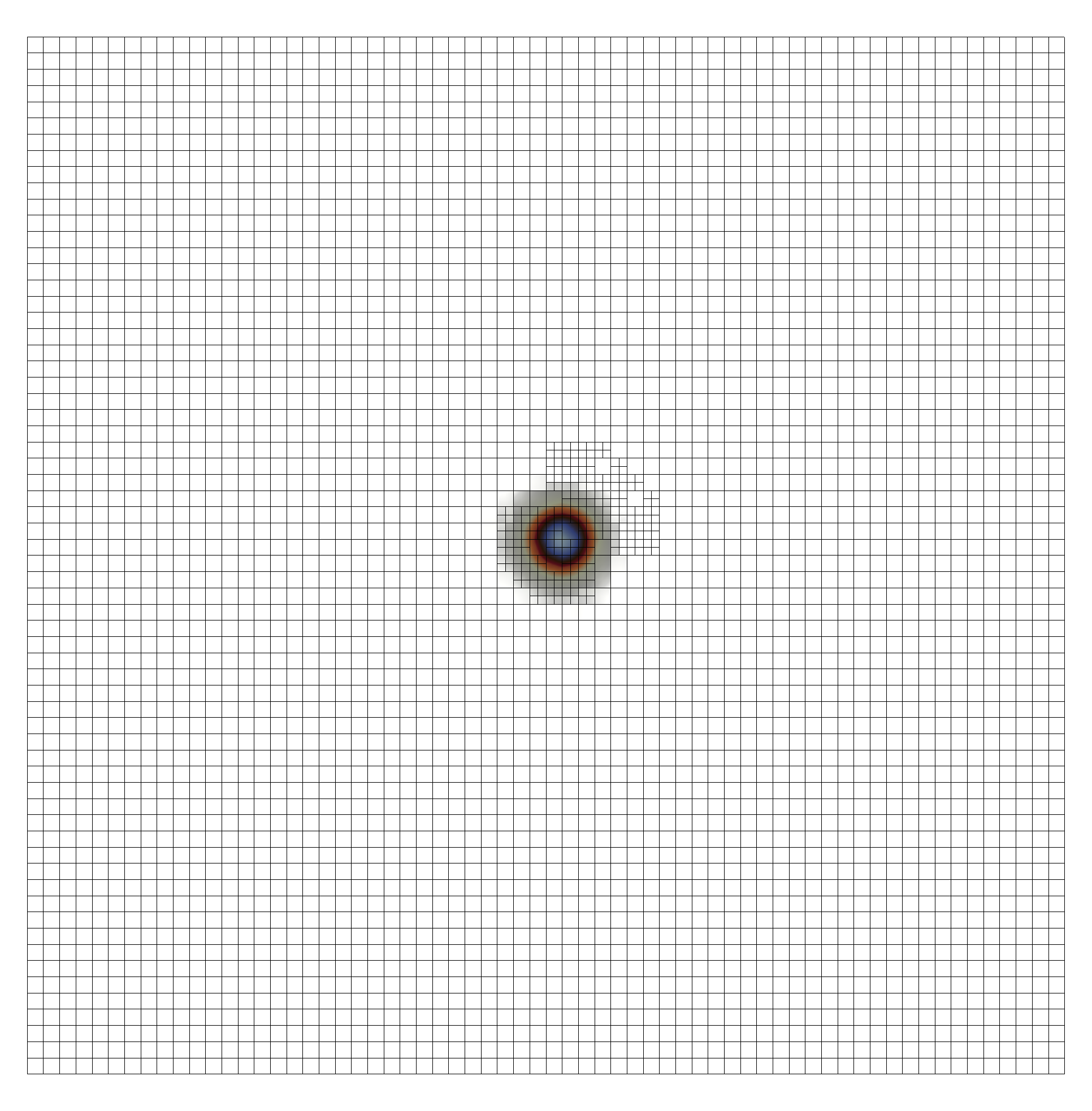}
    \includegraphics[width=0.49\linewidth]{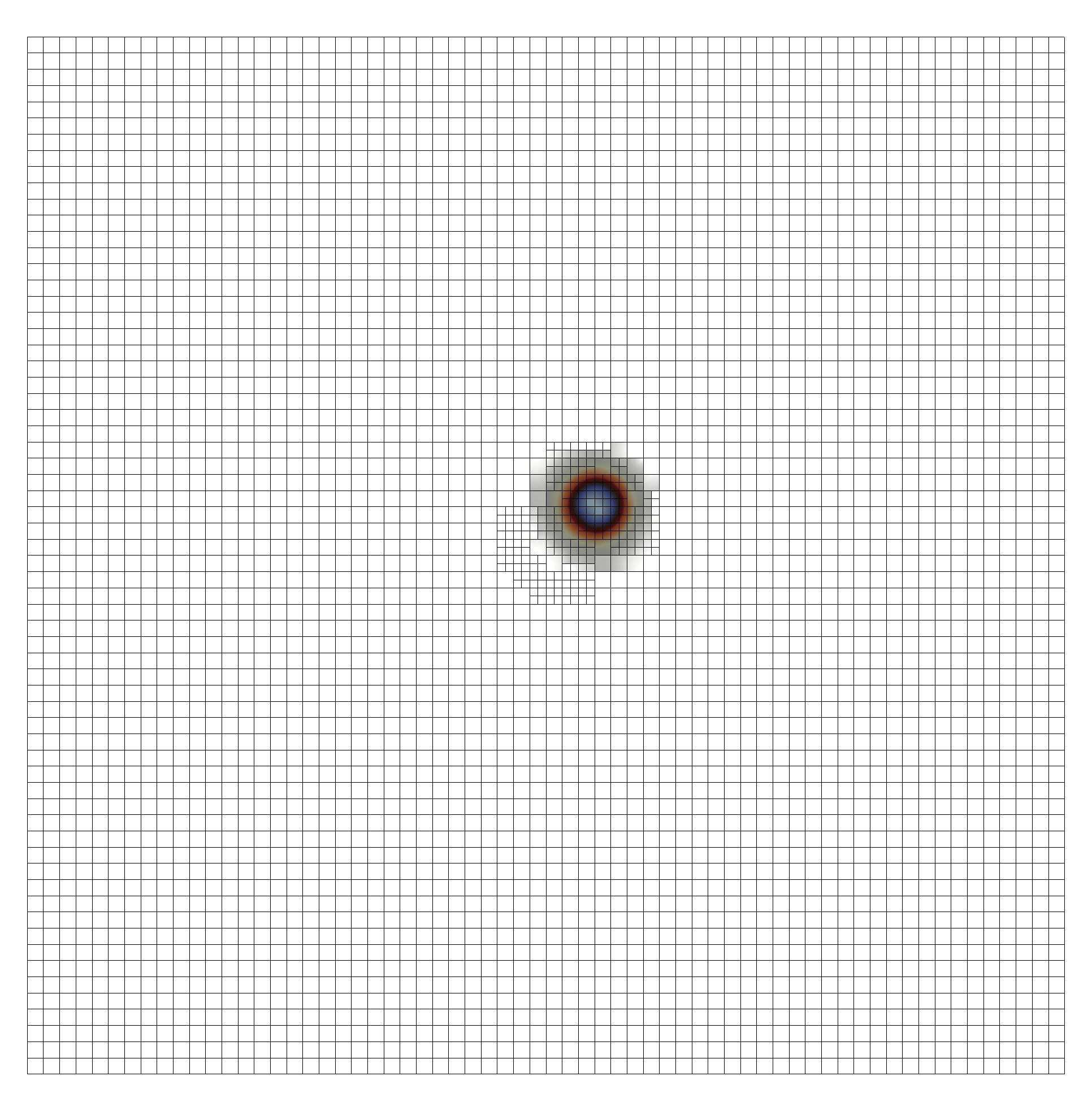}
    \captionsetup{labelformat=empty}
    \caption{$t=9$}
\end{subfigure}
\hfill
\begin{subfigure}[t]{0.49\linewidth}
    \centering
    \includegraphics[width=0.49\linewidth]{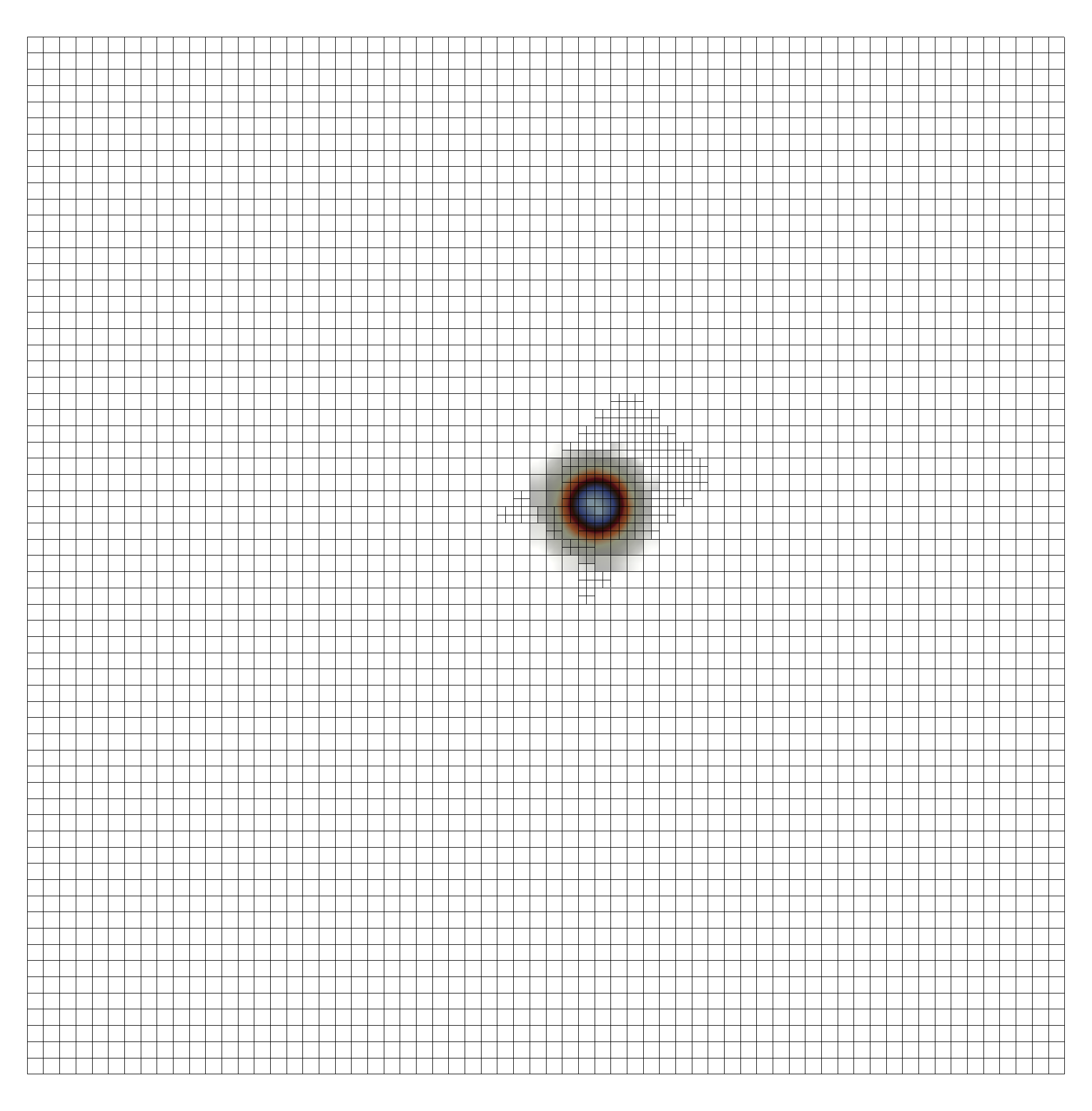}
    \includegraphics[width=0.49\linewidth]{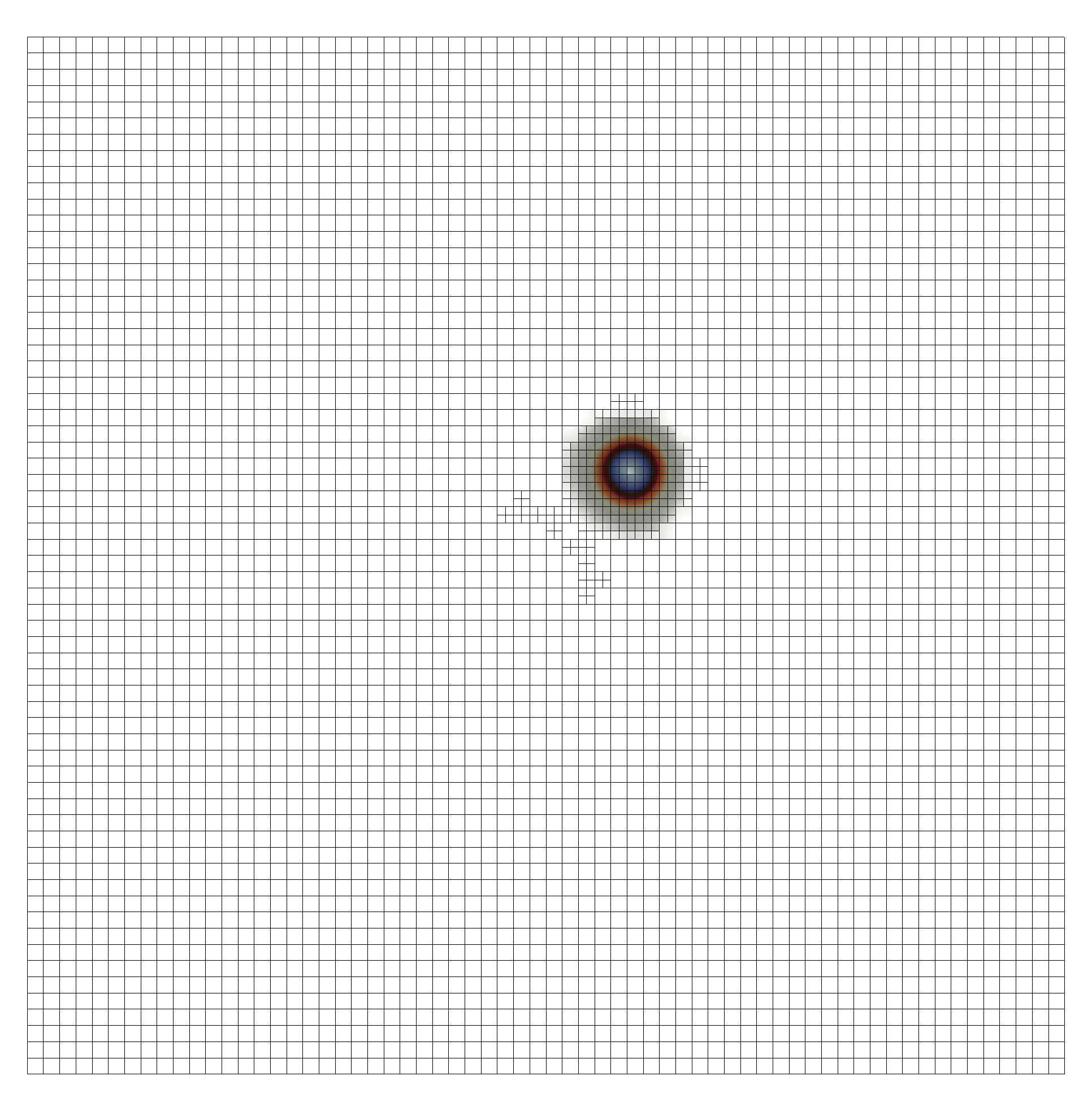}
    \captionsetup{labelformat=empty}
    \caption{$t=10$}
\end{subfigure}
\caption{Continued from \Cref{fig:iso_periodic_velunif_nx16_ny16_depth1_tstep0p25_vdn_graphnet_nodoftime_3_ep210800_nx64_ny64}. VDGN generalizes to larger meshes not seen in training.}
\label{fig:iso_periodic_velunif_nx16_ny16_depth1_tstep0p25_vdn_graphnet_nodoftime_3_ep210800_nx64_ny64_continued}
\end{figure*}

\end{document}